\documentclass[reqno]{amsart}

\usepackage[utf8]{inputenc} 
\usepackage[T1]{fontenc}    
\usepackage{url}            
\usepackage{booktabs}       
\usepackage{caption}
\usepackage{amsfonts}       
\usepackage{nicefrac}       
\usepackage{microtype}      
\usepackage[]{xcolor}         

\usepackage{amsmath,amssymb,amsfonts,mathrsfs}
\usepackage{nicefrac}
\usepackage{algorithmic}
\usepackage{algorithm}
\usepackage{graphicx}

\usepackage{multirow}

\usepackage{wrapfig}

\usepackage{booktabs}

\usepackage{tikz}
\usetikzlibrary{positioning, arrows.meta, calc, matrix}

\usepackage{xcolor}

\definecolor{strblue}{HTML}{0F1ED2}
\definecolor{strred}{HTML}{E61E8C}

\usepackage{natbib}

\usepackage{comment}

\usepackage{aliascnt}
\usepackage{hyperref}
\usepackage{cleveref}

\theoremstyle{definition}

\newtheorem{theorem}{Theorem}[section]
\crefname{theorem}{Theorem}{Theorems}
\Crefname{theorem}{Theorem}{Theorems}

\newaliascnt{definition}{theorem}
\newtheorem{definition}[definition]{Definition}
\aliascntresetthe{definition}
\crefname{definition}{Definition}{Definitions}
\Crefname{definition}{Definition}{Definitions}

\newaliascnt{proposition}{theorem}
\newtheorem{proposition}[proposition]{Proposition}
\aliascntresetthe{proposition}
\crefname{proposition}{Proposition}{Propositions}
\Crefname{proposition}{Proposition}{Propositions}

\newaliascnt{lemma}{theorem}
\newtheorem{lemma}[lemma]{Lemma}
\aliascntresetthe{lemma}
\crefname{lemma}{Lemma}{Lemmas}
\Crefname{lemma}{Lemma}{Lemmas}

\newaliascnt{corollary}{theorem}
\newtheorem{corollary}[corollary]{Corollary}
\aliascntresetthe{corollary}
\crefname{corollary}{Corollary}{Corollaries}
\Crefname{corollary}{Corollary}{Corollaries}

\newaliascnt{remark}{theorem}
\newtheorem{remark}[remark]{Remark}
\aliascntresetthe{remark}
\crefname{remark}{Remark}{Remarks}
\Crefname{remark}{Remark}{Remarks}

\newaliascnt{example}{theorem}

\aliascntresetthe{example}
\crefname{example}{Example}{Examples}
\Crefname{example}{Example}{Examples}

\newaliascnt{assumption}{theorem}
\newtheorem{assumption}[assumption]{Assumption}
\aliascntresetthe{assumption}
\crefname{assumption}{Assumption}{Assumptions}
\Crefname{assumption}{Assumption}{Assumptions}

\numberwithin{equation}{section}

\crefname{equation}{Equation}{Equations}
\crefname{figure}{Figure}{Figures}
\crefname{table}{Table}{Tables}
\crefname{algorithm}{Algorithm}{Algorithms}

\crefname{section}{}{}
\creflabelformat{section}{\S #2#1#3}
\crefname{subsection}{}{}
\creflabelformat{subsection}{\S #2#1#3}
\crefname{appendix}{Appendix}{Appendixes}

\DeclareMathOperator*{\argmin}{arg\,min} 
\DeclareMathOperator*{\argmax}{arg\,max} 
\DeclareMathOperator{\Conv}{Conv}

\DeclareMathOperator{\Id}{Id}

\DeclareMathOperator{\diam}{diam}

\DeclareMathOperator{\tr}{tr}
\DeclareMathOperator{\aff}{aff}
\DeclareMathOperator{\spn}{span}
\DeclareMathOperator{\dist}{dist}
\DeclareMathOperator{\cone}{cone}
\DeclareMathOperator{\rec}{rec}

\usepackage{pifont}

\title[Online Inverse Integer Linear Optimization via SGS]{Online Inverse Integer Linear Optimization via Small-Gradient Skipping: Constant Regret and Finite Mistakes}

\author{%
  Akira Kitaoka
}
\address{NEC Corporation, 1753 Shimonumabe, Nakahara-ku, Kawasaki, Kanagawa, Japan }
\email{akira-kitaoka@nec.com}


\begin{document}

\maketitle

\begin{abstract}%
In online inverse linear optimization, the learner predicts a weight at each round, observes the optimal action of the agent, and updates its prediction.
In the general setting, the gap of $\log T$ between the regret upper bound $O(d \log T)$ and the lower bound $\Omega(d)$ is unresolved (here $T$ is the total number of rounds and $d$ is the dimension).
When the action set is M-convex, the regret is known to be bounded by $O(d \log d)$, but the method attaining it computes a center of gravity at every round.
This paper therefore proposes \textbf{Small-Gradient Skipping} (SGS), a mechanism that skips the update at rounds without a mistake, and applies it to online gradient descent, the online Newton step, and MetaGrad.
When the forward problem is an integer linear program with a unique optimal solution, the number of mistakes is bounded, for all three, by a quantity independent of $T$; and for the online Newton step and for MetaGrad with SGS, the dimension dependence of the regret becomes $O(d^2)$, that is, the factor $\log T$ is removed.
Moreover, when the action set is M-convex, the regret is bounded efficiently without computing a center of gravity.%
\end{abstract}

\section{Introduction}
\label{sec:intro}

The problem of estimating, from observed actions, the criterion by which a decision maker chooses its actions has been studied as imitation learning and inverse reinforcement learning \citep{ng2000algorithms} and as inverse optimization \citep{ahuja2001inverse,heuberger2004inverse,chan2023inverse}.
The estimated weight can be interpreted as an objective function expressing the reason for the decision, and it has applications to the estimation of undisclosed objective functions in electricity markets \citep{birge2017inverse,liang2023data} and in healthcare \citep{chan2022inverse}.
This paper treats the case where the objective function of the forward problem (the optimization problem solved by the agent) is linear, in the online setting where states arrive sequentially: at each round the learner predicts a weight, observes the optimal action of the agent, and updates its prediction.
Methods from online learning are effective in this setting, and it is standard to measure the performance by the cumulative gap, measured by the true weight, between the action induced by the learner's prediction and the correct action (the cumulative decision regret $R^{\mathrm{est}}_T$) \citep{Barmann-2018-online,besbes2021online,gollapudi2021contextual,sakaue2025online,oki2026finite}.

However, most of the known upper bounds grow with the total number of rounds $T$: $O(\sqrt{T})$ for online gradient descent \citep{Barmann-2018-online}, and even for methods attaining logarithmic regret the bound is $O(d \log T)$ (where $d$ is the dimension of the weight) \citep{gollapudi2021contextual,sakaue2025online}.
Upper bounds independent of $T$ do exist, but each has a limitation: that of \citet[Theorem 4.2]{gollapudi2021contextual}, $\exp(O(d \log d))$, assumes neither a margin nor a gap but is exponential in the dimension; the bound under a gap condition \citep{sakaue2025aistats} is proportional to the inverse square of the gap; and the bound $O(d \log d)$ under the M-convexity of the action set \citep{oki2026finite} requires computing a center of gravity at every round (their deterministic method without it attains $O(d^2)$).
Whether a guarantee independent of $T$ and polynomial in the dimension can be obtained with light computation for general forward problems, including integer programs (with polytopal convex hulls of the feasible sets), was an open question.

Our contributions are the following.
\begin{itemize}
    \item \textbf{Proposal of small-gradient skipping under a uniform margin}: we assume a \textbf{uniform margin}, namely that there is a witness weight $\bar{\theta} \in \Theta$, not necessarily equal to the true weight $\theta^*$, under which the difference in objective value between the correct action and any other candidate action is at least some $\gamma > 0$, uniformly over all states (\Cref{assu:margin}). If the forward problem is an integer linear program, this margin is automatically positive, and its value can be bounded from below in terms of the combinatorial structure of the feasible set. Under this assumption we propose to incorporate into an online algorithm the mechanism---\textbf{small-gradient skipping} (SGS)---that updates neither the iterate nor the internal state at rounds where the proposal is correct, that is, at rounds where $0$ can be chosen as a subgradient. The guarantees are stated in terms of the number of mistakes rather than the total number of rounds, and the margin bounds that number of mistakes finitely.
    \item \textbf{Guarantees independent of $T$ and explicit upper bounds by problem class}: we apply SGS to online gradient descent (OGD), the online Newton step (ONS), and MetaGrad \citep{van2016metagrad,van2021metagrad}, and show that the number of mistakes and the regret are both bounded by quantities independent of the total number of rounds $T$ (\Cref{tab:main_results}). Furthermore, we bound the margin from below in terms of the combinatorial structure of the feasible set (\Cref{tab:gamma_lower}) and substitute it into the upper bounds to obtain explicit upper bounds by problem class (\Cref{tab:explicit_upper}). In particular, ONS and growing-grid SGS-MetaGrad, whose bounds depend on the margin only logarithmically, attain, when the forward problem is an ILP and the weight space is the probability simplex, $R^{\mathrm{est}}_T = O(d^2 \| M \|_2 \log (2\| M \|_2))$ (where $\| M \|_2$ is the norm of the vector $M$ in \cref{eq:def_M}), a bound polynomial in the dimension, with the light computation of $O(d^2)$ plus one generalized projection per mistake round. This means that, under a uniform margin, the gap of $\log T$ between the upper bound $O(d \log T)$ and the lower bound $\Omega(d)$, raised as an open problem by \citet{sakaue2025online}, disappears from the upper bound, so that the gap no longer depends on $T$. A comparison with existing methods is summarized in \Cref{tab:comparison}.
\end{itemize}
All detailed proofs are deferred to the appendices.

\begin{table}[ht]
\centering
\caption{Comparison of the decision regret \cref{eq:decision_regret} and the computational cost in the case where the forward problem is an integer linear program (ILP) and the weight space is the probability simplex $\Theta = \Delta^{d-1}$. Here $\| M \|_2$ is the norm of the vector $M$ of coordinatewise ranges \cref{eq:def_M}, $K$ $(\leq T)$ is the number of mistakes of each method, $\tau_{\mathrm{solve}}$ is the time for one linear optimization that computes the proposal $\hat{x}^t$, and $\tau_{\mathrm{E\text{-}proj}}$ / $\tau_{\mathrm{G\text{-}proj}}$ is the time for one Euclidean / generalized projection onto $\Theta$.
The lower part of the table lists the guarantees obtained in this paper, with the explicit lower bounds on the uniform margin (\Cref{sec:lower_bounds}) substituted in; these values do not depend on the total number of rounds $T$. Of the algorithms in that part, SGS-OGD (\Cref{alg:ogd}) and growing-grid SGS-MetaGrad (\Cref{alg:metagrad_anytime}) are proposed in this paper, whereas ONS (\Cref{alg:ons}) is the existing method of \citet{hazan2007logarithmic} (applied to online inverse linear optimization by \citealp{sakaue2025online}). They are the ILP / probability simplex entries of \Cref{tab:explicit_upper}, and the details of the substitution are in Appendix~\ref{app:proof_corollaries}.
${}^{\S}$: a result under the gap condition $\Delta > 0$ rather than a uniform margin. The value rewrites \citet[Theorem 5.2]{sakaue2025aistats} in the notation of this paper; the derivation is in Appendix~\ref{app:gap_entry}. The dash ``---'' in the total computational cost indicates that it is not compared in this table. The column ``Finite $K$'' indicates whether an upper bound on the number of mistakes $K$ that does not depend on $T$ is given ($\checkmark$) or not ($\times$).}
\label{tab:comparison}
{\footnotesize
\setlength{\tabcolsep}{3pt}
\begin{tabular}{llcl}
\toprule
Method & $R^{\mathrm{est}}_T$ & \begin{tabular}[c]{@{}c@{}} Finite \\ $K$ \end{tabular} & Total computational cost \\
\midrule
\begin{tabular}[c]{@{}l@{}} OGD \\ \citep{Barmann-2018-online} \end{tabular} & $O ( \| M \|_2 \sqrt{T} )$ & $\times$ & $O ( T ( \tau_{\mathrm{solve}} + \tau_{\mathrm{E\text{-}proj}} + d ) )$ \\
\citet{sakaue2025aistats}${}^{\S}$ & $O ( \| M \|_\infty (\log d)^{3/2} / \Delta^2 )$ & $\times$ & --- \\
\begin{tabular}[c]{@{}l@{}} ONS \\ \citep{sakaue2025online} \end{tabular} & $O ( \| M \|_2 d \log \frac{T}{d} )$ & $\times$ & $O ( T ( \tau_{\mathrm{solve}} + d^2 + \tau_{\mathrm{G\text{-}proj}} ) )$ \\
\begin{tabular}[c]{@{}l@{}} MetaGrad \\ \citep{sakaue2025online} \end{tabular} & $O ( \| M \|_2 d \log \frac{T}{d} )$ & $\times$ & $O ( T ( \tau_{\mathrm{solve}} + ( d^2 + \tau_{\mathrm{G\text{-}proj}} ) \log T ) )$ \\
\midrule
SGS-OGD & $O ( d\, 2^{d} \| M \|_2^{d+1} )$ & $\checkmark$ & $O ( T \tau_{\mathrm{solve}} + K ( d + \tau_{\mathrm{E\text{-}proj}} ) )$ \\
ONS & $O ( d^2 \| M \|_2 \log (2\| M \|_2) )$ & $\checkmark$ & $O ( T \tau_{\mathrm{solve}} + K ( d^2 + \tau_{\mathrm{G\text{-}proj}} ) )$ \\
\begin{tabular}[c]{@{}l@{}} Growing-grid \\ SGS-MetaGrad \end{tabular} & $O ( d^2 \| M \|_2 \log (2\| M \|_2) )$ & $\checkmark$ & $O ( T \tau_{\mathrm{solve}} + K ( d^2 + \tau_{\mathrm{G\text{-}proj}} ) \log K )$ \\
\bottomrule
\end{tabular}
}
\end{table}

\paragraph{Organization}
\Cref{sec:related} describes related work. \Cref{sec:setup} gives the problem setting and \Cref{sec:sgs} defines SGS. \Cref{sec:main_results} gives the guarantees for SGS-OGD, ONS, and SGS-MetaGrad. \Cref{sec:lower_bounds} gives lower bounds on the margin by structure, and \Cref{sec:explicit_upper} gives explicit upper bounds by problem class.

\section{Related work}
\label{sec:related}

\paragraph{Finitely many updates under a margin condition: the classical line}
That a margin keeps the number of updates finite is a classical theme. It begins with the Perceptron convergence theorem for linearly separable data, continues with ALMA~\citep{gentile2001new}, which approximates the maximum-margin classifier without being given the value of the margin explicitly, and reaches inverse optimization with \citet{sun2023maximum}, who gives, by a Perceptron-type method, the skeleton that leads from separability through finitely many mistakes to exact recovery. The former two are, however, results for binary classification, and they do not apply directly to the suboptimality loss treated in this paper.
The mechanism of not advancing the internal state in rounds without a mistake also has precedents: \citet{gollapudi2021contextual} skip the update at correct rounds in their reduction to a cutting-plane algorithm, and \citet{besbes2021online,besbes2025contextual} use a threshold-type skip that leaves the ellipsoidal cone unchanged in periods where the decision is nearly optimal. What this paper does anew is to formulate this mechanism for first- and second-order online convex optimization methods, whose internal state would otherwise advance even in rounds without a mistake, and to derive from it, under a uniform margin, guarantees independent of the total number of rounds together with explicit constants by problem class.

\paragraph{Finite regret in online inverse optimization}
This paper is not the first to bound the regret in online inverse linear optimization by a constant independent of the total number of rounds $T$: there are the bound under a gap condition \citep{sakaue2025aistats}, the bound under the M-convexity of the action set \citep{oki2026finite}, and the bound of \citet[Theorem 4.2]{gollapudi2021contextual}, which assumes neither a margin nor a gap; \citet{dewasurendra2026multiscale} obtains $O(d)$ for arbitrary compact action sets without a margin, but it keeps up to $\lceil \log_2 T \rceil (12T)^d$ weights, one for each element of covers of the parameter ball, so its cost per round is exponential in $d$.
The difference from the existing work is twofold. First, neither our algorithms nor their $T$-independent guarantees (\Cref{tab:main_results}) require any structure beyond the uniform margin, and each mistake round costs only $O(d^2)$ plus one generalized projection onto $\Theta$. Second, we bound the uniform margin explicitly from below in terms of the combinatorial structure of the feasible set (\Cref{sec:lower_bounds}) and reduce it to explicit upper bounds by problem class (\Cref{sec:explicit_upper}).

For the classical line, for a detailed comparison with each of these results, for a precedent of the uniform margin in offline inverse optimization, and for how the generality of the weight space $\Theta$ differs from that in the existing work, see Appendix~\ref{app:related}.

\section{Problem setting}
\label{sec:setup}

We consider an online learning setting with two players, the \emph{learner} and the \emph{agent}.\footnote{\label{footnote:linear-model} The ``agent'' is sometimes called an ``expert'', but we do not use that name in order to avoid confusion with the experts of online learning (see \Cref{sec:main_results}).}
Let $d$ be a positive integer and let $\mathbb{R}^d$ be the space on which the forward optimization is defined.
We call a nonempty set $\mathcal{S}$ the set of states, and for each state $s \in \mathcal{S}$ we write $X(s) \subseteq \mathbb{R}^d$ for the set of feasible actions.
For a weight $\theta \in \mathbb{R}^d$ and a state $s \in \mathcal{S}$, we write the forward problem (a linear optimization) and its optimal solution as
\begin{equation}
    \label{eq:fop}
    x^* (\theta , s) :\in \argmax_{x \in X(s)} \langle \theta , x \rangle
\end{equation}
(the attainment of the maximum follows from the compactness of $X(s)$ assumed in \Cref{assu:margin}(2)).
The agent has an unknown objective vector $\theta^* \in \mathbb{R}^d$, and for $t = 1,\dots,T$, when a state $s^t \in \mathcal{S}$ is given, it chooses $x^t = x^* (\theta^* , s^t) \in X(s^t)$ as its action.
From the observations $\{ (s^t, x^t) \}_{t=1}^T$ we want to find a weight $\theta$ satisfying $x^*(\theta^*, s) \in \argmax_{x \in X(s)} \langle \theta, x \rangle$ at each state $s^t$ (the inverse linear optimization problem).

Note that the set $X(s)$ is not necessarily convex.
If $X (s)$ is a polyhedron, then the solution returned by any solver for linear programming (LP) can serve as an oracle for $x^* (\theta ,s )$.
Also when $X(s)$ is defined by integer linear constraints, an optimal solution can be obtained with an empirically efficient solver such as Gurobi.

The learner predicts $\theta^*$ sequentially for $t=1,\dots,T$.
Let $\Theta \subseteq \mathbb{R}^d$ be the set of linear objective vectors from which the learner chooses its predictions (the conditions imposed on $\Theta$ and $\theta^*$ are collected in \Cref{assu:margin}).
Below, $\|\cdot\|$ denotes the $\ell_2$ norm and $\log$ the natural logarithm (we write, for instance, $\log_2$ only when the base is made explicit). We set the diameter of the weight space and the constant expressing the spread of the actions to be
\begin{equation}
    D := \diam(\Theta) ,
    \qquad
    L := \sup_{s \in \mathcal{S}} \sup_{x \in X(s)} \| x - x^*(\theta^*, s) \|
    \label{eq:def_DL}
\end{equation}
respectively.
For $t=1,\dots,T$, the learner outputs a prediction $\hat{\theta}^t \in \Theta$ of $\theta^*$ based on the past observations $\{ (s^{t^\prime}, x^{t^\prime}) \}_{t^\prime =1}^{t-1}$, and receives $(s^t, x^t)$ as feedback from the agent.
Let $Y(s)$ be the set of extreme points of the convex set $\Conv X(s)$ (since, for a compact set $X(s)$, the extreme points of $\Conv X(s)$ belong to $X(s)$, we have $Y(s) \subseteq X(s)$).
The proposal $\hat{x}^t$ induced by the learner's $t$-th prediction $\hat{\theta}^t$ is defined as an extreme optimal solution
\begin{equation}
    \hat{x}^t \in \argmax_{x \in Y(s^t)} \langle \hat{\theta}^t, x \rangle .
    \label{eq:proposal}
\end{equation}
Since the maximum of a linear function on $\Conv X(s^t)$ is attained at an extreme point, the optimal value of \cref{eq:proposal} coincides with the maximum on $X(s^t)$ (the value of \cref{eq:fop}). Moreover, the simplex method returns an extreme point of the feasible polyhedron; for an ILP, returning an extreme optimal solution when several optimal solutions exist is imposed as a requirement on the oracle \cref{eq:proposal}.\footnote{The simplex method returns an extreme point of the feasible polyhedron. For an ILP, if the optimal solution is unique then any optimal integer solution is automatically an extreme point of $\Conv X(s^t)$ (since the maximum of a linear function is also attained at an extreme point, uniqueness makes the two coincide). Returning an extreme optimal solution when there are several optimal solutions is imposed as a requirement on the oracle \cref{eq:proposal}, and the implementation details of the forward-problem solver are outside the scope of this paper.}

In inverse linear optimization, the suboptimality loss \citep{Mohajerin-2018-Data} is a useful criterion.
For a weight $\theta$ and a state $s$, the suboptimality loss is defined, using \cref{eq:fop}, by
\begin{equation}
    \ell_{\mathrm{sub}} (\theta , s) := \langle \theta, x^*(\theta, s) - x^*(\theta^*, s) \rangle
    = \max_{x \in X(s)} \langle \theta, x - x^*(\theta^*, s) \rangle
    \geq 0 .
    \label{eq:suboptimality_loss}
\end{equation}
If the suboptimality loss is $0$, then $x^*(\theta^*, s) \in \argmax_{x \in X(s)} \langle \theta, x \rangle$ holds, which means that the inverse optimization problem is solved at that state.

\paragraph{Performance criteria}
This paper measures the quality of the learner's sequence of predictions $\hat{\theta}^1, \ldots, \hat{\theta}^T$ by the following three quantities.
The first is the \textbf{cumulative suboptimality regret}
\begin{equation}
    R^{\mathrm{sub}}_T := \sum_{t=1}^T \ell_{\mathrm{sub}} (\hat{\theta}^t, s^t)
    = \sum_{t=1}^T \langle \hat{\theta}^t,\, \hat{x}^t - x^t \rangle ,
    \label{eq:cum_loss}
\end{equation}
the cumulative sum of the suboptimality loss \cref{eq:suboptimality_loss}, which expresses how poorly the agent's action $x^t$ is explained from the viewpoint of the learner's weight $\hat{\theta}^t$.
The second is the \textbf{cumulative decision regret}
\begin{equation}
    R^{\mathrm{est}}_T := \sum_{t=1}^T \langle \theta^*,\, x^t - \hat{x}^t \rangle ,
    \label{eq:decision_regret}
\end{equation}
which expresses how suboptimal the learner's proposal $\hat{x}^t$ is from the viewpoint of the true weight $\theta^*$.
It is this $R^{\mathrm{est}}_T$ that \citet{besbes2021online,gollapudi2021contextual,besbes2025contextual,sakaue2025online,oki2026finite} simply call the regret.
Which criterion is bounded in which reference is summarized in \Cref{tab:metric_comparison}.
The third is the \textbf{sum} of the two,
\begin{equation}
    \widetilde{R}_T := R^{\mathrm{sub}}_T + R^{\mathrm{est}}_T
    = \sum_{t=1}^T \langle \hat{\theta}^t - \theta^*,\, \hat{x}^t - x^t \rangle ,
    \label{eq:regret_decomp}
\end{equation}
which amounts to the quantity $\widetilde{R}^{c^*}_T$ introduced by \citet{sakaue2025online}.
Since $x^t$ and $\hat{x}^t$ are optimal for $\theta^*$ and $\hat{\theta}^t$ respectively, we have $R^{\mathrm{sub}}_T, R^{\mathrm{est}}_T \geq 0$, and hence
\begin{equation}
    \max ( R^{\mathrm{sub}}_T,\; R^{\mathrm{est}}_T ) \leq \widetilde{R}_T
    \label{eq:max_leq_tilde}
\end{equation}
holds.
An upper bound on one of the components does not give an upper bound on the other, but each theorem of this paper bounds the sum $\widetilde{R}_T$ itself independently of $T$, so that $R^{\mathrm{sub}}_T$ and $R^{\mathrm{est}}_T$ are bounded simultaneously.

To solve online inverse linear optimization, we introduce the following assumption.
\begin{assumption}
    \label{assu:margin}
    \begin{description}
        \item[(1)] $\Theta \subset \mathbb{R}^d$ is a nonempty bounded closed convex set with $D > 0$ (boundedness gives $D < \infty$).
        \item[(2)] For each state $s \in \mathcal{S}$, the set $X(s) \subset \mathbb{R}^d$ is nonempty and compact, and the set of extreme points $Y(s)$ (defined just before \cref{eq:proposal}) is a finite set.
        \item[(3)] $\theta^* \in \Theta$, and for each state $s \in \mathcal{S}$ the set $\argmax_{x \in X(s)} \langle \theta^* , x \rangle$ is a singleton. We write its unique element as $x^* (\theta^* ,s)$ (since the unique maximizer of a linear function is an extreme point of $\Conv X(s)$, we have $x^* (\theta^* ,s) \in Y(s)$).
        \item[(4)] (Uniform margin) There exist $\bar{\theta} \in \Theta$ and $\gamma > 0$ such that, for every $s \in \mathcal{S}$ and every $x \in Y(s) \setminus \{ x^*(\theta^*, s) \}$,
        \begin{equation}
            \langle \bar{\theta}, x^*(\theta^*, s) - x \rangle \geq \gamma .
            \label{eq:margin}
        \end{equation}
        \item[(5)] The constant $L$ in \cref{eq:def_DL} satisfies $0 < L < \infty$.
    \end{description}
\end{assumption}

\begin{remark}
    For the constant $\gamma$ of the uniform margin in \Cref{assu:margin}(4), a concrete lower bound can be obtained, for instance, when the constraint set $X(s)$ is given by integer linear constraints.
    Concrete examples for each structure of the feasible set are given in \Cref{sec:lower_bounds}. The lower bounds by structure are summarized in \Cref{tab:gamma_lower}.
    Unlike the gap condition of \citet{sakaue2025aistats}, imposed on $\theta^*$ itself (Appendix~\ref{app:gap_entry}), (4) only requires a witness $\bar{\theta}$, which need not be $\theta^*$.
\end{remark}

\section{Proposed method: small-gradient skipping (SGS)}
\label{sec:sgs}

\begin{definition}[Small-gradient skipping]
    \label{defi:sgs}
    For an online learning method with an iterate $\hat{\theta}^t$ and an internal state (a learning-rate index, an information matrix, the weights of experts, and so on), \textbf{small-gradient skipping} (Small-Gradient Skipping; SGS) refers to the following mechanism: at a round where no mistake occurred ($\hat{x}^t = x^t$), neither the iterate nor the internal state is updated at all ($\hat{\theta}^{t+1} = \hat{\theta}^t$), and the index $k$ of the internal state is not advanced either. Only at mistake rounds is the update performed with the subgradient $g = \hat{x}^t - x^t$, and $k$ advanced by one.
\end{definition}

The name comes from the following observation: at a round where no mistake occurs, the value of the loss $\ell_{\mathrm{sub}}(\hat{\theta}^t, s^t) = \langle \hat{\theta}^t, \hat{x}^t - x^t \rangle = 0$ is the minimum value of $\ell_{\mathrm{sub}}(\cdot, s^t) \geq 0$, and then $0 \in \partial_\theta \ell_{\mathrm{sub}}(\hat{\theta}^t, s^t)$, that is, the learner can choose $0$ (a sufficiently small gradient) as a subgradient. SGS skips exactly these rounds.
The oracle-based subgradient $g^t = \hat{x}^t - x^t$ used in this paper is, as a vector, $g^t = 0$ at rounds where no mistake occurs.

\begin{remark}[When SGS changes the algorithm and when it changes only the analysis]
    \label{rem:sgs_algorithmic}
    For an online learning method whose update rule depends only on $g^t$, it follows that, even without incorporating SGS, the weight of the objective function does not move at rounds with $g^t = 0$. That is, incorporating SGS does not change the algorithm. Examples are ONS and fixed-grid MetaGrad (\Cref{alg:metagrad}).
    On the other hand,
    for an online learning method whose update rule depends, in addition to $g^t$, on the round index, incorporating SGS does change the algorithm. Examples are
    first-order methods whose learning rate decays with the round number (the $\alpha k^{-1/2}$ of SGS-OGD differs from $\alpha t^{-1/2}$) and
    growing-grid SGS-MetaGrad (\Cref{alg:metagrad_anytime}), whose learning-rate grid is refined according to the number of mistakes instead of the round index (\Cref{sec:main_results}).
\end{remark}

We call a round $t$ with $\hat{x}^t \neq x^t$ a \textbf{mistake}, and write $K$ for the total number of such rounds.
By \Cref{defi:sgs}, the iterate is updated only at mistake rounds.

In this paper we treat the following three instances of SGS: \textbf{SGS-OGD} (\Cref{alg:ogd}), which incorporates SGS into projected online gradient descent; \textbf{ONS} (\Cref{alg:ons}), a second-order method; and \textbf{SGS-MetaGrad} (\Cref{alg:metagrad}; including \Cref{alg:metagrad_anytime}, which grows the learning-rate grid with the number of mistakes), a universal method.
The pseudocode and all the accompanying remarks are collected in Appendix~\ref{app:algorithms}, and the guarantees are given in \Cref{sec:main_results}.

\section{Main results: finitely many mistakes and regret independent of \texorpdfstring{$T$}{T}}
\label{sec:main_results}

\paragraph{SGS-OGD (a first-order method)}
As the basic form of a first-order method, we consider SGS-OGD (\Cref{alg:ogd}), which incorporates SGS into projected online gradient descent: only at mistake rounds does it perform a subgradient step with step size $\alpha k^{-1/2}$ (where $k$ is the index counting the mistakes) followed by a Euclidean projection.
Its only parameters are $L$ and $D$.

\paragraph{ONS (a second-order method)}
To improve the dependence of SGS-OGD on $\gamma$, we use the second-order method ONS \citep[cf.][]{hazan2007logarithmic} (\Cref{alg:ons}).
Even when SGS is incorporated into ONS, it holds automatically that neither the iterate nor the information matrix moves at rounds without a mistake, so the SGS version generates the same sequence of iterates as plain ONS (\Cref{rem:sgs_algorithmic,rem:ons_is_ons}).
This paper therefore incorporates the SGS viewpoint of counting only mistake rounds into the analysis of ONS, and shows that under a margin the guarantee of ONS improves to one independent of $T$.
Its only parameters are $L$ and $D$.

\paragraph{SGS-MetaGrad}
Even if ONS is replaced by MetaGrad \citep{van2016metagrad,van2021metagrad} (\Cref{alg:metagrad}), the update at rounds without a mistake is automatically the identity, so SGS works on the side of the analysis and a guarantee of the same order is obtained (\Cref{rem:metagrad_is_metagrad}; the statement and the proof of the fixed-grid version are in Appendix~\ref{app:proof_metagrad}).
However, since MetaGrad constructs its learning-rate grid before execution, an upper bound $\bar{K}$ on the number of mistakes is needed to determine its size, and the factor $c_0(\bar{K}) = 2 \log ( \frac{1}{2} \log_2 \bar{K} + 3 )$ remains in the guarantee.
Since in general only $\bar{K} = T$ can be taken as an a priori upper bound, we get $c_0(\bar{K}) = O(\log \log T)$, and strict independence from the total number of rounds fails to that extent.
This $\bar{K}$ originates solely from fixing the grid before execution, so if we do not fix the grid but keep adding smaller learning rates as the number of mistakes $k$ progresses, the input $\bar{K}$ itself becomes unnecessary, $c_0$ is replaced by $c_0(K)$ with the realized number of mistakes $K$, and the guarantee becomes completely independent of $T$.
This \textbf{growing-grid SGS-MetaGrad} (\Cref{alg:metagrad_anytime}) is an algorithm of this paper that differs from ordinary MetaGrad in that the grid is refined according to the number of mistakes (\Cref{rem:metagrad_anytime}).

\paragraph{Main results}
The guarantees of the three methods are summarized below.

\begin{theorem}[Summary of the main results]
    \label{theo:summary}
    Under \Cref{assu:margin}, if SGS-OGD (with $\alpha = D/(L\sqrt{2})$), ONS, and growing-grid SGS-MetaGrad are run on an arbitrary (possibly adaptive) sequence of states $\{ s^t \}_{t=1}^T$, then the number of mistakes $K$, the cumulative suboptimality regret $R^{\mathrm{sub}}_T$, and the sum $\widetilde{R}_T$ satisfy the upper bounds in \Cref{tab:main_results}.
    All of them hold deterministically for every $T$, and the right-hand sides do not depend on the total number of rounds $T$.
    Moreover, by \cref{eq:max_leq_tilde}, the cumulative decision regret $R^{\mathrm{est}}_T$ has the same upper bound as the regret of the sum.
\end{theorem}

The complete statements with explicit constants, together with their proofs, are \Cref{theo:ogd} (Appendix~\ref{app:proof_sgs_ogd}) for SGS-OGD, \Cref{theo:main} (Appendix~\ref{app:proof_ons}) for ONS, and \Cref{theo:metagrad_anytime_main} (Appendix~\ref{app:proof_metagrad_anytime}) for growing-grid SGS-MetaGrad.

\begin{table}[ht]
\centering
\caption{The guarantees of the proposed methods (\Cref{theo:summary}). None of the right-hand sides depends on the total number of rounds $T$. The decision regret has the same upper bound as the regret of the sum (\cref{eq:max_leq_tilde}). For the explicit form for growing-grid SGS-MetaGrad see \Cref{theo:metagrad_anytime_main} (the coefficient of the logarithmic term becomes $52$ against the $2$ of ONS, and an additional term $\log\log d$ enters).}
\label{tab:main_results}
{\footnotesize
\setlength{\tabcolsep}{4pt}
\begin{tabular}{llll}
\toprule
 & \begin{tabular}[c]{@{}l@{}} SGS-OGD \\ (\Cref{theo:ogd}) \end{tabular} & \begin{tabular}[c]{@{}l@{}} ONS \\ (\Cref{theo:main}) \end{tabular} & \begin{tabular}[c]{@{}l@{}} Growing-grid SGS-MetaGrad \\ (\Cref{theo:metagrad_anytime_main}) \end{tabular} \\
\midrule
$K$ & $\frac{2 L^2 D^2}{\gamma^2}$ & $d + \frac{2LD}{\gamma} ( 1 + d \log \max ( \tfrac{2LD}{\gamma}, 1 ) )$ & $O ( \tfrac{d L D}{\gamma} \log \max ( \tfrac{2LD}{\gamma}, 2 ) )$ \\[4mm]
$R^{\mathrm{sub}}_T$ & $\frac{L^2 D^2}{2 \gamma}$ & $LD ( 1 + d \log \max ( \tfrac{LD}{\gamma}, 1 ) )$ & $O ( L D\, d \log \max ( \tfrac{LD}{\gamma}, 1 ) )$ \\[4mm]
$\widetilde{R}_T$ & $\frac{2 L^2 D^2}{\gamma}$ & $LD ( 1 + 2 d \log ( 2 + \tfrac{2LD}{\gamma} ) )$ & $O ( L D\, d \log \max ( \tfrac{LD}{\gamma}, 2 ) )$ \\
\bottomrule
\end{tabular}
}
\end{table}

\paragraph{Comparison of the methods}
Compared with SGS-OGD, ONS reduces the dependence on the margin $\gamma$ from $\gamma^{-2}$ to $\gamma^{-1}$ in the number of mistakes, and from $\gamma^{-1}$ to $\log \gamma^{-1}$ in the cumulative suboptimality regret.
Since $\log \gamma^{-1} \leq \gamma^{-1}$ holds for every $\gamma > 0$, and moreover the margin is often small, as we shall see in \Cref{sec:lower_bounds} (for a general ILP the lower bound on $\gamma$ can be exponentially small in the dimension), this replacement is a substantial improvement.
The price is that ONS incurs a linear dependence on the dimension $d$, and the per-round computational cost also increases from $O(d + \tau_{\mathrm{E\text{-}proj}})$ to $O(d^2 + \tau_{\mathrm{G\text{-}proj}})$ (\Cref{tab:complexity}).
Therefore, on problems with a small margin ($\gamma \lesssim LD / (d \log(LD/\gamma))$) ONS is superior, whereas on problems with a large margin the bound $2L^2D^2/\gamma^2 = O(1)$ of SGS-OGD is superior.
Growing-grid SGS-MetaGrad attains a guarantee of the same order as ONS at the price of worse constants, and in addition the Lipschitz-adaptive version does not even require knowledge of $L$ (\Cref{rem:metagrad_adaptive}).
Moreover, compared with the guarantees of \citet{sakaue2025online} for ONS and MetaGrad, under a margin the $\log T$ in the cumulative suboptimality regret is replaced by $\log\max(LD/\gamma, 1)$ and an upper bound on the number of mistakes independent of $T$ is added: it is the degree of separation of the problem, not the total number of rounds, that determines the logarithmic term.

\section{Integer programming and lower bounds on the uniform margin}
\label{sec:lower_bounds}

The upper bounds of \Cref{theo:ogd,theo:main,theo:metagrad_main,theo:metagrad_anytime_main} are given in closed form in the margin $\gamma$.
In this section we quantify $\gamma$ from below according to the combinatorial structure of the forward problem of data-driven inverse optimization (DDIOP), and by substituting the result into the upper bounds we make the number of mistakes and the cumulative regret explicit by problem class.
The main target is the general integer linear program (ILP), where an explicit finite upper bound is obtained unconditionally from an explicit lower bound on $\gamma$ (which is exponentially small in the dimension, but positive).\footnote{To be precise, the explicit lower bounds hold unconditionally both when the weight space is the probability simplex (\Cref{theo:gamma_sub_lowerbound_on_simplex}) and when it is the unit ball (\Cref{theo:gamma_sub_lowerbound_on_ball}).}
Under a discrete convex structure such as M-convexity, the lower bound improves to a polynomial and the upper bounds become polynomial in the dimension.

\subsection{The largest attainable margin}
\label{sec:ddiop_setup}

In this section we assume (1), (2) and (3) of \Cref{assu:margin} and ask how large the margin $\gamma$ in (4) can be taken.
Integrality is imposed only from \Cref{sec:ilp_lower} on, at the stage where the lower bound is evaluated from the combinatorial structure.
We define the largest attainable margin and the constant expressing the spread of the actions by
\begin{align}
    \gamma_\mathrm{sub} & := \max_{\theta \in \Theta} \inf_{s \in \mathcal{S}} \min_{x \in Y(s) \setminus \{ x^*(\theta^*, s) \}} \langle \theta, x^*(\theta^*, s) - x \rangle ,
    \label{eq:gamma_SL} \\
    L_\mathrm{sub} & := \sup_{s \in \mathcal{S}} \sup_{x \in X(s)} \left\| x - x^*(\theta^*, s) \right\|
    \label{eq:Lipschitz_SL}
\end{align}
respectively (states $s$ with $Y(s) \setminus \{x^*(\theta^*, s)\} = \emptyset$ are excluded from $\inf_s$). The inner $\min$ is a minimum over a finite set by \Cref{assu:margin}(2), and hence is attained.
The quantity $\gamma_\mathrm{sub}$ is the value, at the weight that makes it largest, of ``the difference in objective value between the correct action and the other candidate actions''.

\begin{lemma}[Attainment of $\gamma_\mathrm{sub}$ and validity of the margin]
    \label{lem:margin_from_gamma}
    Assume (1), (2) and (3) of \Cref{assu:margin} and $0 < L_\mathrm{sub} < \infty$.
    Then the $\max_{\theta \in \Theta}$ in \cref{eq:gamma_SL} is attained.
    Furthermore, if $\gamma_\mathrm{sub} > 0$, then \Cref{assu:margin} holds with
    $\gamma = \gamma_\mathrm{sub}$
    for a maximizer $\bar{\theta} \in \Theta$ (and, since the definitions \cref{eq:def_DL} and \cref{eq:Lipschitz_SL} are identical, $L = L_\mathrm{sub}$).
\end{lemma}

See Appendix~\ref{app:proof_lower} for the proof.
All the lower-bound theorems below are for $\gamma_\mathrm{sub}$, and they can be substituted into the upper bounds through \Cref{lem:margin_from_gamma}.

\begin{remark}[Generalization to a general feature map]
    \label{rem:general_feature}
    In data-driven inverse optimization, the forward problem is often written as $\max_{u} \langle \theta, f(u, s) \rangle$ with a decision variable $u$ and a feature map $f$.
    Also in this case, the results below apply as they are once the set $X(s)$ of this section is read as the image $f(\mathcal{U}(s), s)$ of the features (where $\mathcal{U}(s)$ is the feasible set of the decision variable).
    This reading requires that the image satisfy the assumptions of each result; in particular, the integer-programming lower bounds require $f(\mathcal{U}(s), s) \subset \mathbb{Z}^d$, which holds, for instance, when $\mathcal{U}(s)$ consists of integer points and each component of $f(\cdot, s)$ is an affine map with integer coefficients. If $f$ is not integer-valued, the image is in general not a set of integer points, and those lower bounds do not apply directly.
    Below, to keep the notation simple, we regard the feature map as the identity and argue on $X(s)$.
\end{remark}

\subsection{Explicit lower bounds for general ILPs}
\label{sec:ilp_lower}

\begin{assumption}[Integer programming]
    \label{assu:ILP}
    For every $s \in \mathcal{S}$ we have $X(s) \subset \mathbb{Z}^d$.
    Furthermore, we set the coordinatewise ranges to be
    \begin{equation}
        M_i := \sup_{s \in \mathcal{S}} \left( \max_{x \in X(s)} x_i - \min_{x \in X(s)} x_i \right)
        \quad (i = 1, \ldots, d),
        \qquad
        M := ( M_1, \ldots, M_d ) .
        \label{eq:def_M}
    \end{equation}
\end{assumption}

An explicit lower bound on $\gamma_\mathrm{sub}$ is obtained for each combinatorial structure of the feasible set.
The results in the case where the weight space $\Theta$ is the probability simplex $\Delta^{d-1} := \{ \theta \in \mathbb{R}_{\geq 0}^d : \sum_{i=1}^d \theta_i = 1 \}$ and in the case where it is the unit ball $B^d := \{ \theta \in \mathbb{R}^d : \| \theta \|_2 \leq 1 \}$ are summarized in \Cref{tab:gamma_lower}.
Both the statements and the proofs are placed in Appendix~\ref{app:margin_lower}.

\begin{table}[ht]
\centering
\caption{Lower bounds on $\gamma_\mathrm{sub}$ by structure ($d \geq 2$). The structure in each row is defined by the assumption in parentheses, and the theorem in parentheses below each value is the statement in which that bound is proved. Here $\| M \|_2$ is the norm of the vector $M$ in \cref{eq:def_M}, and $C_g = g_\infty(\widetilde{A})$ is the $\ell_\infty$ norm of the Graver basis of the slack-augmented matrix $\widetilde{A} = [A \mid \Id_N]$ (where $\Id_N$ is the identity matrix of order $N$); if $A$ is totally unimodular then $C_g = 1$; for $A$, $N$ and $b(s)$, see \Cref{assu:linear-inequality}.
}
\label{tab:gamma_lower}
{\footnotesize
\setlength{\tabcolsep}{3pt}
\begin{tabular}{lcc}
\toprule
& \multicolumn{2}{c}{Lower bound on $\gamma_\mathrm{sub}$} \\
\cmidrule(lr){2-3}
Structure of the feasible set & $\Theta = \Delta^{d-1}$ & $\Theta = B^d$ \\
\midrule
\begin{tabular}[c]{@{}l@{}} General ILP \\ (\Cref{assu:ILP}) \end{tabular} & \begin{tabular}[c]{@{}c@{}} $\dfrac{1}{2^{d-1} \max(d-1, \sqrt{2})\, \| M \|_2^{d-1}}$ \\ (\Cref{theo:gamma_sub_lowerbound_on_simplex}) \end{tabular} & \begin{tabular}[c]{@{}c@{}} $\dfrac{1}{2^{d-1} \| M \|_2^{d-1}}$ \\ (\Cref{theo:gamma_sub_lowerbound_on_ball}) \end{tabular} \\[2mm]
\begin{tabular}[c]{@{}l@{}} Linear inequalities \\ (\Cref{assu:linear-inequality}) \end{tabular} & \begin{tabular}[c]{@{}c@{}} $\dfrac{1}{\max(d-1, \sqrt{2})\, (2 C_g \sqrt{d})^{d-1}}$ \\ (\Cref{theo:gamma_sub_linear_inequality_simplex}) \end{tabular} & \begin{tabular}[c]{@{}c@{}} $\dfrac{1}{(2 C_g \sqrt{d})^{d-1}}$ \\ (\Cref{theo:gamma_sub_linear_inequality_ball}) \end{tabular} \\[2mm]
\begin{tabular}[c]{@{}l@{}} M-convex set \\ (\Cref{assu:M-convex-projection}) \end{tabular} & \begin{tabular}[c]{@{}c@{}} $\dfrac{2}{d(d-1)}$ \\ (\Cref{theo:gamma_sub_M_convex_poly}) \end{tabular} & \begin{tabular}[c]{@{}c@{}} $\dfrac{2\sqrt{3}}{\sqrt{d(d^2-1)}}$ \\ (\Cref{theo:gamma_sub_M_convex_poly}) \end{tabular} \\[2mm]
\begin{tabular}[c]{@{}l@{}} M${}^\natural$-convex set \\ (\Cref{assu:Mnatural-convex-projection}) \end{tabular} & \begin{tabular}[c]{@{}c@{}} $\dfrac{2}{d(d+1)}$ \\ (\Cref{theo:gamma_sub_Mnatural_convex_poly}) \end{tabular} & \begin{tabular}[c]{@{}c@{}} $\sqrt{\dfrac{6}{d(d+1)(2d+1)}}$ \\ (\Cref{theo:gamma_sub_Mnatural_convex_poly}) \end{tabular} \\
\bottomrule
\end{tabular}
}
\end{table}

\section{Explicit upper bounds on the number of mistakes and the regret by problem class}
\label{sec:explicit_upper}

The upper bounds of \Cref{theo:ogd,theo:main,theo:metagrad_anytime_main} are nonincreasing in $\gamma$, so substituting the lower bounds of \Cref{tab:gamma_lower} yields explicit upper bounds by problem class.
The orders of the results are summarized in \Cref{tab:explicit_upper} (for the statements including constants, see Appendix~\ref{app:proof_corollaries}).

\begin{table}[ht]
\centering
\caption{Orders of the explicit upper bounds by problem class. The rows are indexed by the structure of the feasible set, the weight space $\Theta$, and the criterion, and the columns by the method. The rows for linear inequalities refer to constraints of the form $A x \leq b(s)$, with $C_g := g_\infty(\widetilde{A})$ (the same as in \Cref{tab:gamma_lower}); $\| M \|_2$ is the norm of the vector $M$ in \cref{eq:def_M}, $L$ is the constant in \cref{eq:def_DL}, and in the ILP rows the bound $L \leq \| M \|_2$ (\cref{eq:L_leq_m}) has been substituted. The values in the rows for $R^{\mathrm{est}}_T$ are obtained by substituting the lower bounds into the upper bounds on the sum $\widetilde{R}_T$ (item (iii) of each of \Cref{theo:ogd,theo:main,theo:metagrad_anytime_main}), and the cumulative suboptimality regret $R^{\mathrm{sub}}_T$ has the same upper bound (\cref{eq:max_leq_tilde}). The M-convex and M${}^\natural$-convex cases are combined since they have the same order. ONS and growing-grid SGS-MetaGrad are combined into one column since their orders coincide. The details of the substitution are given in Appendix~\ref{app:proof_corollaries}.}
\label{tab:explicit_upper}
{\footnotesize
\renewcommand{\arraystretch}{1.4}
\setlength{\tabcolsep}{3pt}
\begin{tabular}{lllcc}
\toprule
Structure & $\Theta$ & Criterion & SGS-OGD & \begin{tabular}[c]{@{}c@{}} ONS \\ Growing-grid SGS-MetaGrad \end{tabular} \\
\midrule
\multirow{4}{*}{\begin{tabular}[c]{@{}l@{}} General ILP \\ (\Cref{assu:ILP}) \end{tabular}}
 & \multirow{2}{*}{$\Delta^{d-1}$} & $K$ & $O(d^2 4^{d} \| M \|_2^{2d})$ & $O(d^{3} 2^{d} \| M \|_2^{d} \log (2\| M \|_2))$ \\
 & & $R^{\mathrm{est}}_T$ & $O(d\, 2^{d} \| M \|_2^{d+1})$ & $O(d^2 \| M \|_2 \log (2\| M \|_2))$ \\
\cmidrule(l){2-5}
 & \multirow{2}{*}{$B^d$} & $K$ & $O(4^{d} \| M \|_2^{2d})$ & $O(d^{2} 2^{d} \| M \|_2^{d} \log (2\| M \|_2))$ \\
 & & $R^{\mathrm{est}}_T$ & $O(2^{d} \| M \|_2^{d+1})$ & $O(d^2 \| M \|_2 \log (2\| M \|_2))$ \\
\midrule
\multirow{4}{*}{\begin{tabular}[c]{@{}l@{}} Linear inequalities \\ (\Cref{assu:linear-inequality}) \end{tabular}}
 & \multirow{2}{*}{$\Delta^{d-1}$} & $K$ & $O(L^2 d^2 (2C_g\sqrt{d})^{2d-2})$ & $O(L d^{3} (2C_g\sqrt{d})^{d-1} \log (2C_g dL))$ \\
 & & $R^{\mathrm{est}}_T$ & $O(L^2 d\, (2C_g\sqrt{d})^{d-1})$ & $O(L d^2 \log (2C_g dL))$ \\
\cmidrule(l){2-5}
 & \multirow{2}{*}{$B^d$} & $K$ & $O(L^2 (2C_g\sqrt{d})^{2d-2})$ & $O(L d^{2} (2C_g\sqrt{d})^{d-1} \log (2C_g dL))$ \\
 & & $R^{\mathrm{est}}_T$ & $O(L^2 (2C_g\sqrt{d})^{d-1})$ & $O(L d^2 \log (2C_g dL))$ \\
\midrule
\multirow{4}{*}{\begin{tabular}[c]{@{}l@{}} M-convex \\ (\Cref{assu:M-convex-projection}), \\ M${}^\natural$-convex \\ (\Cref{assu:Mnatural-convex-projection}) \end{tabular}}
 & \multirow{2}{*}{$\Delta^{d-1}$} & $K$ & $O(L^2 d^4)$ & $O(L d^{3} \log (2dL))$ \\
 & & $R^{\mathrm{est}}_T$ & $O(L^2 d^2)$ & $O(L d \log (2dL))$ \\
\cmidrule(l){2-5}
 & \multirow{2}{*}{$B^d$} & $K$ & $O(L^2 d^3)$ & $O(L d^{5/2} \log (2dL))$ \\
 & & $R^{\mathrm{est}}_T$ & $O(L^2 d^{3/2})$ & $O(L d \log (2dL))$ \\
\bottomrule
\end{tabular}
}
\end{table}

From \Cref{tab:explicit_upper} we read off the following two points, each a comparison with the existing methods.

First, even for a general ILP, the regret of ONS and of growing-grid SGS-MetaGrad is $O(d^2 \| M \|_2 \log (2\| M \|_2))$, which is independent of the total number of rounds $T$ and polynomial in the dimension.
Compared with the bound $O(\| M \|_2 d \log \frac{T}{d})$ of \citet{sakaue2025online} for the same setting, the factor $\log T$ disappears while the power of the dimension increases by one.
Among the bounds independent of $T$, it turns the bound $\exp(O(d \log d))$ of \citet[Theorem 4.2]{gollapudi2021contextual}, which assumes neither a margin nor a gap, into one polynomial in the dimension; this paper does assume a uniform margin, and its bound depends linearly on $\| M \|_2$.

Second, the regret $O(L d \log (2dL))$ for M-convex and M${}^\natural$-convex structures admits a direct comparison with the bound $O(d \log d)$ obtained by \citet{oki2026finite} under the same structure.
The latter is a value under the normalization that makes the per-round regret $O(1)$ \citep[Assumption 2.2]{oki2026finite}, which in the notation of this paper amounts to $\| \theta^* \| L = O(1)$.
Its bound on the number of mistake rounds does not depend on that normalization, so without it their bound becomes $O(\| \theta^* \| L d \log d)$, and the only difference from our bound is the argument of the logarithm.
The difference lies in the computational cost: whereas \citet{oki2026finite} computes a center of gravity at every round, our updates need only $O(d^2)$ plus one generalized projection onto $\Theta$ per mistake round.

\section{Conclusion}
\label{sec:conclusion}

For online inverse linear optimization, this paper has proposed a mechanism---small-gradient skipping (SGS)---that skips both the update of the iterate and the advancement of the index of the internal state at rounds without a mistake.
Under a uniform margin $\gamma > 0$, we have shown that the three methods obtained by applying SGS to OGD, ONS and MetaGrad bound the number of mistakes $K$, the cumulative suboptimality regret $R^{\mathrm{sub}}_T$, and the cumulative decision regret $R^{\mathrm{est}}_T$ all by quantities independent of the total number of rounds $T$ (\Cref{tab:main_results}).
Furthermore, by substituting the lower bounds on the margin by structure (\Cref{tab:gamma_lower}), we have given explicit upper bounds for situations in which the forward problem comes from an integer linear program.
In particular, for an ILP with the probability simplex, ONS and growing-grid SGS-MetaGrad attain $R^{\mathrm{est}}_T = O(d^2 \| M \|_2 \log(2\| M \|_2))$ (\Cref{tab:comparison}).
The problem raised in \Cref{sec:intro} was that every known bound independent of the total number of rounds has a limitation: the bound that assumes neither a margin nor a gap is exponential in the dimension \citep{gollapudi2021contextual}; the bound under a gap condition is proportional to the inverse square of the gap \citep{sakaue2025aistats}; and the bound under M-convexity, $O(d \log d)$, is smaller than ours but requires computing a center of gravity at every round \citep{oki2026finite}.
Our bound assumes only a uniform margin, is polynomial in the dimension, and is obtained with a deterministic and light update---$O(d^2)$ plus one generalized projection per mistake round. It also removes the $\log T$ dependence from the bound $O(\| M \|_2 d \log \frac{T}{d})$ of \citet{sakaue2025online}.

We list the remaining issues.
\begin{itemize}
    \item \textbf{The gap between the upper and lower bounds}: in the case of an ILP, there is a gap of a factor $d$ in the dimension, as well as a factor involving the coordinatewise ranges $\| M \|_2$, between our $R^{\mathrm{est}}_T = O(d^2 \| M \|_2 \log (2\| M \|_2))$ and the known lower bound $\Omega(d)$ \citep{sakaue2025online,oki2026finite}. Which of the two should be improved is an open question.
    \item \textbf{Extension to noise and corruption}: this paper is restricted to the noiseless setting. Frameworks that handle suboptimal feedback \citep{sakaue2025online} or corruption \citep{oki2026finite} have already been studied, and incorporating the SGS viewpoint into those analyses is an important direction for future work. Under corruption the per-mistake-round progress guaranteed by the uniform margin (\Cref{lem:r_bounds}) is weakened, so the treatment of the quadratic term in the analysis of ONS has to be replaced by a per-round inequality involving the amount of corruption.
\end{itemize}

\bibliographystyle{apalike}
\bibliography{suboptimality_loss}

\appendix

\newpage

\section{Related work in detail}
\label{app:related}

\paragraph{Finitely many updates under a margin condition: the classical line}
We describe in detail the classical line and the precedents of the skipping mechanism mentioned in \Cref{sec:related}.
The Perceptron convergence theorem for linearly separable data is a classical result showing that the number of mistakes (updates) is bounded by $(R/\gamma)^2$ in terms of the radius $R$ of the data and the margin $\gamma$; it shares the $\gamma^{-2}$-type structure of the upper bound with our bound $2L^2D^2/\gamma^2$ on the number of mistakes of SGS-OGD (\Cref{theo:ogd}).
ALMA~\citep{gentile2001new} is a method that approximates the maximum-margin classifier without being given the margin $\gamma$ explicitly, by means of the decaying step size $\eta_k \propto k^{-1/2}$; it shares its idea with the parameter-free step size design of this paper.
These are, however, results for binary classification (the 0/1 loss and linear surrogate losses), and they do not apply directly to the suboptimality loss treated in this paper (for an overview of the relation between online convex optimization and the Perceptron, see \citet{orabona2019modern}).
In the context of inverse optimization, \citet{sun2023maximum} (Maximum Optimality Margin) gives the skeleton ``separability $\to$ finitely many mistakes $\to$ exact recovery'' by a Perceptron-type method, and is the prior work closest to the framework of this paper.
The mechanism of not advancing the internal state in rounds without a mistake also has precedents. In the reduction from contextual recommendation to a cutting-plane algorithm of \citet[Theorem 3.1]{gollapudi2021contextual}, a round in which the proposal is correct is skipped: the state of the cutting-plane algorithm is reset to its state at the beginning of that round. \citet{besbes2021online,besbes2025contextual} also use a threshold-type skip, leaving the ellipsoidal cone unchanged in periods where the decision is nearly optimal, which they introduce in order to keep the ellipsoid method from becoming ill-conditioned.
The methods for which this paper formulates the mechanism are those whose internal state---the step-size index, the matrix $\Sigma_t$ of ONS, and the grid of MetaGrad---would otherwise advance even in rounds without a mistake.

\paragraph{Finite regret in online inverse optimization}
We now describe in detail how this paper relates to the four $T$-independent results listed in \Cref{sec:related}.
\citet{sakaue2025aistats} gives a finite regret of $O(1/\Delta^2)$ under the gap condition $\Delta > 0$, which is the closest to the uniform margin assumption of this paper. (For the definition of $\Delta$ and its rewriting in the notation of this paper, see Appendix~\ref{app:gap_entry}.)
\citet{oki2026finite} gives $R^{\mathrm{est}}_T = O(d \log d)$ under the M-convexity of the action set. Their method, however, computes a center of gravity at every round. The exact computation is \#P-hard, and although \citet{oki2026finite} also give a polynomial-time randomized implementation with approximate centers of gravity, that implementation guarantees the bound only in expectation and costs $O(d^6 \log d \log T)$ per round up to polylogarithmic factors arising from the random-walk implementation. M-convexity appears in this paper not as a requirement of the method but as a structural condition that makes the lower bound on the margin polynomial in the dimension (\Cref{tab:gamma_lower}). Consequently, our upper bound in the M-convex case is $O(L d \log (2dL))$ (\Cref{sec:explicit_upper}). \citet{oki2026finite} also give a deterministic warm-up algorithm that computes no center of gravity: it chooses the prediction by a topological sort of the index pairs whose order has been revealed by the observations, makes at most $\binom{d}{2}$ mistakes, and attains $R^{\mathrm{est}}_T = O(d^2)$ with $O(d^2)$ membership queries per round. Our bound improves its regret to $O(L d \log (2dL))$, whereas its bound on the number of mistakes is smaller than ours (\Cref{tab:explicit_upper}).
The bound $\exp(O(d \log d))$ of \citet[Theorem 4.2]{gollapudi2021contextual} is obtained through \citet[Theorem 3.1]{gollapudi2021contextual}, which reduces contextual recommendation to a cutting-plane algorithm. Its assumptions are weaker than ours in that it requires neither a margin nor a gap condition, but it is exponential in the dimension, and the authors themselves leave the true regret of that algorithm---in particular whether a polynomial dependence on the dimension is attainable---as an open question.
\citet{dewasurendra2026multiscale} gives a $T$-independent bound $O(d)$ for arbitrary compact action sets without a margin or a gap condition, together with a lower bound of order $d$. The method, however, does not output an estimate of $\theta^*$, and its cost is governed by covering numbers of the parameter ball: it keeps a weight for every element of covers at $\lceil \log_2 T \rceil$ scales, up to $\lceil \log_2 T \rceil (12T)^d$ elements in total, and each round chooses the action maximizing their weighted vote over $X(s)$. When $X(s)$ is a finite set listed explicitly, one round takes $O(|X(s)|\, d\, (12T)^d \log T)$ time, and $|X(s)|$ can be exponential in $d$ when $X(s)$ is given implicitly, for instance by integer linear constraints as in this paper; for an infinite $X(s)$, no efficient way of maximizing this vote is known in general, which the author leaves as an open problem. Our algorithms, in contrast, assume a uniform margin, output a weight $\hat{\theta}^t \in \Theta$, access $X(s)$ only through a linear optimization oracle, and cost $O(d^2)$ plus one generalized projection per mistake round.
\citet{fatemi2026tight} show that, in the online setting where the states are drawn i.i.d., the estimators that minimize an objective of a specific class over the weights consistent with the data suffer an expected cumulative regret of $\Omega(d \log T)$ in the worst case over $\theta^*$ and the distribution of the states. Since their estimator attains $O(d \log T)$ in this setting, and the adversarial upper bounds of \citet{gollapudi2021contextual} and \citet{sakaue2025online} are also $O(d \log T)$, this regret is tight in the worst case. Since our algorithms are not of this form, this lower bound does not contradict our $T$-independent bounds under a uniform margin.

\paragraph{Uniqueness of the optimal action}
For an ILP, if the optimal action for $\theta^*$ is not unique, then $\langle \theta^*, x - x' \rangle = 0$ for some $x \neq x'$, that is, $\theta^*$ lies on the hyperplane orthogonal to the nonzero integer vector $x - x'$, and there are only countably many such hyperplanes. Hence non-uniqueness occurs only on a set of measure zero in the weight space, with respect to the Lebesgue measure on $\mathbb{R}^d$ and also with respect to the $(d-1)$-dimensional measure on $\Delta^{d-1}$ (for a proof of the almost-everywhere uniqueness of the optimal solution on $\Delta^{d-1}$, see \citealp[Lemma~3.3]{kitaoka2024fast}). That is, \Cref{assu:margin}(3) holds for almost every $\theta^*$.

\paragraph{Uniform margins in offline inverse optimization}
The uniform margin assumption (\Cref{assu:margin}(4)) has a precedent in offline (batch) inverse optimization.
\citet{kitaoka2024fast} introduced a geometric constant $\gamma(\ell_\mathrm{sub}) > 0$ of the same kind for the suboptimality loss on a finite sample, and showed that the projected subgradient method reaches the minimum value $0$ of the loss in $O(1/\gamma(\ell_\mathrm{sub})^2)$ iterations.
Our upper bound $2L^2D^2/\gamma^2$ on the number of mistakes of SGS-OGD (\Cref{theo:ogd}) amounts to transferring this $\gamma^{-2}$-type dependence to the online setting.
\citet{kitaoka2026explicit} gives explicit lower bounds on this constant by test sets and Graver bases in the case where the forward problem is an ILP, and \Cref{sec:lower_bounds} applies that technique to the uniform margin.
On the other hand, neither bounding the regret by a constant independent of the total number of rounds $T$ nor bounding the number of mistakes $K$ (the number of rounds with $\hat{x}^t \neq x^t$) is new in itself. \citet[Corollary 3.9]{Barmann-2018-online} bound the number of rounds with $\hat{x}^t \neq x^t$ by $O(\sqrt{T})$ under a $\Delta$-stability condition (whenever the action $\hat{x}^t$ induced by a prediction differs from $x^t$, it is worse than $x^t$ by at least $\Delta > 0$ under the true weight $\theta^*$); the optimality-driven perceptron of \citet{sun2023maximum} bounds it by a quantity independent of $T$ under a separability condition; and \citet{oki2026finite} bound the number of rounds with nonzero regret by $O(d \log d)$ under M-convexity. What this paper adds is that a single condition---the existence of a witness $\bar{\theta} \in \Theta$ with a uniform margin---bounds $K$, $R^{\mathrm{sub}}_T$ and $R^{\mathrm{est}}_T$ simultaneously and independently of $T$, without assuming integrality of $\theta^*$ or M-convexity of the action set, and that the margin itself admits explicit lower bounds in terms of combinatorial structure (\Cref{sec:lower_bounds}).

\paragraph{Generality of the weight space}
The weight space $\Theta$ from which the learner chooses its predictions (\Cref{sec:setup}) is also treated differently in the existing work and in this paper. Many of the existing studies state their guarantees for a set $\Theta$ specific to the method: the unit sphere in \citet{besbes2021online,besbes2025contextual}, the unit ball in \citet{gollapudi2021contextual}, and the whole of $\mathbb{R}^d$ in \citet{oki2026finite} and \citet{sakaue2026simple} (\Cref{tab:metric_comparison}). By contrast, \citet{Barmann-2018-online}, \citet{sakaue2025aistats} and \citet{sakaue2025online} allow a general $\Theta$. As with the latter, the only condition we impose on $\Theta$ is that it be nonempty, bounded, closed and convex (\Cref{assu:margin}(1)), which covers both the probability simplex and the unit ball.

This generality is useful for our results. When the forward problem is a general ILP, the explicit lower bound on the margin holds unconditionally both if $\Theta$ is the probability simplex (\Cref{theo:gamma_sub_lowerbound_on_simplex}) and if $\Theta$ is the unit ball (\Cref{theo:gamma_sub_lowerbound_on_ball}). The two choices differ in the prior knowledge they encode: the probability simplex presupposes that $\theta^*$ is componentwise nonnegative, whereas the unit ball allows $\theta^*$ of arbitrary sign. Being able to choose $\Theta$ thus lets the explicit upper bounds of \Cref{tab:explicit_upper} cover both situations.

\section{Details of the algorithms}
\label{app:algorithms}

In this appendix we collect the pseudocode of the algorithms treated in \Cref{sec:main_results} together with the accompanying remarks.

\subsection{SGS-OGD}

\begin{algorithm}[ht]
    \caption{Small-gradient skipping online gradient descent (SGS-OGD)}
    \label{alg:ogd}
    \begin{algorithmic}[1]
        \REQUIRE step size coefficient $\alpha > 0$ (default $\alpha = D / (L \sqrt{2})$), initial point $\hat{\theta}^1 \in \Theta$
        \STATE $k \leftarrow 1$
        \FOR{$t = 1, \ldots, T$}
            \STATE receive $s^t$, compute and present with the oracle the proposal $\hat{x}^t \in \argmax_{x \in Y(s^t)} \langle \hat{\theta}^t, x \rangle$ (\cref{eq:proposal}), and observe $x^t$
            \IF{$\hat{x}^t \neq x^t$}
                \STATE $g^t \leftarrow \hat{x}^t - x^t$,
                $\hat{\theta}^{t+1} \leftarrow \Pi_{\Theta} \left( \hat{\theta}^t - \alpha k^{-1/2} g^t \right)$,
                $k \leftarrow k + 1$
            \ELSE
                \STATE $\hat{\theta}^{t+1} \leftarrow \hat{\theta}^t$
            \ENDIF
        \ENDFOR
    \end{algorithmic}
\end{algorithm}

Here $\Pi_\Theta(y) := \argmin_{\theta \in \Theta} \| \theta - y \|$ is the Euclidean projection, which is uniquely determined since $\Theta$ is nonempty, closed and convex.

\subsection{ONS}

The method obtained by applying ONS \citep{hazan2007logarithmic} to online inverse linear optimization is shown in \Cref{alg:ons_generic} \citep[cf.][]{sakaue2025online}.
It runs ONS on the exp-concave surrogate loss
\begin{equation}
    \ell^{\eta}_t (\theta) := - \eta \langle \hat{\theta}^t - \theta,\, g^t \rangle + \eta^2 \langle \hat{\theta}^t - \theta,\, g^t \rangle^2 ,
    \qquad
    g^t := \hat{x}^t - x^t
    \label{eq:surrogate_ons}
\end{equation}
for a learning rate $\eta > 0$, using the fact that its gradient at $\theta = \hat{\theta}^t$ is $\nabla \ell^{\eta}_t (\hat{\theta}^t) = \eta\, g^t$.
Below, for symmetric matrices $\Sigma, \Sigma'$ of order $d$, $\Sigma \succeq \Sigma'$ means that $\Sigma - \Sigma'$ is positive semidefinite and $\Sigma \succ \Sigma'$ that $\Sigma - \Sigma'$ is positive definite (the Loewner order); in particular $\Sigma \succ 0$ means that $\Sigma$ is positive definite. We also write $\Id_d$ for the identity matrix of order $d$ (and likewise $\Id_n$, $\Id_N$ for other orders).

\begin{algorithm}[ht]
    \caption{Online Newton Step (ONS) for online inverse linear optimization}
    \label{alg:ons_generic}
    \begin{algorithmic}[1]
        \REQUIRE weight space $\Theta$ (\Cref{assu:margin}), initial point $\hat{\theta}^1 \in \Theta$, learning rate $\eta > 0$, parameter $\kappa > 0$, positive definite matrix $\Sigma_0 \succ 0$
        \FOR{$t = 1, \ldots, T$}
            \STATE receive $s^t$, compute and present with the oracle the proposal $\hat{x}^t \in \argmax_{x \in Y(s^t)} \langle \hat{\theta}^t, x \rangle$ (\cref{eq:proposal}), and observe $x^t$
            \STATE $g^t \leftarrow \hat{x}^t - x^t$, $\nabla^t \leftarrow \eta\, g^t$
            \hfill{$\triangleright$ gradient of the surrogate loss \cref{eq:surrogate_ons}}
            \STATE $\Sigma_t \leftarrow \Sigma_{t-1} + \nabla^t (\nabla^t)^\top$
            \STATE $\hat{\theta}^{t+1} \leftarrow \Pi^{\Sigma_t}_{\Theta} \left( \hat{\theta}^t - \frac{1}{\kappa} \Sigma_t^{-1} \nabla^t \right)$
            \hfill{$\triangleright$ generalized projection}
        \ENDFOR
    \end{algorithmic}
\end{algorithm}

\begin{algorithm}[ht]
    \caption{Online Newton Step (ONS)}
    \label{alg:ons}
    \begin{algorithmic}[1]
        \REQUIRE geometric constants $L, D$ (\cref{eq:def_DL}), initial point $\hat{\theta}^1 \in \Theta$
        \STATE $k \leftarrow 1$, $\eta \leftarrow \frac{1}{LD}$, $\Sigma_0 \leftarrow D^{-2} \Id_d$
        \FOR{$t = 1, \ldots, T$}
            \STATE receive $s^t$, compute and present with the oracle the proposal $\hat{x}^t \in \argmax_{x \in Y(s^t)} \langle \hat{\theta}^t, x \rangle$ (\cref{eq:proposal}), and observe $x^t$
            \IF{$\hat{x}^t \neq x^t$ (a mistake)}
                \STATE $g^t \leftarrow \hat{x}^t - x^t$, $\nabla^t \leftarrow \eta\, g^t$,
                $\Sigma_k \leftarrow \Sigma_{k-1} + \nabla^t (\nabla^t)^\top$
                \STATE $\hat{\theta}^{t+1} \leftarrow \Pi^{\Sigma_k}_{\Theta} \left( \hat{\theta}^t - \Sigma_k^{-1} \nabla^t \right)$,
                $k \leftarrow k + 1$
            \ELSE
                \STATE $\hat{\theta}^{t+1} \leftarrow \hat{\theta}^t$
            \ENDIF
        \ENDFOR
    \end{algorithmic}
\end{algorithm}

Here $\Pi^{\Sigma}_{\Theta}(y) := \argmin_{\theta \in \Theta} \| \theta - y \|_\Sigma^2$ (where $\|x\|_\Sigma^2 := x^\top \Sigma x$) is the generalized projection with respect to the $\Sigma$-norm (a convex quadratic program over $\Theta$), which is uniquely determined since $\Sigma \succ 0$ and $\Theta$ is nonempty, closed and convex.
The matrix $\Sigma_k \succeq D^{-2} \Id_d \succ 0$ is always positive definite.
The inverse $\Sigma_k^{-1}$ can be updated by a rank-one update via the Sherman--Morrison formula, and one update costs $O(d^2)$.
The parameters of the algorithm are only $L$ and $D$; no knowledge of the margin $\gamma$ or of the total number of rounds $T$ is required.

When the loss is $\alpha$-exp-concave and satisfies $\max_{\theta \in \Theta} | \langle \nabla \ell^{\eta}_t (\hat{\theta}^t), \theta - \hat{\theta}^t \rangle | \leq \beta$, the standard choice for ONS is
$\kappa = \frac{1}{2} \min \{ \frac{1}{\beta},\, \alpha \}$, $\Sigma_0 = \frac{d}{D^2 \kappa^2} \Id_d$.
The $\eta$-experts of MetaGrad (\Cref{defi:eta_expert}) follow this choice (there, since the center $\hat{\theta}^t$ of the surrogate loss differs from the point of the expert, the gradient acquires a factor $1 - 2\eta \langle g^t, \hat{\theta}^t - \theta \rangle$).
By contrast, \Cref{alg:ons} uses the same update formula with the different choice $\eta = \frac{1}{LD}$, $\kappa = 1$, $\Sigma_0 = D^{-2} \Id_d$ (and in addition omits the computation at rounds without a mistake).
The validity of this choice is shown directly by the potential inequality of Appendix~\ref{app:proof_ons}, without going through the standard regret bound.

\begin{remark}[\Cref{alg:ons} is ONS itself]
    \label{rem:ons_is_ons}
    The branching in \Cref{alg:ons} is there to make explicit that the computation at rounds without a mistake is omitted; it does not change the sequence of iterates.
    Indeed, at a round without a mistake we have $g^t = \hat{x}^t - x^t = 0$, that is, $\nabla^t = 0$, so the information matrix is unchanged, $\Sigma \leftarrow \Sigma + \nabla^t (\nabla^t)^\top = \Sigma$, and the update becomes
    $\Pi^{\Sigma}_{\Theta} ( \hat{\theta}^t - \Sigma^{-1} \cdot 0 ) = \Pi^{\Sigma}_{\Theta} ( \hat{\theta}^t ) = \hat{\theta}^t$
    (the projection is the identity since $\hat{\theta}^t \in \Theta$).
    That is, \Cref{alg:ons} generates the same sequence of iterates as plain ONS \citep{hazan2007logarithmic} using the subgradient $g^t = \hat{x}^t - x^t$ at every round.
    Consequently the contribution of this subsection is not the proposal of an algorithm, but the improvement of the existing guarantee for ONS by incorporating the SGS viewpoint (\Cref{rem:sgs_algorithmic}) into the \emph{analysis} of ONS: evaluating the log-det potential by the number of mistakes $K$ rather than by the total number of rounds $T$, and balancing it against the per-mistake progress guaranteed by the uniform margin, replaces the regret upper bound $O(d \log T)$ by the $T$-independent \Cref{theo:main}.
    In implementation, the advantage remains that the $O(d^2)$ matrix update and the generalized projection can be omitted at rounds without a mistake (\Cref{tab:complexity}).
\end{remark}

\begin{remark}[Relation to \citet{sakaue2025online}]
    \label{rem:sakaue_alg}
    \Cref{alg:ons} uses the same update rule as the construction of \citet[Theorem 3.1]{sakaue2025online} applying ONS to the exp-concave surrogate loss
    $\ell^{\eta}_k(\theta) = -\eta \langle \hat{\theta}^{t_k} - \theta, g^{t_k} \rangle + \eta^2 \langle \hat{\theta}^{t_k} - \theta, g^{t_k} \rangle^2$ (where $t_k$ is the $k$-th mistake round),
    with the different choice of parameters $\eta = 1/(LD)$, $\kappa = 1$, $\Sigma_0 = D^{-2} \Id_d$ (\citet{sakaue2025online} take $\eta = 1/(5LD)$ in our notation and the standard choice of $\kappa$ and $\Sigma_0$).
    Apart from this choice of parameters, there is no difference in the algorithm (\Cref{rem:ons_is_ons}); the difference is on the side of the guarantee: that paper shows $O(LDd\log\frac{T}{d})$ without assuming a margin, whereas this paper shows a $T$-independent number of mistakes and cumulative suboptimality regret under a uniform margin.
    Our analysis does not use the surrogate loss explicitly, but proves the same content directly as a quadratic potential inequality (Appendix~\ref{app:proof_ons}).
\end{remark}

\subsection{MetaGrad}

MetaGrad is a universal online learning method that runs in parallel the $\eta$-experts (which apply ONS to the surrogate loss $\ell^{\eta}_k$ of \Cref{rem:sakaue_alg}) for each $\eta$ in a learning-rate grid $\mathcal{E}$; it consists of two layers, these $\eta$-experts and a \textbf{master}.
Here the master is the algorithm that, at each round $j$, updates the weight $p^{\eta}_j$ attached to the $\eta$-expert by exponential weighting with respect to the surrogate losses and outputs the weighted average of the points $w^{\eta}_j$ of the $\eta$-experts, weighted by the learning rates,
\begin{equation}
    w_j := \frac{\sum_{\eta \in \mathcal{E}} \eta\, p^{\eta}_j\, w^{\eta}_j}{\sum_{\eta \in \mathcal{E}} \eta\, p^{\eta}_j} ;
    \label{eq:master_point}
\end{equation}
we call $w_j$ the \textbf{point of the master}.
The concrete forms of the grid $\mathcal{E}$, the weights $p^{\eta}_j$ and the points $w^{\eta}_j$ are given in \Cref{alg:metagrad_generic}.
\Cref{alg:metagrad_generic} shows MetaGrad for a general sequence of convex losses (a restatement of Algorithm 2 of \citealp{sakaue2025online}; the grid is constructed from $\bar{m}$), and \Cref{alg:metagrad} shows its SGS version.
The ONS used by the $\eta$-experts is made concrete for the surrogate loss as follows.

\begin{definition}[The ONS of an $\eta$-expert; following Appendix B of \citealp{sakaue2025online}]
    \label{defi:eta_expert}
    Let $\mathcal{W} \subset \mathbb{R}^n$ be a nonempty closed convex set whose $\ell_2$ diameter is at most $W > 0$, let $G, H > 0$, and take a learning rate $\eta \in ( 0, \frac{1}{5H} ]$.
    For the surrogate loss
    $\ell^{\eta}_j (w) = - \eta \langle w_j - w, g_j \rangle + \eta^2 \langle w_j - w, g_j \rangle^2$
    associated with the point $w_j \in \mathcal{W}$ of the master and a subgradient $g_j$ (with $\| g_j \| \leq G$ and $\sup \{ \langle w' - w, g_j \rangle \mid w, w' \in \mathcal{W} \} \leq H$), the \textbf{$\eta$-expert} is the ONS that starts from an initial point $w^{\eta}_1 \in \mathcal{W}$ and updates
    \begin{align}
        \nabla^{\eta}_j & := \nabla \ell^{\eta}_j (w^{\eta}_j)
        = \eta \left( 1 - 2 \eta \langle g_j,\, w_j - w^{\eta}_j \rangle \right) g_j ,
        \label{eq:expert_grad} \\
        \Sigma^{\eta}_j & := \Sigma^{\eta}_{j-1} + \nabla^{\eta}_j (\nabla^{\eta}_j)^\top ,
        \qquad
        \Sigma^{\eta}_0 := \frac{n}{W^2 \kappa_\eta^2} \Id_n ,
        \label{eq:expert_matrix} \\
        w^{\eta}_{j+1} & := \Pi^{\Sigma^{\eta}_j}_{\mathcal{W}} \left( w^{\eta}_j - \frac{1}{\kappa_\eta} ( \Sigma^{\eta}_j )^{-1} \nabla^{\eta}_j \right) ,
        \qquad
        \kappa_\eta := \frac{1}{( 1 + 2 \eta H )^2} .
        \label{eq:expert_update}
    \end{align}
    The only difference from \Cref{alg:ons_generic} is that, since the center $w_j$ of the surrogate loss differs from the point $w^{\eta}_j$ being updated, the gradient \cref{eq:expert_grad} acquires the factor $1 - 2\eta \langle g_j, w_j - w^{\eta}_j \rangle$; the parameters follow the standard choice $\kappa = \kappa_\eta$, $\Sigma^{\eta}_0 = \frac{n}{W^2 \kappa_\eta^2} \Id_n$. Here $\Pi^{\Sigma}_{\mathcal{W}}(y) := \argmin_{w \in \mathcal{W}} \| w - y \|_\Sigma^2$ is the same generalized projection with respect to the $\Sigma$-norm as in \Cref{alg:ons}.
\end{definition}

The origin of the parameter $\kappa_\eta$ is as follows. The standard form of ONS sets $\kappa = \frac{1}{2} \min \{ \frac{1}{\beta}, \alpha \}$ and $\Sigma_0 = \frac{n}{W^2\kappa^2} \Id_n$ from the exp-concavity constant $\alpha$ and an upper bound $\beta$ on the inner product with the gradient \citep[cf.][]{hazan2007logarithmic}.
For the surrogate loss, from $\nabla^2 \ell^{\eta}_j (w) = 2\eta^2 g_j g_j^\top$ and \cref{eq:expert_grad} we have
$\nabla \ell^{\eta}_j (w) \nabla \ell^{\eta}_j (w)^\top \preceq \eta^2 (1 + 2\eta H)^2 g_j g_j^\top = \frac{(1+2\eta H)^2}{2} \nabla^2 \ell^{\eta}_j (w)$,
so it is $\alpha = \frac{2}{(1 + 2\eta H)^2}$-exp-concave, and
$\max_{w \in \mathcal{W}} | \langle \nabla \ell^{\eta}_j (w^{\eta}_j), w - w^{\eta}_j \rangle | \leq \beta := \eta H + 2 \eta^2 H^2$,
$\| \nabla \ell^{\eta}_j (w) \| \leq \lambda := \eta ( 1 + 2\eta H ) G$ hold.
Since $\frac{1}{\alpha} = \frac{1}{2} + 2\eta H + 2\eta^2 H^2 \geq \beta$, we get $\kappa = \frac{\alpha}{2} = \kappa_\eta$, and under $\eta \leq \frac{1}{5H}$ we have $\kappa_\eta \in [ \frac{25}{49}, 1 )$ and $\kappa_\eta \lambda = \frac{\eta G}{1 + 2\eta H} \leq \frac{G}{7 H}$ (this $\frac{1}{49}$ is the origin of the denominator $49$ in \cref{eq:metagrad_lambda}).
In the SGS version, as in \Cref{alg:ons}, the update of the experts, the update of the weights, and the advancement of the index $k$ are restricted to mistake rounds only.
Here too the skipping holds automatically, and the substantial difference from plain MetaGrad is limited to the construction of the learning-rate grid (\Cref{rem:metagrad_is_metagrad}): building the grid from the side of the number of mistakes rather than the total number of rounds is what makes a $T$-independent guarantee possible.
The construction of the grid uses an upper bound $\bar{K} \geq K$ on the number of mistakes (since there are at most $T$ mistakes, $\bar{K} = T$ is admissible; that the dependence stays at $\log \log \bar{K}$ is stated in \Cref{rem:metagrad_comparison}).
This dependence on $\bar{K}$ can be removed by growing the grid according to the number of mistakes (\Cref{theo:metagrad_anytime_main}).

\begin{algorithm}[ht]
    \caption{MetaGrad (the version whose learning-rate grid is constructed from $\bar{m}$)}
    \label{alg:metagrad_generic}
    \begin{algorithmic}[1]
        \REQUIRE nonempty closed convex set $\mathcal{W} \subset \mathbb{R}^n$, constants $W, H > 0$, an upper bound $\bar{m}$ on the number of rounds $m$; the convex losses $h_1, \ldots, h_m \colon \mathcal{W} \to \mathbb{R}$ are given online (the ONS of an $\eta$-expert is \Cref{defi:eta_expert})
        \STATE learning-rate grid $\mathcal{E} \leftarrow \left\{ \eta_i := \frac{2^{-i}}{5H} \;\middle|\; i = 0, 1, \ldots, \left\lceil \frac{1}{2} \log_2 \bar{m} \right\rceil \right\}$
        \STATE for each $\eta_i \in \mathcal{E}$, prepare an initial weight $p^{\eta_i}_1 \leftarrow \frac{C}{(i+1)(i+2)}$ (where $C$ is the normalizing constant making $\sum_{\eta \in \mathcal{E}} p^{\eta}_1 = 1$) and an initial point $w^{\eta_i}_1 \in \mathcal{W}$ of the $\eta_i$-expert
        \FOR{$j = 1, \ldots, m$}
            \STATE output $w_j \leftarrow \sum_{\eta \in \mathcal{E}} \eta\, p^{\eta}_j w^{\eta}_j \,/ \sum_{\eta \in \mathcal{E}} \eta\, p^{\eta}_j$ and observe a subgradient $g_j \in \partial h_j (w_j)$
            \STATE define the surrogate losses $\ell^{\eta}_j (w) := - \eta \langle w_j - w,\, g_j \rangle + \eta^2 \langle w_j - w,\, g_j \rangle^2$ ($\eta \in \mathcal{E}$)
            \STATE for each $\eta \in \mathcal{E}$: $p^{\eta}_{j+1} \leftarrow p^{\eta}_j \exp ( - \ell^{\eta}_j (w^{\eta}_j) ) / Z_j$ (where $Z_j := \sum_{\eta' \in \mathcal{E}} p^{\eta'}_j \exp ( - \ell^{\eta'}_j (w^{\eta'}_j) )$), and compute $w^{\eta}_{j+1}$ by the ONS update of the $\eta$-expert on $\ell^{\eta}_j$
        \ENDFOR
    \end{algorithmic}
\end{algorithm}

\begin{algorithm}[ht]
    \caption{Small-gradient skipping MetaGrad (SGS-MetaGrad)}
    \label{alg:metagrad}
    \begin{algorithmic}[1]
        \REQUIRE geometric constants $L, D$ (\cref{eq:def_DL}), an upper bound $\bar{K}$ on the number of mistakes (for instance $\bar{K} = T$)
        \STATE learning-rate grid $\mathcal{E} \leftarrow \left\{ \eta_i := \frac{2^{-i}}{5LD} \;\middle|\; i = 0, 1, \ldots, \left\lceil \frac{1}{2} \log_2 \bar{K} \right\rceil \right\}$
        \STATE for each $\eta_i \in \mathcal{E}$, prepare an initial weight $p^{\eta_i}_1 \leftarrow \frac{C}{(i+1)(i+2)}$ (where $C$ is the normalizing constant making $\sum_{\eta \in \mathcal{E}} p^{\eta}_1 = 1$) and an initial point $\theta^{\eta_i}_1 \in \Theta$ of the $\eta_i$-expert (\Cref{defi:eta_expert})
        \STATE $k \leftarrow 1$, $\hat{\theta}^1 \leftarrow \sum_{\eta \in \mathcal{E}} \eta\, p^{\eta}_1 \theta^{\eta}_1 \,/ \sum_{\eta \in \mathcal{E}} \eta\, p^{\eta}_1$
        \FOR{$t = 1, \ldots, T$}
            \STATE receive $s^t$, compute and present with the oracle the proposal $\hat{x}^t \in \argmax_{x \in Y(s^t)} \langle \hat{\theta}^t, x \rangle$ (\cref{eq:proposal}), and observe $x^t$
            \IF{$\hat{x}^t \neq x^t$ (a mistake)}
                \STATE $g^t \leftarrow \hat{x}^t - x^t$,
                $\ell^{\eta}_k (\theta) := - \eta \langle \hat{\theta}^t - \theta,\, g^t \rangle + \eta^2 \langle \hat{\theta}^t - \theta,\, g^t \rangle^2$ ($\eta \in \mathcal{E}$)
                \STATE for each $\eta \in \mathcal{E}$: $p^{\eta}_{k+1} \leftarrow p^{\eta}_k \exp ( - \ell^{\eta}_k (\theta^{\eta}_k) ) / Z_k$ (where $Z_k := \sum_{\eta' \in \mathcal{E}} p^{\eta'}_k \exp ( - \ell^{\eta'}_k (\theta^{\eta'}_k) )$), and compute $\theta^{\eta}_{k+1}$ by the ONS update of the $\eta$-expert on $\ell^{\eta}_k$
                \STATE $\hat{\theta}^{t+1} \leftarrow \sum_{\eta \in \mathcal{E}} \eta\, p^{\eta}_{k+1} \theta^{\eta}_{k+1} \,/ \sum_{\eta \in \mathcal{E}} \eta\, p^{\eta}_{k+1}$,
                $k \leftarrow k + 1$
            \ELSE
                \STATE $\hat{\theta}^{t+1} \leftarrow \hat{\theta}^t$
            \ENDIF
        \ENDFOR
    \end{algorithmic}
\end{algorithm}

Every iterate of \Cref{alg:metagrad,alg:metagrad_anytime} stays in $\Theta$.
Indeed, each $\eta$-expert starts at a point of $\Theta$ and is updated by \cref{eq:expert_update}, whose generalized projection $\Pi^{\Sigma^{\eta}_j}_{\mathcal{W}}$ maps into $\mathcal{W} = \Theta$, so that $\theta^{\eta}_k \in \Theta$;
and the output of the master is the weighted average with coefficients $\eta\, p^{\eta} / \sum_{\eta'} \eta'\, p^{\eta'}$, which are nonnegative and sum to one, so $\hat{\theta}^t$ is a convex combination of points of $\Theta$ and hence $\hat{\theta}^t \in \Theta$ by \Cref{assu:margin}(1).
In particular the hypothesis $w_1, \ldots, w_m \in \mathcal{W}$ of \Cref{prop:metagrad_regret} is satisfied.

\begin{remark}[\Cref{alg:metagrad} is MetaGrad itself]
    \label{rem:metagrad_is_metagrad}
    The branching in \Cref{alg:metagrad} does not change the sequence of iterates either, for the same reason as in \Cref{rem:ons_is_ons}.
    At a round without a mistake we have $g^t = \hat{x}^t - x^t = 0$, so the surrogate losses become $\ell^{\eta}_k \equiv 0$, the weights are unchanged, $p^{\eta} \exp(0) = p^{\eta}$, and the gradient \cref{eq:expert_grad} of an expert is also $\nabla^{\eta} = 0$, so neither $\Sigma^{\eta}$ nor the point of the expert moves.
    Hence the output of the master does not change either.
    That is, \Cref{alg:metagrad} with the grid upper bound taken as $\bar{K} = T$ generates the same sequence of iterates as plain MetaGrad using the subgradient $g^t = \hat{x}^t - x^t$ at every round.
    The claim of this subsection is likewise not the proposal of a new algorithm, but the improvement of the guarantee of MetaGrad under a margin by incorporating the SGS viewpoint into the \emph{analysis} of MetaGrad: evaluating the variance term appearing in the regret upper bound of MetaGrad only at mistake rounds and balancing it, by self-bounding, against the per-mistake progress guaranteed by the uniform margin replaces the $O(LDd\log\frac{T}{d})$ of \citet{sakaue2025online} by the $T$-independent \Cref{theo:metagrad_main}.
The difference between \Cref{alg:metagrad} and plain MetaGrad is limited to two points: (i) the implementation advantage that the expert updates (the grid size times ($O(d^2)$ plus a generalized projection)) can be omitted at rounds without a mistake, and (ii) that if an upper bound $\bar{K} < T$ on the number of mistakes is known then the grid (and hence the number of experts) can be taken smaller.
    This identity, however, concerns only the fixed-grid version \Cref{alg:metagrad} and does not extend to the growing-grid version \Cref{alg:metagrad_anytime} (\Cref{rem:metagrad_anytime}).
\end{remark}

\subsection{Growing-grid SGS-MetaGrad}

\paragraph{Removing the dependence on $\bar{K}$ by a growing grid}
The factor $c_0$ in \Cref{rem:metagrad_comparison}(b) originates from fixing the grid $\mathcal{E}$ in advance by an upper bound $\bar{K}$ on the number of mistakes.
If we do not fix the grid but keep adding smaller learning rates as the number of mistakes $k$ progresses, then the input $\bar{K}$ itself becomes unnecessary and $c_0$ is replaced by $c_0(K) := 2 \log ( \frac{1}{2} \log_2 K + 3 )$ with the realized number of mistakes $K$.
The self-bounding of the number of mistakes closes without any dependence on $T$ under this replacement as well, and the number of mistakes and the cumulative suboptimality regret become constants that are completely independent of $T$ (\Cref{theo:metagrad_anytime_main}).
\Cref{alg:metagrad_anytime} shows the growing-grid version. It differs from \Cref{alg:metagrad} in the following three points:
(i) the prior weights are taken as $p_i = \frac{1}{(i+1)(i+2)}$ on the countable grid $\{ \eta_i = \frac{2^{-i}}{5LD} \mid i \in \mathbb{Z}_{\geq 0} \}$ (since $p_i = \frac{1}{i+1} - \frac{1}{i+2}$ gives $\sum_{i \geq 0} p_i = 1$, no normalizing constant is needed);
(ii) the $\eta_i$-expert is created from the $(4^{i-1} + 1)$-st update on (since $\lceil \frac{1}{2} \log_2 k \rceil \geq i \iff k \geq 4^{i-1} + 1$, the grid created so far always coincides with the grid of \Cref{alg:metagrad} with $\bar{K} = k$);
(iii) the weights are kept unnormalized (the output of the master is determined by the ratios of the weights alone).
The freezing property of \Cref{defi:sgs} is preserved: at a round where no mistake occurs, neither the grid, nor the weights, nor any expert changes at all.

\begin{algorithm}[ht]
    \caption{Growing-grid SGS-MetaGrad}
    \label{alg:metagrad_anytime}
    \begin{algorithmic}[1]
        \REQUIRE geometric constants $L, D$ (\cref{eq:def_DL})
        \STATE $k \leftarrow 1$, $I \leftarrow 0$, $\eta_0 \leftarrow \frac{1}{5LD}$, unnormalized weight $\tilde{p}^{\eta_0} \leftarrow \frac{1}{2}$, an initial point $\theta^{\eta_0} \in \Theta$ of the $\eta_0$-expert (\Cref{defi:eta_expert}), $\hat{\theta}^1 \leftarrow \theta^{\eta_0}$
        \FOR{$t = 1, \ldots, T$}
            \STATE receive $s^t$, compute and present with the oracle the proposal $\hat{x}^t \in \argmax_{x \in Y(s^t)} \langle \hat{\theta}^t, x \rangle$ (\cref{eq:proposal}), and observe $x^t$
            \IF{$\hat{x}^t \neq x^t$ (a mistake)}
                \STATE $g^t \leftarrow \hat{x}^t - x^t$,
                $\ell^{\eta}_k (\theta) := - \eta \langle \hat{\theta}^t - \theta,\, g^t \rangle + \eta^2 \langle \hat{\theta}^t - \theta,\, g^t \rangle^2$ ($\eta \in \{ \eta_0, \ldots, \eta_I \}$)
                \STATE for each $\eta \in \{ \eta_0, \ldots, \eta_I \}$: $\tilde{p}^{\eta} \leftarrow \tilde{p}^{\eta} \exp ( - \ell^{\eta}_k (\theta^{\eta}) )$, and update $\theta^{\eta}$ by the ONS update of the $\eta$-expert on $\ell^{\eta}_k$
                \STATE $k \leftarrow k + 1$
                \IF{$\lceil \frac{1}{2} \log_2 k \rceil > I$}
                    \STATE $I \leftarrow I + 1$, $\eta_I \leftarrow \frac{2^{-I}}{5LD}$, $\tilde{p}^{\eta_I} \leftarrow \frac{1}{(I+1)(I+2)}$, prepare an initial point $\theta^{\eta_I} \in \Theta$ of the $\eta_I$-expert (arbitrary; for instance $\hat{\theta}^t$)
                \ENDIF
                \STATE $\hat{\theta}^{t+1} \leftarrow \sum_{i=0}^{I} \eta_i\, \tilde{p}^{\eta_i} \theta^{\eta_i} \,/ \sum_{i=0}^{I} \eta_i\, \tilde{p}^{\eta_i}$
            \ELSE
                \STATE $\hat{\theta}^{t+1} \leftarrow \hat{\theta}^t$
            \ENDIF
        \ENDFOR
    \end{algorithmic}
\end{algorithm}

In \Cref{alg:metagrad,alg:metagrad_anytime}, the $\eta$-experts of \Cref{defi:eta_expert} are instantiated with $n = d$, $\mathcal{W} = \Theta$, $W = D$ and $H = LD$ (the constant $G = L$ enters only the analysis), so that $\Sigma^{\eta}_0 = \frac{d}{D^2 \kappa_\eta^2} \Id_d$ and $\kappa_\eta = (1 + 2\eta L D)^{-2}$.
In an implementation of \Cref{alg:metagrad_anytime}, it is advisable to store the logarithms of the unnormalized weights and to compute the output of the master after subtracting their maximum (the log-sum-exp trick). Since $\eta \leq \frac{1}{5LD}$, each update changes a log-weight by at most $\frac{1}{5} + \frac{1}{25}$ in absolute value, so this is a safeguard against underflow over many mistake rounds rather than a change of the algorithm.

\begin{remark}[The place of the growing-grid version]
    \label{rem:metagrad_anytime}
Unlike \Cref{alg:metagrad}, \Cref{alg:metagrad_anytime} does not coincide with the plain anytime version of MetaGrad \citep{van2021metagrad}.
    The mechanism of growing the grid during execution is itself of the same kind, but this paper refines the grid by the number of mistakes $k$ rather than by the round $t$: if it were refined by the round $t$, a new expert would be created even at rounds without a mistake and the internal state would change, so the freezing property of \Cref{defi:sgs} would break, and the grid size would swell to $O(\log T)$, leaving $c_0 = O(\log\log T)$.
    Moreover, since an expert created earlier receives updates and weightings at the subsequent mistake rounds, the sequences of outputs of the master themselves generally differ between refinement by $t$ and refinement by $k$.
    Therefore the growing-grid version of this subsection is not an ``improvement of the analysis of an existing algorithm'' in the sense of \Cref{rem:metagrad_is_metagrad}, but falls under the case where SGS actually changes the algorithm (\Cref{rem:sgs_algorithmic}).
    The point of \Cref{prop:metagrad_regret_anytime,theo:metagrad_anytime_main} lies in this refinement by $k$ and in making its constants explicit.
    Compared with \Cref{theo:metagrad_main}, the argument of $c_0$ changes from $\bar{K}$ to the realized value $K$, which makes the guarantee completely independent of $T$, while the coefficient of the logarithmic term changes hardly at all, from $\frac{152}{3} \approx 50.7$ to $52$.
    The computational cost per mistake round is proportional to the number $1 + \lceil \frac{1}{2} \log_2 K \rceil$ of created experts, which is also independent of $T$ (\Cref{tab:complexity}).
    The parameters of the algorithm are only $L$ and $D$. To dispense even with the knowledge of $L$, the following Lipschitz-adaptive version is needed.
\end{remark}

\begin{remark}[Parameter adaptation by the refined version of MetaGrad]
    \label{rem:metagrad_adaptive}
    The Lipschitz-adaptive anytime version of MetaGrad of \citet[Algorithms 1 and 2]{van2021metagrad} requires no prior knowledge of $G, H$ or of the number of rounds and operates using only (a guess of) the value of $W$; according to \citet[Appendix B.4]{sakaue2025online} it attains, in the setting of \Cref{prop:metagrad_regret} and for every $u \in \mathcal{W}$ (with $V^u_m := \sum_{j=1}^m \langle w_j - u, g_j \rangle^2$ as in \cref{eq:metagrad_regret}),
    \[
        \sum_{j=1}^m \langle w_j - u, g_j \rangle
        = O \left( \sqrt{ n \log \left( \frac{W G m}{n} \right) \cdot V^u_m } + H n \log \left( \frac{W G m}{n} \right) \right)
    \]
    (note that the argument of this logarithm is not scale invariant; this comes from the normalization of the source).
    Applying this to the subsequence of mistake rounds, the self-bounding and the application of the transcendental inequality (\Cref{lem:transcendental}) in the proof of \Cref{theo:metagrad_main} go through as they are, and, allowing constants in $O(\cdot)$ form, \Cref{theo:metagrad_anytime_main} holds in the same order up to the argument of the logarithm changing from $K/d$ to $D L K / d$ (an addition of the order of $\log (1 + LD)$).
    In that case the only prior knowledge needed is $D$, and $L$ and $\bar{K}$ are unnecessary. Since $\Theta$ is a set designed by the learner itself, $D = \diam \Theta$ is known, so the only substantially unknown parameter was $L$.
    The explicit constants are not tracked, since the upper bound of the source is in $O(\cdot)$ form.
\end{remark}

Finally, \Cref{tab:complexity} summarizes the per-round computational cost of each method. Here $\tau_{\mathrm{solve}}$ is the time for one linear optimization that computes the proposal $\hat{x}^t$, $\tau_{\mathrm{E\text{-}proj}}$ / $\tau_{\mathrm{G\text{-}proj}}$ is the time for one Euclidean / generalized projection onto $\Theta$, and $K$ is the number of mistakes (the value for growing-grid SGS-MetaGrad is that of \Cref{alg:metagrad_anytime}). The column $\Theta$ records the weight space assumed in each reference; ``any'' means an arbitrary nonempty bounded closed convex set (\Cref{assu:margin}(1)), so that the probability simplex is admissible.

The upper part is based on \citet[Table 1]{sakaue2025online}: \citet{besbes2021online,besbes2025contextual} and \citet{gollapudi2021contextual} only claim that the total computational cost is $\mathrm{poly}(d, T)$, and the scrutiny of \citet{sakaue2025online} estimates the per-round cost of \citet{gollapudi2021contextual} to be at least $O(\tau_{\mathrm{solve}} + d^5 T^3)$. CoRectron \citep{sakaue2026simple} (marked ${}^{\ast}$ in the table) imposes no constraint on the iterate $\hat{\theta}^t$; it applies to $\Theta = \mathbb{R}^d$ rather than to a general $\Theta$, and is listed for reference. The values in the lower part are the costs at mistake rounds; rounds without a mistake need only the computation of the proposal ($O(\tau_{\mathrm{solve}})$). This is the origin of the dependence on $K$ in the total computational cost of \Cref{tab:comparison}.

\begin{table}[ht]
\centering
\caption{Comparison of the per-round computational cost (for the total cost see \Cref{tab:comparison}). The values in the upper part are the costs at every round, whereas those in the lower part are the costs at mistake rounds; a round without a mistake costs only $O(\tau_{\mathrm{solve}})$.}
\label{tab:complexity}
{\small
\setlength{\tabcolsep}{3pt}
\begin{tabular}{lll}
\toprule
 & $\Theta$ & Per-round computational cost \\
\midrule
\citet{Barmann-2018-online} & any & $O ( \tau_{\mathrm{solve}} + \tau_{\mathrm{E\text{-}proj}} + d )$ \\
\citet{besbes2021online,besbes2025contextual} & $\| \theta \| = 1$ & not claimed \\
\citet{gollapudi2021contextual} & $B^d$ & not claimed \\
ONS \citep{sakaue2025online} & any & $O ( \tau_{\mathrm{solve}} + d^2 + \tau_{\mathrm{G\text{-}proj}} )$ \\
MetaGrad \citep{sakaue2025online} & any & $O ( \tau_{\mathrm{solve}} + ( d^2 + \tau_{\mathrm{G\text{-}proj}} ) \log T )$ \\
CoRectron${}^{\ast}$ \citep{sakaue2026simple} & $\mathbb{R}^d$ & $O ( \tau_{\mathrm{solve}} + d^2 )$ \\
\midrule
SGS-OGD & any & $O ( \tau_{\mathrm{solve}} + d + \tau_{\mathrm{E\text{-}proj}} )$ \\
ONS & any & $O ( \tau_{\mathrm{solve}} + d^2 + \tau_{\mathrm{G\text{-}proj}} )$ \\
Growing-grid SGS-MetaGrad & any & $O ( \tau_{\mathrm{solve}} + ( d^2 + \tau_{\mathrm{G\text{-}proj}} ) \log K )$ \\
\bottomrule
\end{tabular}
}
\end{table}

\subsection{Comparison of the performance criteria with existing methods}
\label{app:metric_comparison}

The existing work on online inverse linear optimization all takes the cumulative decision regret $R^{\mathrm{est}}_T$ \cref{eq:decision_regret} as the main object of evaluation, calling it the ``regret'', and does not necessarily claim an upper bound on the cumulative suboptimality regret $R^{\mathrm{sub}}_T$ \cref{eq:cum_loss}.
While the two are bounded simultaneously by the sum $\widetilde{R}_T$ as in \cref{eq:max_leq_tilde}, an upper bound on one of the components does not imply an upper bound on the other.
We therefore organize in \Cref{tab:metric_comparison} which reference bounds which criterion.

The breakdown is as follows.
\citet{Barmann-2018-online} takes as its direct object the total error $\sum_t \langle \hat{\theta}^t - \theta^*, \hat{x}^t - x^t \rangle$, which amounts to the sum $\widetilde{R}_T$, makes explicit that it decomposes into the sum of the objective-function error $R^{\mathrm{sub}}_T$ and the solution error $R^{\mathrm{est}}_T$ (both components being nonnegative), and then shows $O(\sqrt{T})$.
\citet[Theorem 3.1]{sakaue2025online} (ONS) and \citet[Theorem 4.1]{sakaue2025online} (MetaGrad) likewise give upper bounds on $\widetilde{R}^{c^*}_T$ (the $\widetilde{R}_T$ of this paper), and both bound the two criteria simultaneously.
By contrast, the logarithmic regret of \citet{besbes2021online,besbes2025contextual} and \citet{gollapudi2021contextual}, and the $T$-independent regret of \citet{oki2026finite}, are claims about $R^{\mathrm{est}}_T$, and no upper bound on $R^{\mathrm{sub}}_T$ is claimed.
The three methods of this paper bound the sum $\widetilde{R}_T$ independently of $T$ under a uniform margin (\Cref{assu:margin}), and hence bound the two criteria simultaneously and independently of the total number of rounds.

We add a few words on how to read the table. The upper part lists existing methods that do not assume a uniform margin and the lower part the proposed methods of this paper (under \Cref{assu:margin}); ``not claimed'' indicates that the reference in question does not claim an upper bound on that criterion. The values marked with ${}^\ddagger$ come from an upper bound on the sum $\widetilde{R}_T$ and bound the two criteria simultaneously, and the mark ${}^\dagger$ indicates that the value holds only when the feasible set is M-convex. The column $\Theta$ records the weight space assumed in each reference; ``any'' means an arbitrary nonempty bounded closed convex set (\Cref{assu:margin}(1)), so that the probability simplex is admissible. Here $L, D$ are the constants in \cref{eq:def_DL}. The column ``Finite $K$'' is read as in \Cref{tab:comparison}; for the center-of-gravity method of \citet{oki2026finite}, the number of rounds with $\hat{x}^t \neq x^t$ is bounded by $O(d \log d)$ under M-convexity (marked ${}^\dagger$).

The entries come from the following sources. Those of \citet{Barmann-2018-online} and \citet{sakaue2025online} rewrite the upper bounds of the original papers in the notation of this paper (for the former, $\frac{3}{2} L D \sqrt{T}$; for the latter, the upper bound on the per-round linearized regret is evaluated as $LD$ and the upper bound on the norm of a subgradient as $L$). The entry of \citet{sakaue2025aistats} (marked ${}^{\S}$) is a result under the gap condition $\Delta > 0$ rather than a uniform margin (\citep[Theorem 5.2]{sakaue2025aistats}; for the definition of $\Delta$ and the derivation see Appendix~\ref{app:gap_entry}), and its constant depends on the regularizer of the FTRL and on the sizes of $\Theta$ and $X(s)$. The entry of CoRectron (marked ${}^{\ast}$) is listed for reference, since that method imposes no constraint on the iterate $\hat{\theta}^t$ and applies to $\Theta = \mathbb{R}^d$ rather than to a general $\Theta$; its value is \citet[Theorem 1]{sakaue2026simple} with its weight space specialized to $\mathbb{R}^d$, with the range of $\langle \theta^*, \cdot \rangle$ on each $X(s)$ bounded by $L \| \theta^* \|$ and its regularization parameter treated as a constant. The entries of \citet{besbes2021online,besbes2025contextual}, \citet{gollapudi2021contextual} and \citet{oki2026finite} are the values of the original papers, whose normalizations differ from one another. The two values in the entry of \citet{gollapudi2021contextual} correspond to two different algorithms of that paper: $O(d \log T)$ is \citet[Theorem 4.4]{gollapudi2021contextual} and $\exp(O(d \log d))$ is \citet[Theorem 4.2]{gollapudi2021contextual} (both through the reduction of \citet[Theorem 3.1]{gollapudi2021contextual}), and the latter does not depend on the total number of rounds $T$. The entries for the proposed methods restate \Cref{tab:main_results}.

\begin{table}[ht]
\centering
\caption{Comparison of the performance criteria under a general margin $\gamma$. The column ``Finite $K$'' indicates whether an upper bound on the number of mistakes $K$ that does not depend on $T$ is given ($\checkmark$) or not ($\times$), as in \Cref{tab:comparison}. ${}^\ddagger$: obtained from an upper bound on $\widetilde{R}_T$, so that the two criteria are bounded simultaneously; ${}^\dagger$: valid only when the feasible sets are M-convex; ${}^{\S}$: under the gap condition $\Delta > 0$ rather than a uniform margin; ${}^{\ast}$: no constraint is imposed on $\hat{\theta}^t$ ($\Theta = \mathbb{R}^d$).}
\label{tab:metric_comparison}
{\footnotesize
\renewcommand{\arraystretch}{1.4}
\setlength{\tabcolsep}{1.5pt}
\begin{tabular}{llllc}
\toprule
Method & $\Theta$ & $R^{\mathrm{sub}}_T$ & $R^{\mathrm{est}}_T$ & \begin{tabular}[c]{@{}c@{}} Finite \\ $K$ \end{tabular} \\
\midrule
\begin{tabular}[c]{@{}l@{}} OGD \\ \citep{Barmann-2018-online} \end{tabular} & any & $O ( L D \sqrt{T} )^{\ddagger}$ & $O ( L D \sqrt{T} )^{\ddagger}$ & $\times$ \\
\citet{besbes2021online,besbes2025contextual} & $\| \theta \| = 1$ & not claimed & $O ( d^4 \log T )$ & $\times$ \\
\citet{gollapudi2021contextual} & $B^d$ & not claimed & \begin{tabular}[c]{@{}l@{}} $O ( d \log T )$, \\ $\exp ( O ( d \log d ) )$ \end{tabular} & $\times$ \\
\citet{sakaue2025aistats}${}^{\S}$ & any & $O ( 1/\Delta^2 )^{\ddagger}$ & $O ( 1/\Delta^2 )^{\ddagger}$ & $\times$ \\
\begin{tabular}[c]{@{}l@{}} ONS, MetaGrad \\ \citep{sakaue2025online} \end{tabular} & any & $O ( L D\, d \log \tfrac{T}{d} )^{\ddagger}$ & $O ( L D\, d \log \tfrac{T}{d} )^{\ddagger}$ & $\times$ \\
\begin{tabular}[c]{@{}l@{}} CoRectron${}^{\ast}$ \\ \citep{sakaue2026simple} \end{tabular} & $\mathbb{R}^d$ & not claimed & $O ( L \| \theta^* \| d \log T )$ & $\times$ \\
\begin{tabular}[c]{@{}l@{}} Center-of-gravity \\ method${}^\dagger$ \\ \citep{oki2026finite} \end{tabular} & $\mathbb{R}^d$ & not claimed & $O ( d \log d )$ & $\checkmark$ \\
\midrule
\begin{tabular}[c]{@{}l@{}} SGS-OGD \\ (\Cref{theo:ogd}) \end{tabular} & any & $\frac{L^2 D^2}{2 \gamma}$ & $\frac{2 L^2 D^2}{\gamma}{}^{\ddagger}$ & $\checkmark$ \\
\begin{tabular}[c]{@{}l@{}} ONS \\ (\Cref{theo:main}) \end{tabular} & any & $L D ( 1 + d \log \max ( \tfrac{LD}{\gamma}, 1 ) )$ & $L D ( 1 + 2 d \log ( 2 + \tfrac{2LD}{\gamma} ) )^{\ddagger}$ & $\checkmark$ \\
\begin{tabular}[c]{@{}l@{}} Growing-grid \\ SGS-MetaGrad \\ (\Cref{theo:metagrad_anytime_main}) \end{tabular} & any & $O ( L D\, d \log \max ( \tfrac{LD}{\gamma}, 1 ) )$ & $O ( L D\, d \log \max ( \tfrac{LD}{\gamma}, 2 ) )^{\ddagger}$ & $\checkmark$ \\
\bottomrule
\end{tabular}
}
\end{table}

\section{Analysis of SGS-OGD (proof of \texorpdfstring{\Cref{theo:ogd}}{the theorem})}
\label{app:proof_sgs_ogd}

In the analyses from this appendix on, we reindex by mistake rounds.
We write the mistake rounds, in order of occurrence, as $t_1 < t_2 < \cdots < t_K$ (with the convention $t_{K+1} := T + 1$).
Since $\hat{\theta}^{t+1} = \hat{\theta}^t$ at rounds where no mistake occurs, the iteration can be described by the sequence $\hat{\theta}^{t_1}, \ldots, \hat{\theta}^{t_K}$ of mistake rounds alone, and at the $k$-th mistake round the subgradient is
$g^{t_k} = \hat{x}^{t_k} - x^{t_k}$ and the loss is $\ell^{t_k} := \ell_{\mathrm{sub}}(\hat{\theta}^{t_k}, s^{t_k})$.
We further use the quantity
\begin{equation}
    r_{t_k} := \langle \hat{\theta}^{t_k} - \bar{\theta},\, g^{t_k} \rangle
    \label{eq:def_r}
\end{equation}
relative to $\bar{\theta}$ (\Cref{assu:margin}(4)). The following lemma is the core of the analysis, common to the first-order and second-order methods, and states that the uniform margin \cref{eq:margin} makes $r_{t_k}$ exceed $\ell^{t_k}$ by at least $\gamma$ at every mistake round.
Moreover, in the analysis of the regret of the sum we use the quantity
\begin{equation}
    \widetilde{r}_{t_k} := \langle \hat{\theta}^{t_k} - \theta^*,\, g^{t_k} \rangle
    \label{eq:tilde_r_def}
\end{equation}
obtained by replacing $\bar{\theta}$ with $\theta^*$. Since at rounds without a mistake the summand is $0$ because $\hat{x}^t = x^t$, we have $\widetilde{R}_T = \sum_{k=1}^K \widetilde{r}_{t_k}$, and by the nonnegativity of both components of \cref{eq:regret_decomp} and the Cauchy--Schwarz inequality ($\hat{\theta}^{t_k}, \theta^* \in \Theta$, \cref{eq:def_DL}),
\begin{equation}
    0 \leq \widetilde{r}_{t_k} \leq \| \hat{\theta}^{t_k} - \theta^* \| \, \| g^{t_k} \| \leq L D
    \label{eq:tilde_r_bounds}
\end{equation}
holds.
This is the counterpart of $0 \leq r_{t_k} \leq LD$ in \Cref{lem:r_bounds} with $\bar{\theta}$ replaced by $\theta^*$, the only difference being that it does not have the lower bound $\gamma$ coming from the margin.

\begin{lemma}[Lower and upper bounds on $r_{t_k}$]
    \label{lem:r_bounds}
    Under \Cref{assu:margin}, for every mistake round $t_k$,
    \begin{equation}
        \ell^{t_k} + \gamma \leq r_{t_k} \leq L D .
        \label{eq:r_bounds}
    \end{equation}
    In particular $0 < \gamma \leq r_{t_k}$.
\end{lemma}

\begin{proof}[Proof of \Cref{lem:r_bounds}]
    Lower bound: from $\hat{x}^{t_k} \in \argmax_{x \in Y(s^{t_k})} \langle \hat{\theta}^{t_k}, x \rangle$ we have
    $\langle \hat{\theta}^{t_k}, g^{t_k} \rangle = \max_{x} \langle \hat{\theta}^{t_k}, x \rangle - \langle \hat{\theta}^{t_k}, x^{t_k} \rangle = \ell^{t_k}$,
    and at a mistake round we have $\hat{x}^{t_k} \neq x^{t_k}$ and $\hat{x}^{t_k} \in Y(s^{t_k})$, so \cref{eq:margin} gives
    $-\langle \bar{\theta}, g^{t_k} \rangle = \langle \bar{\theta}, x^{t_k} - \hat{x}^{t_k} \rangle \geq \gamma$.
    Adding the two, we obtain $r_{t_k} = \langle \hat{\theta}^{t_k}, g^{t_k} \rangle - \langle \bar{\theta}, g^{t_k} \rangle \geq \ell^{t_k} + \gamma$.
    Upper bound: by the Cauchy--Schwarz inequality together with $\| g^{t_k} \| \leq L$ and $\| \hat{\theta}^{t_k} - \bar{\theta} \| \leq D$ (both from \cref{eq:def_DL}).
\end{proof}

\begin{theorem}
    \label{theo:ogd}
    Under \Cref{assu:margin}, run SGS-OGD (\Cref{alg:ogd}) on an arbitrary sequence of states $\{ s^t \}_{t=1}^T$, and set $C_\alpha := D^2 / (2\alpha) + L^2 \alpha$. Then, for every $T$, the following hold.
    \begin{description}
        \item[(i)] ($K$)
        \begin{equation}
            K \leq \frac{C_\alpha^2}{\gamma^2} .
            \label{eq:ogd_mistake}
        \end{equation}
        \item[(ii)] ($R^{\mathrm{sub}}_T$)
        \begin{equation}
            R^{\mathrm{sub}}_T \leq \frac{C_\alpha^2}{4 \gamma} .
            \label{eq:ogd_regret}
        \end{equation}
        \item[(iii)] ($\widetilde{R}_T$)
        \begin{equation}
            \widetilde{R}_T \leq \frac{C_\alpha^2}{\gamma} .
            \label{eq:tilde_ogd}
        \end{equation}
    \end{description}
    In particular, choosing $\alpha = D / (L \sqrt{2})$ gives $C_\alpha^2 = 2 L^2 D^2$ and
    \[
        K \leq \frac{2 L^2 D^2}{\gamma^2} ,
        \qquad
        R^{\mathrm{sub}}_T \leq \frac{L^2 D^2}{2 \gamma} ,
        \qquad
        \widetilde{R}_T \leq \frac{2 L^2 D^2}{\gamma} .
    \]
    By \cref{eq:max_leq_tilde}, the cumulative decision regret $R^{\mathrm{est}}_T$ also has the same upper bound as \cref{eq:tilde_ogd}.
\end{theorem}

\begin{proof}[Proof of \Cref{theo:ogd}]
    \textbf{Step 1.} (Reindexing by SGS)
    At rounds where no mistake occurs we have $\hat{x}^t = x^t$, so $\ell_{\mathrm{sub}}(\hat{\theta}^t, s^t) = \langle \hat{\theta}^t, \hat{x}^t - x^t \rangle = 0$, and the update rule gives $\hat{\theta}^{t+1} = \hat{\theta}^t$. Hence rounds without a mistake contribute neither to the cumulative suboptimality regret nor to the trajectory of $\hat{\theta}^t$, and
    \[
        R^{\mathrm{sub}}_T = \sum_{k=1}^K \ell^{t_k} ,
        \qquad
        \hat{\theta}^{t_{k+1}} = \Pi_\Theta ( \hat{\theta}^{t_k} - \eta_k g^{t_k} ) ,
        \quad \eta_k := \alpha k^{-1/2}
    \]
    (the reindexing of \Cref{sec:sgs}). Moreover, \cref{eq:def_DL} gives $\| g^{t_k} \| = \| \hat{x}^{t_k} - x^{t_k} \| \leq L$.

    \textbf{Step 2.} (Potential estimate for an arbitrary $u \in \Theta$: one step)
    Below let $u \in \Theta$ be an arbitrary point and set $a_k(u) := \| \hat{\theta}^{t_k} - u \|^2$.
    By the nonexpansiveness of the Euclidean projection ($u \in \Theta$) and the update formula of Step 1,
    \begin{align*}
        a_{k+1}(u)
        & \leq \| \hat{\theta}^{t_k} - \eta_k g^{t_k} - u \|^2
        = a_k(u) - 2 \eta_k \langle g^{t_k}, \hat{\theta}^{t_k} - u \rangle + \eta_k^2 \| g^{t_k} \|^2 \\
        & \leq a_k(u) - 2 \eta_k \langle g^{t_k}, \hat{\theta}^{t_k} - u \rangle + \eta_k^2 L^2 .
    \end{align*}
    Dividing both sides by $2 \eta_k$ and rearranging,
    \[
        \langle g^{t_k}, \hat{\theta}^{t_k} - u \rangle
        \leq \frac{1}{2 \eta_k} \left( a_k(u) - a_{k+1}(u) \right) + \frac{\eta_k L^2}{2} .
    \]

    \textbf{Step 3.} (Estimate by Abel summation)
    Summing over $k = 1, \ldots, K$ and setting $\zeta_k := 1/(2\eta_k)$,
    \[
        \sum_{k=1}^K \langle g^{t_k}, \hat{\theta}^{t_k} - u \rangle
        \leq \sum_{k=1}^K \zeta_k ( a_k(u) - a_{k+1}(u) ) + \frac{L^2}{2} \sum_{k=1}^K \eta_k .
    \]
    By Abel summation,
    \begin{align*}
        \sum_{k=1}^K \zeta_k ( a_k(u) - a_{k+1}(u) )
        & = \zeta_1 a_1(u) + \sum_{k=2}^K ( \zeta_k - \zeta_{k-1} ) a_k(u) - \zeta_K a_{K+1}(u) \\
        & \leq \zeta_1 a_1(u) + \sum_{k=2}^K ( \zeta_k - \zeta_{k-1} ) a_k(u) .
    \end{align*}
    Since $\eta_k = \alpha k^{-1/2}$ is monotonically decreasing in $k$, the sequence $\zeta_k = k^{1/2}/(2\alpha)$ is monotonically increasing and $\zeta_k - \zeta_{k-1} \geq 0$. Moreover, $\hat{\theta}^{t_k}, u \in \Theta$ and $D = \diam (\Theta)$ give $a_k(u) \leq D^2$, so
    \[
        \zeta_1 a_1(u) + \sum_{k=2}^K ( \zeta_k - \zeta_{k-1} ) a_k(u)
        \leq D^2 \left( \zeta_1 + \sum_{k=2}^K ( \zeta_k - \zeta_{k-1} ) \right)
        = D^2 \zeta_K = \frac{D^2 \sqrt{K}}{2 \alpha} .
    \]
    On the other hand, comparison with an integral gives $\sum_{k=1}^K k^{-1/2} \leq \int_0^K x^{-1/2}\, \mathrm{d}x = 2\sqrt{K}$, whence
    $\frac{L^2}{2} \sum_{k=1}^K \eta_k \leq L^2 \alpha \sqrt{K}$.
    Combining the above, we obtain, for every $u \in \Theta$,
    \begin{equation}
        \sum_{k=1}^K \langle g^{t_k}, \hat{\theta}^{t_k} - u \rangle
        \leq \left( \frac{D^2}{2\alpha} + L^2 \alpha \right) \sqrt{K}
        = C_\alpha \sqrt{K} .
        \label{eq:ogd_anchor}
    \end{equation}
    The only property of $u$ used here is that $u \in \Theta$ (through the nonexpansiveness of the projection and $a_k(u) \leq D^2$).

    \textbf{Step 4.} ((i) and (ii): $u = \bar{\theta}$)
    Since $\bar{\theta} \in \Theta$ by \Cref{assu:margin}(4), \cref{eq:ogd_anchor} can be used with $u = \bar{\theta}$, and each term on the left-hand side is $r_{t_k}$. Substituting the lower bound $r_{t_k} \geq \ell^{t_k} + \gamma$ of \Cref{lem:r_bounds}, we obtain
    \begin{equation}
        \sum_{k=1}^K \ell^{t_k} + \gamma K
        \leq C_\alpha \sqrt{K} .
        \label{eq:ogd_key}
    \end{equation}
    From $\ell^{t_k} \geq 0$ and \cref{eq:ogd_key} we get $\gamma K \leq C_\alpha \sqrt{K}$, that is, $K \leq C_\alpha^2/\gamma^2$, which is (i).
    Moreover, \cref{eq:ogd_key} gives $\sum_k \ell^{t_k} \leq C_\alpha\sqrt{K} - \gamma K \leq \max_{x \geq 0} ( C_\alpha\sqrt{x} - \gamma x ) = C_\alpha^2/(4\gamma)$, which is (ii) (the right-hand side is maximized at $x_{\max} = C_\alpha^2/(4\gamma^2)$).
    Moreover, since $\hat{\theta}^{t+1} = \hat{\theta}^t$ at rounds without a mistake, the distinct iterates are only $\hat{\theta}^1$ and the points immediately after each mistake round, so the total number of iterates satisfies $| \{ \hat{\theta}^t \mid t = 1, \ldots, T \} | \leq K + 1$.

    \textbf{Step 5.} ((iii): $u = \theta^*$)
    Since $\theta^* \in \Theta$ by \Cref{assu:margin}(3), \cref{eq:ogd_anchor} can be used with $u = \theta^*$ as well, and each term on the left-hand side is $\widetilde{r}_{t_k}$ (\cref{eq:tilde_r_def}), so
    $\widetilde{R}_T = \sum_{k=1}^K \widetilde{r}_{t_k} \leq C_\alpha \sqrt{K}$.
    Since $\theta^*$ need not satisfy the margin \cref{eq:margin}, the lower bound on $r_{t_k}$ cannot be substituted, unlike in Step 4. Instead, using the monotonicity of $\sqrt{\cdot}$ and the bound $K \leq C_\alpha^2/\gamma^2$ of (i), we obtain \cref{eq:tilde_ogd}.
    In particular, when $\alpha = D/(L\sqrt{2})$ we have $C_\alpha = D^2 \cdot \frac{L\sqrt{2}}{2D} + L^2 \cdot \frac{D}{L\sqrt{2}} = \frac{LD}{\sqrt{2}} + \frac{LD}{\sqrt{2}} = \sqrt{2} L D$, so $C_\alpha^2 = 2 L^2 D^2$.
\end{proof}

\section{Analysis of ONS (proof of \texorpdfstring{\Cref{theo:main}}{the theorem})}
\label{app:proof_ons}

\begin{lemma}[{\citealp[cf.][Proposition 2.11]{orabona2019modern}}]
    \label{lem:projection}
    Let $\Sigma \succ 0$, let $\Theta$ be a nonempty closed convex set and let $u \in \Theta$. Then, for every $y \in \mathbb{R}^d$,
    $\| \Pi^\Sigma_\Theta(y) - u \|_\Sigma \leq \| y - u \|_\Sigma$.
\end{lemma}

\begin{lemma}[{\citealp[cf.][\S 4]{hazan2022introduction}}]
    \label{lem:logdet}
    Let $\Sigma_0 := D^{-2} \Id_d$, let $\nabla_{t_k} \in \mathbb{R}^d$ with $\| \nabla_{t_k} \| \leq 1/D$ ($k = 1, \ldots, K$), and let $\Sigma_k := \Sigma_{k-1} + \nabla_{t_k} \nabla_{t_k}^\top$. Then
    \begin{equation}
        \sum_{k=1}^K \nabla_{t_k}^\top \Sigma_k^{-1} \nabla_{t_k}
        \leq \log \frac{\det \Sigma_K}{\det \Sigma_0}
        \leq d \log \left( 1 + \frac{K}{d} \right) .
        \label{eq:logdet}
    \end{equation}
\end{lemma}

\begin{proof}
    The first inequality: for each $k$ we have $\Sigma_k \succeq \Sigma_0 + \nabla_{t_k}\nabla_{t_k}^\top = D^{-2}\Id_d + \nabla_{t_k} \nabla_{t_k}^\top$, so by the order reversal of the inverse and the Sherman--Morrison formula,
    \[
        \nabla_{t_k}^\top \Sigma_k^{-1} \nabla_{t_k}
        \leq \nabla_{t_k}^\top \left( D^{-2} \Id_d + \nabla_{t_k} \nabla_{t_k}^\top \right)^{-1} \nabla_{t_k}
        = \frac{\| \nabla_{t_k} \|^2}{D^{-2} + \| \nabla_{t_k} \|^2}
        < 1 .
    \]
    By the matrix determinant lemma,
    \[
        \det \Sigma_{k-1} = \det ( \Sigma_k - \nabla_{t_k} \nabla_{t_k}^\top ) = \det \Sigma_k \left( 1 - \nabla_{t_k}^\top \Sigma_k^{-1} \nabla_{t_k} \right) ,
    \]
    and applying $u \leq - \log(1 - u)$ (for $u < 1$) with $u = \nabla_{t_k}^\top \Sigma_k^{-1}\nabla_{t_k}$ gives
    $\nabla_{t_k}^\top \Sigma_k^{-1} \nabla_{t_k} \leq \log \frac{\det \Sigma_k}{\det \Sigma_{k-1}}$.
    Summing over $k = 1, \ldots, K$, the terms $\log \det \Sigma_k$ of adjacent summands cancel on the right-hand side, so the sum equals $\log \det \Sigma_K - \log \det \Sigma_0$, and we obtain the first inequality.

    The second inequality: writing the eigenvalues of $\Sigma_K$ as $\lambda_1, \ldots, \lambda_d > 0$, the arithmetic--geometric mean inequality gives
    $\det \Sigma_K = \prod_i \lambda_i \leq ( \tr \Sigma_K / d )^d$,
    and since $\tr \Sigma_K = d D^{-2} + \sum_k \| \nabla_{t_k} \|^2 \leq d D^{-2} + K D^{-2}$,
    \[
        \log \frac{\det \Sigma_K}{\det \Sigma_0}
        \leq d \log \frac{ (d + K) D^{-2} / d }{ D^{-2} }
        = d \log \left( 1 + \frac{K}{d} \right) .
    \]
\end{proof}

\begin{lemma}
    \label{lem:transcendental}
    Let $c_1, c_2 > 0$ and suppose that $y \geq 0$ satisfies $y \leq c_1 + c_2 \log (1 + y)$. Then
    \[
        y \leq 2 c_1 + 1 + 2 c_2 \log \max ( 2 c_2 ,\, 1 ) .
    \]
\end{lemma}

\begin{proof}
    Set $\mu := \max(2 c_2, 1) > 0$. By the concavity of $\log$, the tangent-line inequality $\log w \leq \log \mu + \frac{w}{\mu} - 1$ holds for every $w > 0$, so taking $w = 1 + y$,
    \[
        y \leq c_1 + c_2 \left( \log \mu + \frac{1 + y}{\mu} - 1 \right)
        \leq c_1 + c_2 \log \mu + \frac{1 + y}{2} - c_2 ,
    \]
    where we used $c_2 / \mu \leq 1/2$. Rearranging,
    $\frac{y}{2} \leq c_1 + \frac{1}{2} + c_2 \log \mu - c_2 \leq c_1 + \frac{1}{2} + c_2 \log \mu$,
    that is, $y \leq 2 c_1 + 1 + 2 c_2 \log \mu$.
\end{proof}

\begin{lemma}
    \label{lem:ons_potential}
    Under \Cref{assu:margin}, run \Cref{alg:ons} on an arbitrary sequence of states. For the quantities $\eta = 1/(LD)$, $\nabla_{t_k} = \eta\, g^{t_k}$ (the vector $\nabla^t$ of \Cref{alg:ons} at the $k$-th mistake round $t = t_k$, written with a subscript in this appendix) and $\Sigma_k$ (with $\Sigma_0 = D^{-2} \Id_d$) of \Cref{alg:ons} and an arbitrary $u \in \Theta$, setting $z_{t_k} := \langle \nabla_{t_k}, \hat{\theta}^{t_k} - u \rangle$, we have
    \begin{equation}
        \sum_{k=1}^K z_{t_k} - \frac{1}{2} \sum_{k=1}^K z_{t_k}^2
        \leq \frac{1}{2} \| \hat{\theta}^1 - u \|_{\Sigma_0}^2 + \frac{1}{2} \sum_{k=1}^K \nabla_{t_k}^\top \Sigma_k^{-1} \nabla_{t_k} .
        \label{eq:ons_potential}
    \end{equation}
\end{lemma}

\begin{proof}
    The update of \Cref{alg:ons} is
    $\hat{\theta}^{t_{k+1}} = \Pi^{\Sigma_k}_\Theta ( \hat{\theta}^{t_k} - \Sigma_k^{-1} \nabla_{t_k} )$,
    $\Sigma_k = \Sigma_{k-1} + \nabla_{t_k}\nabla_{t_k}^\top$,
    and $\| \nabla_{t_k} \| = \| g^{t_k} \| / (LD) \leq 1/D$ (\cref{eq:def_DL}).

    By \Cref{lem:projection} (with $\Sigma = \Sigma_k$ and $u \in \Theta$),
    \begin{align*}
        \| \hat{\theta}^{t_{k+1}} - u \|_{\Sigma_k}^2
        & \leq \| \hat{\theta}^{t_k} - \Sigma_k^{-1} \nabla_{t_k} - u \|_{\Sigma_k}^2 \\
        & = \| \hat{\theta}^{t_k} - u \|_{\Sigma_k}^2 - 2 \langle \nabla_{t_k}, \hat{\theta}^{t_k} - u \rangle + \nabla_{t_k}^\top \Sigma_k^{-1} \nabla_{t_k} ,
    \end{align*}
    where the cross term is $2 (\Sigma_k^{-1}\nabla_{t_k})^\top \Sigma_k (\hat{\theta}^{t_k} - u) = 2\langle \nabla_{t_k}, \hat{\theta}^{t_k} - u \rangle$ and the quadratic term is $(\Sigma_k^{-1}\nabla_{t_k})^\top \Sigma_k (\Sigma_k^{-1}\nabla_{t_k}) = \nabla_{t_k}^\top \Sigma_k^{-1}\nabla_{t_k}$. Rearranging,
    \begin{equation}
        z_{t_k}
        \leq \frac{1}{2} \left( \| \hat{\theta}^{t_k} - u \|_{\Sigma_k}^2 - \| \hat{\theta}^{t_{k+1}} - u \|_{\Sigma_k}^2 \right) + \frac{1}{2} \nabla_{t_k}^\top \Sigma_k^{-1} \nabla_{t_k} .
        \label{eq:ons_onestep}
    \end{equation}
    We sum over $k = 1, \ldots, K$. For the sum of the first term, using
    $\| \hat{\theta}^{t_k} - u \|_{\Sigma_k}^2 = \| \hat{\theta}^{t_k} - u \|_{\Sigma_{k-1}}^2 + z_{t_k}^2$,
    which follows from $\Sigma_k = \Sigma_{k-1} + \nabla_{t_k} \nabla_{t_k}^\top$, we obtain
    \begin{align*}
        \sum_{k=1}^K \left( \| \hat{\theta}^{t_k} - u \|_{\Sigma_k}^2 - \| \hat{\theta}^{t_{k+1}} - u \|_{\Sigma_k}^2 \right)
        & = \sum_{k=1}^K \left( \| \hat{\theta}^{t_k} - u \|_{\Sigma_{k-1}}^2 + z_{t_k}^2 - \| \hat{\theta}^{t_{k+1}} - u \|_{\Sigma_k}^2 \right) \\
        & = \| \hat{\theta}^1 - u \|_{\Sigma_0}^2 - \| \hat{\theta}^{t_{K+1}} - u \|_{\Sigma_K}^2 + \sum_{k=1}^K z_{t_k}^2 \\
        & \leq \| \hat{\theta}^1 - u \|_{\Sigma_0}^2 + \sum_{k=1}^K z_{t_k}^2 ,
    \end{align*}
    where the second equality holds because the term $- \| \hat{\theta}^{t_{k+1}} - u \|_{\Sigma_k}^2$ of the $k$-th summand and the term $\| \hat{\theta}^{t_{k+1}} - u \|_{\Sigma_k}^2$ of the $(k+1)$-st summand cancel, leaving only $\| \hat{\theta}^{1} - u \|_{\Sigma_0}^2$ from $k = 1$ and $- \| \hat{\theta}^{t_{K+1}} - u \|_{\Sigma_K}^2$ from $k = K$.
    Substituting this into the sum of \cref{eq:ons_onestep} gives \cref{eq:ons_potential}.
\end{proof}

\begin{proposition}[Logarithmic upper bound on $\sum_k r_{t_k}$]
    \label{prop:ons_main}
    Under \Cref{assu:margin}, run \Cref{alg:ons} on an arbitrary sequence of states.
    If $u \in \Theta$ satisfies $\langle g^{t_k}, \hat{\theta}^{t_k} - u \rangle \geq 0$ at every mistake round $t_k$, then
    \begin{equation}
        \sum_{k=1}^K \langle g^{t_k}, \hat{\theta}^{t_k} - u \rangle
        \leq L D \left( 1 + d \log \left( 1 + \frac{K}{d} \right) \right) .
        \label{eq:cumulative_r}
    \end{equation}
    Both $u = \bar{\theta}$ and $u = \theta^*$ satisfy this condition, and then the left-hand side of \cref{eq:cumulative_r} is $\sum_{k=1}^K r_{t_k}$ and $\widetilde{R}_T$ respectively.
\end{proposition}

\begin{proof}[Proof of \Cref{prop:ons_main}]
    We use \Cref{lem:ons_potential} with this $u$.
    We have $z_{t_k} = \eta \langle g^{t_k}, \hat{\theta}^{t_k} - u \rangle$, and from the assumption and the Cauchy--Schwarz inequality ($\| g^{t_k} \| \leq L$, and $\| \hat{\theta}^{t_k} - u \| \leq D$ since $\hat{\theta}^{t_k}, u \in \Theta$) we get
    $0 \leq \langle g^{t_k}, \hat{\theta}^{t_k} - u \rangle \leq LD$,
    so $z_{t_k} \in [0, 1]$ and hence $z_{t_k}^2 \leq z_{t_k}$.
    The left-hand side of \Cref{lem:ons_potential} is bounded from below by
    $\sum_k z_{t_k} - \frac{1}{2}\sum_k z_{t_k}^2 \geq \frac{1}{2} \sum_k z_{t_k}$,
    and the right-hand side is bounded, by $\| \hat{\theta}^1 - u \|_{\Sigma_0}^2 = D^{-2}\| \hat{\theta}^1 - u \|^2 \leq 1$ and \Cref{lem:logdet} (the condition $\|\nabla_{t_k}\| \leq 1/D$ having been checked), as
    $\frac{1}{2} + \frac{d}{2} \log ( 1 + K/d )$
    from above (neither of the two terms on the right-hand side depends on $u$). Hence
    $\sum_k z_{t_k} \leq 1 + d \log( 1 + K/d )$,
    and multiplying back by $\langle g^{t_k}, \hat{\theta}^{t_k} - u \rangle = LD\, z_{t_k}$ gives \cref{eq:cumulative_r}.

    Verification of the condition: we have $\bar{\theta} \in \Theta$ (\Cref{assu:margin}(4)), and \Cref{lem:r_bounds} gives $r_{t_k} = \langle g^{t_k}, \hat{\theta}^{t_k} - \bar{\theta} \rangle \geq \ell^{t_k} + \gamma > 0$.
    We have $\theta^* \in \Theta$ (\Cref{assu:margin}(3)), and \cref{eq:tilde_r_bounds} gives $\widetilde{r}_{t_k} = \langle g^{t_k}, \hat{\theta}^{t_k} - \theta^* \rangle \geq 0$.
    In the latter case, since the contribution of the rounds without a mistake is $0$, the left-hand side equals $\sum_{k=1}^K \widetilde{r}_{t_k} = \widetilde{R}_T$.
\end{proof}

\begin{theorem}
    \label{theo:main}
    Under \Cref{assu:margin}, run ONS (\Cref{alg:ons}) on an arbitrary sequence of states $\{ s^t \}_{t=1}^T$. Then, for every $T$, the following hold.
    \begin{description}
        \item[(i)] ($K$)
        \begin{equation}
            K
            \leq d + \frac{2LD}{\gamma} \left( 1 + d \log \max \left( \frac{2LD}{\gamma},\, 1 \right) \right) .
            \label{eq:mistake_bound}
        \end{equation}
        \item[(ii)] ($R^{\mathrm{sub}}_T$)
        \begin{equation}
            R^{\mathrm{sub}}_T
            \leq LD \left( 1 + d \log \max \left( \frac{LD}{\gamma},\, 1 \right) \right) .
            \label{eq:regret_bound}
        \end{equation}
        \item[(iii)] ($\widetilde{R}_T$)
        \begin{equation}
            \widetilde{R}_T
            \leq L D \left( 1 + 2 d \log \left( 2 + \frac{2 L D}{\gamma} \right) \right) .
            \label{eq:tilde_ons}
        \end{equation}
    \end{description}
    By \cref{eq:max_leq_tilde}, the cumulative decision regret $R^{\mathrm{est}}_T$ also has the same upper bound as \cref{eq:tilde_ons}.
\end{theorem}

\begin{proof}[Proof of \Cref{theo:main}]
    \textbf{(i)}
    Since \Cref{lem:r_bounds} gives $\sum_k r_{t_k} \geq \gamma K$, combining it with \Cref{prop:ons_main} (with $u = \bar{\theta}$) yields
    \begin{equation}
        \gamma K \leq LD \left( 1 + d \log \left( 1 + \frac{K}{d} \right) \right) .
        \label{eq:transcendental_K}
    \end{equation}
    Setting $y := K / d$, we have
    $y \leq \frac{LD}{\gamma d} + \frac{LD}{\gamma} \log ( 1 + y )$,
    and \Cref{lem:transcendental} (with $c_1 = \frac{LD}{\gamma d}$ and $c_2 = \frac{LD}{\gamma}$) gives
    \[
        y \leq \frac{2LD}{\gamma d} + 1 + \frac{2LD}{\gamma} \log \max \left( \frac{2LD}{\gamma}, 1 \right) ;
    \]
    multiplying both sides by $d$ gives \cref{eq:mistake_bound}.

    \textbf{(ii)}
    At rounds where no mistake occurs we have $\hat{x}^t = x^t \in \argmax_x \langle \hat{\theta}^t, x \rangle$, so $\ell_{\mathrm{sub}}(\hat{\theta}^t, s^t) = 0$, and hence
    $R^{\mathrm{sub}}_T = \sum_{k=1}^K \ell^{t_k}$.
    From $\ell^{t_k} \leq r_{t_k} - \gamma$ of \Cref{lem:r_bounds} and \Cref{prop:ons_main} (with $u = \bar{\theta}$) we obtain
    \[
        \sum_{k=1}^K \ell^{t_k}
        \leq \sum_{k=1}^K r_{t_k} - \gamma K
        \leq LD \left( 1 + d \log \left( 1 + \frac{K}{d} \right) \right) - \gamma K
        \leq \max_{x \geq 0} \phi(x) ,
    \]
    where we have set $\phi(x) := LD ( 1 + d \log ( 1 + \frac{x}{d} ) ) - \gamma x$.
    The function $\phi$ is differentiable with $\phi'(x) = \frac{LD\, d}{d + x} - \gamma$.
    \emph{Case 1 ($\gamma \geq LD$)}: for every $x \geq 0$ we have $\phi'(x) \leq LD - \gamma \leq 0$, so $\phi$ is nonincreasing and $\max_{x \geq 0} \phi = \phi(0) = LD$.
    \emph{Case 2 ($\gamma < LD$)}: the solution of $\phi' = 0$ is $x_{\max} = \frac{LD\, d}{\gamma} - d > 0$, and $\phi$ is increasing on $[0, x_{\max}]$ and decreasing on $[x_{\max}, \infty)$, so
    \[
        \max_{x \geq 0} \phi
        = \phi(x_{\max})
        = LD + LD\, d \log \frac{LD}{\gamma} - LD\, d + \gamma d
        \leq LD + LD\, d \log \frac{LD}{\gamma} ,
    \]
    where we used $\gamma d \leq LD\, d$. In either case $\max_{x\geq0}\phi \leq LD ( 1 + d \log \max ( LD/\gamma, 1 ) )$, and we obtain \cref{eq:regret_bound}.

    \textbf{(iii)}
    Using \Cref{prop:ons_main} with $u = \theta^*$ we obtain
    \begin{equation}
        \widetilde{R}_T \leq L D \left( 1 + d \log \left( 1 + \frac{K}{d} \right) \right) .
        \label{eq:tilde_ons_K}
    \end{equation}
    Setting $\beta := \frac{2 L D}{\gamma}$, item (i) gives $K \leq K_{\mathrm{ONS}} := d + \beta ( 1 + d \log \max ( \beta, 1 ) )$. Here
    \begin{align*}
        1 + \frac{K_{\mathrm{ONS}}}{d}
        &= 2 + \frac{\beta}{d} + \beta \log \max ( \beta, 1 )
        \leq 2 + \beta + \beta \log \max ( \beta, 1 )
        \\
        &\leq ( 2 + \beta ) \left( 1 + \log \max ( \beta, 1 ) \right)
    \end{align*}
    (the last inequality holds because expanding the right-hand side produces $2 \log \max(\beta,1) \geq 0$).
    Furthermore, since $\log \max ( \beta, 1 ) \leq \beta$ gives $1 + \log \max ( \beta, 1 ) \leq 2 + \beta$, we obtain
    $\log ( 1 + K_{\mathrm{ONS}} / d ) \leq 2 \log ( 2 + \beta )$.
    Since the right-hand side of \cref{eq:tilde_ons_K} is monotonically increasing in $K$, substituting $K \leq K_{\mathrm{ONS}}$ yields \cref{eq:tilde_ons}.
\end{proof}

\section{Analysis of MetaGrad (proof of \texorpdfstring{\Cref{theo:metagrad_main}}{the theorem})}
\label{app:proof_metagrad}

\subsection{The guarantee of the fixed-grid version}

\begin{theorem}
    \label{theo:metagrad_main}
    Under \Cref{assu:margin}, run MetaGrad (\Cref{alg:metagrad}) on an arbitrary sequence of states $\{ s^t \}_{t=1}^T$, and set
    $c_0 (\bar{K}) := 2 \log ( \frac{1}{2} \log_2 \bar{K} + 3 )$. Then the following hold.
    \begin{description}
        \item[(i)] ($K$)
        \begin{equation}
            K
            \leq d + \frac{152 LD}{3 \gamma} ( c_0 (\bar{K}) + d ) + \frac{152 LD\, d}{3 \gamma} \log \max \left( \frac{152 LD}{3 \gamma},\, 1 \right) .
            \label{eq:metagrad_mistake}
        \end{equation}
        \item[(ii)] ($R^{\mathrm{sub}}_T$)
        \begin{equation}
            R^{\mathrm{sub}}_T
            \leq \frac{76}{3} LD \left( c_0 (\bar{K}) + d + d \log \max \left( \frac{76 LD}{3 \gamma},\, 1 \right) \right) .
            \label{eq:metagrad_regret_bound}
        \end{equation}
        \item[(iii)] ($\widetilde{R}_T$) Writing $K_{\max}$ for the right-hand side of \cref{eq:metagrad_mistake},
        \begin{equation}
            \widetilde{R}_T
            \leq \frac{76}{3}\, L D \left( c_0 ( \bar{K} ) + d \left( \log \left( 1 + \frac{K_{\max}}{49\, d} \right) + 1 \right) \right) .
            \label{eq:tilde_metagrad_fixed}
        \end{equation}
    \end{description}
\end{theorem}

The proof is given in the next subsection.

\begin{remark}[Comparison with \Cref{theo:main}]
    \label{rem:metagrad_comparison}
    \Cref{theo:metagrad_main} has the same dependence on $\gamma$ and $d$ as \cref{eq:mistake_bound,eq:regret_bound}: the number of mistakes is $O ( \tfrac{d L D}{\gamma} \log \max ( \tfrac{2LD}{\gamma}, 2 ) )$, and the cumulative suboptimality regret is $O(L D\, d \log \max(LD/\gamma, 1))$.
    The price is (a) worse constants (the coefficient of the logarithmic term is $\frac{152}{3}$ against $2$, and there is an additive term $\frac{76}{3} LD (c_0 (\bar{K}) + d)$); (b) a doubly logarithmic dependence on the grid upper bound $\bar{K}$ (taking $\bar{K} = T$ breaks strict independence from $T$ to the extent of $c_0 (\bar{K}) = O(\log \log T)$; this dependence is removed by considering the growing-grid version of MetaGrad, \Cref{theo:metagrad_anytime_main}); and (c) a computational cost per mistake round multiplied by the grid size ($1 + \lceil \frac{1}{2} \log_2 \bar{K} \rceil$ experts each perform $O(d^2)$ plus a generalized projection; see the end of \S 4 of \citealp{sakaue2025online}).
    The parameters of the algorithm are $L, D, \bar{K}$; as with ONS, no knowledge of $\gamma$ is required.
\end{remark}

\subsection{Proof}

In the analysis we cite the following regret upper bound.

\begin{proposition}[Upper bound on the linearized regret of MetaGrad; cf.\ Proposition 2.6 of \citealp{sakaue2025online}]
    \label{prop:metagrad_regret}
    Let $n$ be a positive integer, let $\mathcal{W} \subset \mathbb{R}^n$ be a nonempty closed convex set whose $\ell_2$ diameter is at most $W > 0$, and take $G, H > 0$ and positive integers $m \leq \bar{m}$.
    Let $h_1, \ldots, h_m \colon \mathcal{W} \to \mathbb{R}$ be a sequence of convex loss functions and let $w_1, \ldots, w_m \in \mathcal{W}$ be the outputs of MetaGrad (\Cref{alg:metagrad_generic}) applied to $h_1, \ldots, h_m$.
    If, for each $j = 1, \ldots, m$, the subgradient $g_j \in \partial h_j (w_j)$ observed in \Cref{alg:metagrad_generic} satisfies $\| g_j \| \leq G$ and $\sup \{ \langle w' - w, g_j \rangle \mid w, w' \in \mathcal{W} \} \leq H$, then, for every $u \in \mathcal{W}$,
    \begin{equation}
        \sum_{j=1}^{m} \langle w_j - u, g_j \rangle
        \leq 3 \sqrt{ \Lambda_m V^u_m } + 10 H \Lambda_m ,
        \qquad
        V^u_m := \sum_{j=1}^m \langle w_j - u, g_j \rangle^2
        \label{eq:metagrad_regret}
    \end{equation}
    holds, where
    \begin{equation}
        \Lambda_m := 2 \log \left( \frac{1}{2} \log_2 \bar{m} + 3 \right) + n \left( \log \left( \frac{W^2 G^2 m}{49\, n H^2} + 1 \right) + 1 \right) .
        \label{eq:metagrad_lambda}
    \end{equation}
\end{proposition}

Below we write the index of the grid as $I := \lceil \frac{1}{2} \log_2 \bar{m} \rceil$ and $\mathcal{E} = \{ \eta_i = \frac{2^{-i}}{5H} \mid i = 0, 1, \ldots, I \}$ (\Cref{alg:metagrad_generic}).
Moreover, for the $n, W, G, H$ of \Cref{prop:metagrad_regret}, we define the function
\begin{equation}
    B_q := n \left( \log \left( \frac{W^2 G^2 q}{49\, n H^2} + 1 \right) + 1 \right)
    \label{eq:def_B}
\end{equation}
of a positive integer $q$. The map $q \mapsto B_q$ is monotonically nondecreasing, and $B_m$ equals the second term of \cref{eq:metagrad_lambda}.

The only external result cited in the proof is the following regret upper bound for a single $\eta$-expert.

\begin{proposition}[Regret upper bound for an $\eta$-expert; Appendix B.2 of \citealp{sakaue2025online}]
    \label{prop:eta_expert_regret}
    Let $n$ be a positive integer, let $\mathcal{W} \subset \mathbb{R}^n$ be a nonempty closed convex set whose $\ell_2$ diameter is at most $W > 0$, and take $G, H > 0$, $\eta \in ( 0, \frac{1}{5H} ]$ and a positive integer $q$.
    For a sequence of points $v_1, \ldots, v_{q} \in \mathcal{W}$ and a sequence of vectors $g_1, \ldots, g_{q} \in \mathbb{R}^n$ (with $\| g_j \| \leq G$ and $\sup \{ \langle w' - w, g_j \rangle \mid w, w' \in \mathcal{W} \} \leq H$), define the surrogate losses by
    $\ell^{\eta}_j (w) := - \eta \langle v_j - w, g_j \rangle + \eta^2 \langle v_j - w, g_j \rangle^2$,
    and suppose that running the $\eta$-expert (\Cref{defi:eta_expert}) on $\ell^{\eta}_1, \ldots, \ell^{\eta}_{q}$ yields $w^{\eta}_1, \ldots, w^{\eta}_{q} \in \mathcal{W}$.
    Then, for every $u \in \mathcal{W}$,
    \[
        \sum_{j=1}^{q} \left( \ell^{\eta}_j (w^{\eta}_j) - \ell^{\eta}_j (u) \right)
        \leq B_{q}
    \]
    holds (where $B_q$ is as in \cref{eq:def_B}).
\end{proposition}

\begin{lemma}[Monotonicity of the potential for a fixed grid]
    \label{lem:fixed_potential}
    In the setting of \Cref{prop:metagrad_regret}, set
    $\mathcal{L}^i_j := \sum_{j'=1}^{j} \ell^{\eta_i}_{j'} (w^{\eta_i}_{j'})$ (with $\mathcal{L}^i_0 := 0$) and
    $\Phi_j := \sum_{i=0}^{I} p^{\eta_i}_1 \exp ( - \mathcal{L}^i_j )$.
    Then $\Phi_m \leq \Phi_0 = 1$, and in particular, for every $i \in \{ 0, 1, \ldots, I \}$,
    $- \mathcal{L}^i_m \leq 2 \log (i+2)$.
\end{lemma}

\begin{proof}
    Setting $\tilde{p}^i_j := p^{\eta_i}_1 \exp ( - \mathcal{L}^i_{j-1} )$, the weight update of \Cref{alg:metagrad_generic} gives
    $p^{\eta_i}_j = \tilde{p}^i_j / \prod_{j' < j} Z_{j'}$,
    and since the point of the master is determined by the ratios of the weights alone,
    \begin{equation}
        w_j
        = \frac{\sum_{i=0}^{I} \eta_i\, p^{\eta_i}_j\, w^{\eta_i}_j}{\sum_{i=0}^{I} \eta_i\, p^{\eta_i}_j}
        = \frac{\sum_{i=0}^{I} \eta_i\, \tilde{p}^i_j\, w^{\eta_i}_j}{\sum_{i=0}^{I} \eta_i\, \tilde{p}^i_j}
        \label{eq:master_weighted}
    \end{equation}
    holds.
    We show $\Phi_j \leq \Phi_{j-1}$ for each $j$. Setting $x_i := \eta_i \langle w_j - w^{\eta_i}_j, g_j \rangle$, we have $\ell^{\eta_i}_j (w^{\eta_i}_j) = - x_i + x_i^2$ and $| x_i | \leq \eta_i H \leq \frac{1}{5}$.
    From the elementary inequality $e^{x - x^2} \leq 1 + x$, valid for $x \geq - \frac{1}{2}$ (because $f(x) := \log (1+x) - x + x^2$ has $f'(x) = \frac{x(2x+1)}{1+x}$ and hence attains its minimum value $0$ at $x = 0$), we obtain
    \[
        \Phi_j
        = \sum_{i=0}^{I} \tilde{p}^i_j\, e^{x_i - x_i^2}
        \leq \sum_{i=0}^{I} \tilde{p}^i_j ( 1 + x_i )
        = \Phi_{j-1} + \Big\langle w_j \sum_{i=0}^{I} \eta_i \tilde{p}^i_j - \sum_{i=0}^{I} \eta_i \tilde{p}^i_j w^{\eta_i}_j ,\, g_j \Big\rangle
        = \Phi_{j-1}
    \]
    (using $\sum_i \tilde{p}^i_j = \Phi_{j-1}$, the last equality being \cref{eq:master_weighted}).
    Since $\Phi_0 = \sum_{i=0}^{I} p^{\eta_i}_1 = 1$, we have $p^{\eta_i}_1 e^{- \mathcal{L}^i_m} \leq \Phi_m \leq 1$.
    Since $\sum_{i=0}^{I} \frac{1}{(i+1)(i+2)} = 1 - \frac{1}{I+2} \leq 1$ implies that the normalizing constant satisfies $C \geq 1$, we get
    $- \mathcal{L}^i_m \leq \log \frac{1}{p^{\eta_i}_1} = \log \frac{(i+1)(i+2)}{C} \leq \log ( (i+1)(i+2) ) \leq 2 \log (i+2)$.
\end{proof}

\begin{proof}[Proof of \Cref{prop:metagrad_regret}]
    Take $u \in \mathcal{W}$ and $i \in \{ 0, 1, \ldots, I \}$, and set $\eta := \eta_i$ and $a_j := \langle w_j - u, g_j \rangle$ (so that $| a_j | \leq H$ by assumption).
    From the definition of the surrogate loss we have $- \ell^{\eta}_j (u) = \eta a_j - \eta^2 a_j^2$, so summing over $j = 1, \ldots, m$ gives
    \begin{equation}
        \sum_{j=1}^m a_j
        = \frac{1}{\eta} \sum_{j=1}^m \left( - \ell^{\eta}_j (u) \right) + \eta V^u_m .
        \label{eq:fixed_decomp}
    \end{equation}
    Decomposing the sum in the first term on the right-hand side as
    \[
        \sum_{j=1}^m ( - \ell^{\eta}_j (u) )
        = ( - \mathcal{L}^i_m ) + \sum_{j=1}^m \left( \ell^{\eta}_j (w^{\eta}_j) - \ell^{\eta}_j (u) \right) ,
    \]
    the first term is at most $2 \log (i+2)$ by \Cref{lem:fixed_potential}, and the second is at most $B_m$ by \Cref{prop:eta_expert_regret} (with $q = m$), since the $\eta$-expert of \Cref{alg:metagrad_generic} is run on $\ell^{\eta}_1, \ldots, \ell^{\eta}_m$ with $v_j = w_j$. Substituting into \cref{eq:fixed_decomp}, we obtain, for every $i \in \{ 0, 1, \ldots, I \}$,
    \begin{equation}
        \sum_{j=1}^m a_j
        \leq \frac{2 \log (i+2) + B_m}{\eta_i} + \eta_i V^u_m .
        \label{eq:fixed_master}
    \end{equation}
    We distinguish cases according to $\eta^* := \sqrt{\Lambda_m / V^u_m}$ (with $\eta^* := + \infty$ when $V^u_m = 0$).
    Below we repeatedly use the fact that the first term of \cref{eq:metagrad_lambda} is at least $2 \log ( \frac{1}{2} \log_2 \bar{m} + 3 ) \geq 2 \log 3$.

    \emph{Case 1 ($\eta^* \geq \frac{1}{5H}$, that is, $V^u_m \leq 25 H^2 \Lambda_m$)}: we use \cref{eq:fixed_master} with $i = 0$ ($\eta_0 = \frac{1}{5H}$). From $2 \log 2 + B_m \leq \Lambda_m$ and $\frac{V^u_m}{5H} \leq 5 H \Lambda_m$,
    \[
        \sum_{j=1}^m a_j
        \leq 5H ( 2 \log 2 + B_m ) + \frac{V^u_m}{5H}
        \leq 10 H \Lambda_m .
    \]

    \emph{Case 2 ($\eta^* < \frac{1}{5H}$)}: we first show $\eta^* > \min \mathcal{E} = \eta_I$. From $| a_j | \leq H$ we have $V^u_m \leq H^2 m \leq H^2 \bar{m}$, and \cref{eq:metagrad_lambda} gives $\Lambda_m \geq 2 \log 3 + n \geq 1$, so
    \[
        \eta^* = \sqrt{\frac{\Lambda_m}{V^u_m}}
        \geq \frac{1}{H \sqrt{\bar{m}}}
        > \frac{1}{5 H \sqrt{\bar{m}}}
        = \frac{2^{- \frac{1}{2} \log_2 \bar{m}}}{5H}
        \geq \frac{2^{-I}}{5H}
        = \eta_I
    \]
    (the last inequality holding because $I \geq \frac{1}{2} \log_2 \bar{m}$).
    Letting $i^*$ be the largest $i$ with $\eta_i \geq \eta^*$, such an $i^*$ exists since $\eta_0 = \frac{1}{5H} > \eta^*$, and $i^* \leq I - 1$ since $\eta_I < \eta^*$.
    By the maximality of $i^*$ we have $\eta_{i^*+1} = \eta_{i^*} / 2 < \eta^*$, hence $\eta^* \leq \eta_{i^*} < 2 \eta^*$.
    Moreover, $I \leq \frac{1}{2} \log_2 \bar{m} + 1$ gives $i^* \leq \frac{1}{2} \log_2 \bar{m}$, so
    $2 \log (i^* + 2) + B_m \leq 2 \log ( \frac{1}{2} \log_2 \bar{m} + 3 ) + B_m \leq \Lambda_m$.
    Using \cref{eq:fixed_master} with $i = i^*$, we obtain
    \[
        \sum_{j=1}^m a_j
        \leq \frac{\Lambda_m}{\eta_{i^*}} + \eta_{i^*} V^u_m
        \leq \frac{\Lambda_m}{\eta^*} + 2 \eta^* V^u_m
        = 3 \sqrt{\Lambda_m V^u_m} .
    \]
    In either case the right-hand side is at most $3 \sqrt{\Lambda_m V^u_m} + 10 H \Lambda_m$, so \cref{eq:metagrad_regret} holds.
\end{proof}

\begin{remark}
    \Cref{prop:metagrad_regret} corresponds to Proposition 2.6 of \citet{sakaue2025online}, but there it is stated in $O(\cdot)$ notation for the case where the grid is constructed from the actual number of rounds ($\bar{m} = m$).
    The proof above makes the constants explicit and treats the case where the grid is constructed from an upper bound $\bar{m}$ on $m$: in our application $m = K$ (the realized number of mistakes) is unknown before execution, so the grid has to be fixed in advance by $\bar{K} \geq K$, which is why this generalization is needed.
    The quantity $\bar{m}$ enters only in the first term of \cref{eq:metagrad_lambda} and at the place in Case 2 where the lower end $\eta_I$ of the grid is estimated.
\end{remark}

\begin{proof}[Proof of \Cref{theo:metagrad_main}]
    If $K = 0$ everything is trivial, so assume $K \geq 1$.
    Since the internal state of \Cref{alg:metagrad} (the index $k$, the weights, the experts, and the prediction $\hat{\theta}^t$) does not change at rounds where no mistake occurs, \Cref{alg:metagrad} is nothing but MetaGrad of \Cref{alg:metagrad_generic} (with $n = d$, $\mathcal{W} = \Theta$, $W = D$, $H = LD$, $\bar{m} = \bar{K}$) applied to the sequence of convex losses $h_k := \ell_{\mathrm{sub}} (\cdot, s^{t_k})$ ($k = 1, \ldots, K$) of length $m = K$ (under the reindexing of \Cref{sec:sgs}, $w_k = \hat{\theta}^{t_k}$ and $g_k = g^{t_k}$; the surrogate loss of \Cref{alg:metagrad} coincides with that of \Cref{alg:metagrad_generic} since $w_k = \hat{\theta}^{t_k}$ at the $k$-th mistake round).
    To apply \Cref{prop:metagrad_regret} with $G = L$, we verify its assumptions. That $g^{t_k} \in \partial h_k (\hat{\theta}^{t_k})$ follows from
    \[
        h_k (\theta)
        \geq \langle \theta,\, \hat{x}^{t_k} - x^{t_k} \rangle
        = h_k (\hat{\theta}^{t_k}) + \langle \theta - \hat{\theta}^{t_k},\, g^{t_k} \rangle
    \]
    for every $\theta \in \Theta$ (the inequality by $\hat{x}^{t_k} \in Y(s^{t_k})$, the equality by $\langle \hat{\theta}^{t_k}, g^{t_k} \rangle = \ell^{t_k}$, from the proof of \Cref{lem:r_bounds}).
    Moreover $\| g^{t_k} \| \leq L = G$ (\cref{eq:def_DL}), the Cauchy--Schwarz inequality gives $\sup \{ \langle \theta' - \theta, g^{t_k} \rangle \mid \theta, \theta' \in \Theta \} \leq D L = H$, and furthermore $K \leq \bar{K}$.

    \textbf{Step 1.} (Estimate of $\sum_k r_{t_k}$)
    For $u = \bar{\theta}$ we have $\langle w_k - u, g_k \rangle = r_{t_k}$, so \cref{eq:metagrad_regret,eq:metagrad_lambda} (note that $W^2 G^2 / H^2 = D^2 L^2 / (LD)^2 = 1$) give
    \begin{equation}
        \begin{gathered}
            R := \sum_{k=1}^K r_{t_k}
            \leq 3 \sqrt{\Lambda V} + 10 LD\, \Lambda ,
            \\
            V := \sum_{k=1}^K r_{t_k}^2 ,
            \quad
            \Lambda := c_0 + d \left( \log \left( 1 + \frac{K}{49\, d} \right) + 1 \right) .
        \end{gathered}
        \label{eq:metagrad_R}
    \end{equation}
    Since \Cref{lem:r_bounds} gives $0 < r_{t_k} \leq LD$, we have $V \leq LD \cdot R$, and hence $R \leq 3 \sqrt{\Lambda\, LD\, R} + 10 LD\, \Lambda = \sqrt{a R} + \zeta$ (with $a := 9 LD\, \Lambda$ and $\zeta := 10 LD\, \Lambda$).
    Here, for all $a, \zeta, R \geq 0$,
    \begin{equation}
        R \leq \sqrt{a R} + \zeta
        \quad \Longrightarrow \quad
        R \leq \frac{4}{3} ( a + \zeta )
        \label{eq:self_bounding}
    \end{equation}
    holds. Indeed, assuming $R \leq \sqrt{a R} + \zeta$,
    \[
        R = \frac{4}{3} R - \frac{1}{3} R
        \leq \frac{4}{3} \left( \sqrt{a R} + \zeta \right) - \frac{1}{3} R
        = - \frac{1}{3} \left( \sqrt{R} - 2 \sqrt{a} \right)^2 + \frac{4}{3} ( a + \zeta )
        \leq \frac{4}{3} ( a + \zeta ) .
    \]
    Hence \cref{eq:self_bounding} yields
    \begin{equation}
        \sum_{k=1}^K r_{t_k} = R
        \leq \frac{4}{3} \cdot 19\, LD\, \Lambda
        = \frac{76}{3} LD \left( c_0 + d \left( \log \left( 1 + \frac{K}{49\, d} \right) + 1 \right) \right) .
        \label{eq:metagrad_r_bound}
    \end{equation}

    \textbf{(i)}
    Since \Cref{lem:r_bounds} gives $\gamma K \leq R$, from \cref{eq:metagrad_r_bound} and $\log ( 1 + \frac{K}{49 d} ) \leq \log ( 1 + \frac{K}{d} )$ we get
    \[
        \gamma K \leq \frac{76}{3} LD ( c_0 + d ) + \frac{76}{3} LD\, d \log \left( 1 + \frac{K}{d} \right) .
    \]
    Setting $y := K / d$, we have $y \leq c_1 + c_2 \log (1 + y)$ (with $c_1 := \frac{76 LD (c_0 + d)}{3 \gamma d}$ and $c_2 := \frac{76 LD}{3 \gamma}$), and \Cref{lem:transcendental} gives $y \leq 2 c_1 + 1 + 2 c_2 \log \max ( 2 c_2, 1 )$. Multiplying both sides by $d$ gives \cref{eq:metagrad_mistake}.

    \textbf{(ii)}
    As in the proof of \Cref{theo:main}(ii) we have $R^{\mathrm{sub}}_T = \sum_{k=1}^K \ell^{t_k} \leq R - \gamma K$, and setting $\rho := \frac{76}{3} LD$, \cref{eq:metagrad_r_bound} gives
    \[
        \sum_{k=1}^K \ell^{t_k}
        \leq \rho ( c_0 + d ) + \max_{x \geq 0} \psi(x) ,
        \qquad
        \psi(x) := \rho\, d \log \left( 1 + \frac{x}{d} \right) - \gamma x .
    \]
    Since $\psi'(x) = \frac{\rho d}{d + x} - \gamma$, if $\gamma \geq \rho$ then $\psi$ is nonincreasing and $\max_{x \geq 0} \psi = \psi(0) = 0$, whereas if $\gamma < \rho$ then at the stationary point $x_{\max} = \frac{\rho d}{\gamma} - d > 0$ we have
    $\psi(x_{\max}) = \rho d \log \frac{\rho}{\gamma} - \rho d + \gamma d \leq \rho d \log \frac{\rho}{\gamma}$.
    In either case $\max_{x \geq 0} \psi \leq \rho d \log \max ( \rho / \gamma, 1 )$, so we obtain \cref{eq:metagrad_regret_bound}.

    \textbf{(iii)}
    Since \Cref{prop:metagrad_regret} holds for every $u \in \mathcal{W}$, we apply it with $u = \theta^*$ (where $\theta^* \in \Theta = \mathcal{W}$ by \Cref{assu:margin}(3)).
    We have $\langle w_k - \theta^*, g_k \rangle = \widetilde{r}_{t_k}$, and the contribution of the rounds without a mistake is $0$, so \cref{eq:metagrad_regret,eq:metagrad_lambda} give the version of \cref{eq:metagrad_R} with $r_{t_k}$ replaced by $\widetilde{r}_{t_k}$,
    \[
        \widetilde{R}_T = \sum_{k=1}^K \widetilde{r}_{t_k}
        \leq 3 \sqrt{\Lambda\, \widetilde{V}} + 10 LD\, \Lambda ,
        \qquad
        \widetilde{V} := \sum_{k=1}^K \widetilde{r}_{t_k}^2
    \]
    (where $\Lambda$ is that of \cref{eq:metagrad_R}).
    Since $0 \leq \widetilde{r}_{t_k} \leq LD$ from \cref{eq:tilde_r_bounds} gives $\widetilde{V} \leq LD\, \widetilde{R}_T$, we have
    $\widetilde{R}_T \leq \sqrt{a \widetilde{R}_T} + \zeta$ (with $a = 9 LD\, \Lambda$ and $\zeta = 10 LD\, \Lambda$),
    and \cref{eq:self_bounding} yields
    \[
        \widetilde{R}_T
        \leq \frac{4}{3} \cdot 19\, LD\, \Lambda
        = \frac{76}{3}\, L D \left( c_0 (\bar{K}) + d \left( \log \left( 1 + \frac{K}{49\, d} \right) + 1 \right) \right) .
    \]
    Since the right-hand side is monotonically increasing in $K$, substituting $K \leq K_{\max}$ from (i) gives \cref{eq:tilde_metagrad_fixed}.
\end{proof}

\section{Analysis of growing-grid MetaGrad (proof of \texorpdfstring{\Cref{theo:metagrad_anytime_main}}{the theorem})}
\label{app:proof_metagrad_anytime}

In this appendix we analyze growing-grid MetaGrad (\Cref{alg:metagrad_anytime}) and prove \Cref{theo:metagrad_anytime_main}. The key is the reduction that regards a not-yet-created $\eta_i$-expert as a ``virtual expert that outputs the point of the master''.

\begin{proposition}[Upper bound on the linearized regret of growing-grid MetaGrad]
    \label{prop:metagrad_regret_anytime}
    In the setting of \Cref{prop:metagrad_regret}, modify MetaGrad (\Cref{alg:metagrad_generic}) as follows: place the prior weights $p_i := \frac{1}{(i+1)(i+2)}$ on the countable grid $\eta_i := \frac{2^{-i}}{5H}$ ($i \in \mathbb{Z}_{\geq 0}$), create the $\eta_i$-expert at round $j_i := 4^{i-1} + 1$ (for $i \geq 1$; $j_0 := 1$) and run ONS from then on, take as the point of the master the sum in \cref{eq:master_point} restricted to the already created experts (those $i$ with $j \geq j_i$), and use the weights without normalizing them.
    Then, for every $u \in \mathcal{W}$,
    \begin{equation}
        \sum_{j=1}^{m} \langle w_j - u, g_j \rangle
        \leq 3 \sqrt{ \Lambda'_m V^u_m } + 10 H \Lambda'_m + \frac{V^u_m}{100\, H \Lambda'_m}
        \label{eq:metagrad_regret_anytime}
    \end{equation}
    holds deterministically. Here $V^u_m$ is the same as in \cref{eq:metagrad_regret}, and
    \begin{equation}
        \Lambda'_m := 2 \log \left( \frac{1}{2} \log_2 m + 3 \right) + n \left( \log \left( \frac{W^2 G^2 m}{49\, n H^2} + 1 \right) + 1 \right)
        \label{eq:metagrad_lambda_anytime}
    \end{equation}
    is the quantity obtained from \cref{eq:metagrad_lambda} by replacing $\bar{m}$ with the actual number of rounds $m$.
\end{proposition}

\begin{lemma}[Monotonicity of the potential including the not-yet-created experts]
    \label{lem:sleeping_potential}
    In the setting of \Cref{prop:metagrad_regret_anytime}, adopt the convention that a not-yet-created ($j < j_i$) $\eta_i$-expert outputs the point of the master ($w^{\eta_i}_j := w_j$), and set
    $\mathcal{L}^i_j := \sum_{j' = 1}^{j} \ell^{\eta_i}_{j'} (w^{\eta_i}_{j'})$ (with $\mathcal{L}^i_0 := 0$) and
    $\Phi_j := \sum_{i \geq 0} p_i \exp ( - \mathcal{L}^i_j )$.
    Then $\Phi_m \leq \Phi_0 = 1$, and in particular, for every $i \in \mathbb{Z}_{\geq 0}$,
    $- \mathcal{L}^i_m \leq \log ( (i+1)(i+2) )$.
\end{lemma}

\begin{proof}
    Set $\tilde{p}^i_j := p_i \exp ( - \mathcal{L}^i_{j-1} )$. The point of the master is defined as
    $w_j = \big( \sum_{i :\, j \geq j_i} \eta_i \tilde{p}^i_j w^{\eta_i}_j \big) \big/ \big( \sum_{i :\, j \geq j_i} \eta_i \tilde{p}^i_j \big)$,
    a sum ranging over the already created $\eta_i$-experts only; we first show that this equals the sum over the whole grid,
    \begin{equation}
        w_j
        = \frac{\sum_{i \geq 0} \eta_i\, \tilde{p}^i_j\, w^{\eta_i}_j}{\sum_{i \geq 0} \eta_i\, \tilde{p}^i_j} .
        \label{eq:master_weighted_anytime}
    \end{equation}
    For a not-yet-created $i$ (that is, $j < j_i$), from $w^{\eta_i}_{j'} = w_{j'}$ ($j' < j_i$) and the definition of the surrogate loss we have $\ell^{\eta_i}_{j'} (w^{\eta_i}_{j'}) = 0$, so $\tilde{p}^i_j = p_i$ and $w^{\eta_i}_j = w_j$. Hence
    \[
        \sum_{i \geq 0} \eta_i\, \tilde{p}^i_j\, w^{\eta_i}_j
        = \sum_{i :\, j \geq j_i} \eta_i\, \tilde{p}^i_j\, w^{\eta_i}_j + w_j \sum_{i :\, j < j_i} \eta_i\, \tilde{p}^i_j
        = w_j \sum_{i \geq 0} \eta_i\, \tilde{p}^i_j
    \]
    (the second equality by the definition of $w_j$), which gives \cref{eq:master_weighted_anytime}. Here each series converges absolutely by $\eta_i \leq \frac{1}{5H}$, $\tilde{p}^i_j \leq p_i$, $\sum_{i \geq 0} p_i = 1$ and the boundedness of $\mathcal{W}$.
    Next we show $\Phi_j \leq \Phi_{j-1}$ for each $j$. Setting $x_i := \eta_i \langle w_j - w^{\eta_i}_j, g_j \rangle$, we have $\ell^{\eta_i}_j (w^{\eta_i}_j) = - x_i + x_i^2$ and $| x_i | \leq \eta_i H \leq \frac{1}{5}$.
    From the elementary inequality $e^{x - x^2} \leq 1 + x$, valid for $x \geq - \frac{1}{2}$ (because $f(x) := \log (1+x) - x + x^2$ has $f'(x) = \frac{x (2x+1)}{1+x}$ and hence attains its minimum value $0$ at $x = 0$), we obtain
    \[
        \Phi_j
        = \sum_{i \geq 0} \tilde{p}^i_j\, e^{x_i - x_i^2}
        \leq \sum_{i \geq 0} \tilde{p}^i_j ( 1 + x_i )
        = \Phi_{j-1} + \Big\langle w_j \sum_{i \geq 0} \eta_i \tilde{p}^i_j - \sum_{i \geq 0} \eta_i \tilde{p}^i_j w^{\eta_i}_j ,\, g_j \Big\rangle
        = \Phi_{j-1} ,
    \]
    the last equality by \cref{eq:master_weighted_anytime}. The final claim follows from $\Phi_0 = \sum_{i \geq 0} p_i = \sum_{i \geq 0} ( \frac{1}{i+1} - \frac{1}{i+2} ) = 1$ and $p_i e^{- \mathcal{L}^i_m} \leq \Phi_m \leq 1$.
\end{proof}

\begin{proof}[Proof of \Cref{prop:metagrad_regret_anytime}]
    We check the consistency of the creation schedule: for integers $j, i$ we have $\lceil \frac{1}{2} \log_2 j \rceil \geq i \iff \frac{1}{2} \log_2 j > i - 1 \iff j > 4^{i-1} \iff j \geq 4^{i-1} + 1 = j_i$ (for $i = 0$ this is always true, corresponding to $j_0 = 1$).

    Take any $i \in \mathbb{Z}_{\geq 0}$ and $u \in \mathcal{W}$, and set $\eta := \eta_i$ and $a_j := \langle w_j - u, g_j \rangle$ (so $| a_j | \leq H$). Summing $- \ell^{\eta}_j (u) = \eta a_j - \eta^2 a_j^2$ over $j = 1, \ldots, m$,
    \begin{equation}
        \sum_{j=1}^m a_j
        = \frac{1}{\eta} \sum_{j=1}^m \left( - \ell^{\eta}_j (u) \right) + \eta V^u_m .
        \label{eq:anytime_decomp}
    \end{equation}
    We decompose the sum in the first term on the right-hand side under the convention of \Cref{lem:sleeping_potential}:
    \[
        \sum_{j=1}^m ( - \ell^{\eta}_j (u) )
        = ( - \mathcal{L}^i_m )
        + \sum_{j < j_i} \left( \ell^{\eta}_j (w_j) - \ell^{\eta}_j (u) \right)
        + \sum_{j = j_i}^{m} \left( \ell^{\eta}_j (w^{\eta}_j) - \ell^{\eta}_j (u) \right) .
    \]
    The first term is at most $\log ( (i+1)(i+2) ) \leq 2 \log (i+2)$ by \Cref{lem:sleeping_potential}.
    The second term is at most $\eta H ( j_i - 1 ) = \eta H\, 4^{i-1}$ (and $0$ for $i = 0$), by $\ell^{\eta}_j (w_j) = 0$ and $- \ell^{\eta}_j (u) \leq \eta a_j \leq \eta H$.
    The third term is the regret of the $\eta_i$-expert created at round $j_i$ when run on the surrogate losses $\ell^{\eta}_{j_i}, \ldots, \ell^{\eta}_{m}$.
    Since this is a run of length $q := m - j_i + 1$ with shifted indices, it is at most $B_{q}$ by \Cref{prop:eta_expert_regret}, and hence at most $B_m$ by $q \leq m$ and the monotonicity of $q \mapsto B_q$.
    Substituting into \cref{eq:anytime_decomp} and using the identity $H\, 4^{i-1} = \frac{1}{100 H \eta_i^2}$, which follows from $\eta_i = \frac{2^{-i}}{5H}$, we obtain
    \begin{equation}
        \sum_{j=1}^m a_j
        \leq \frac{2 \log (i+2) + B_m}{\eta_i} + \frac{1}{100 H \eta_i^2} + \eta_i V^u_m
        \label{eq:anytime_master}
    \end{equation}
    (for $i = 0$ the middle term may be replaced by $0$).
    We distinguish cases according to $\eta^* := \sqrt{\Lambda'_m / V^u_m}$.

    Case 1 ($\eta^* \geq \frac{1}{5H}$, that is, $V^u_m \leq 25 H^2 \Lambda'_m$): we use \cref{eq:anytime_master} with $i = 0$. From $2 \log 2 + B_m \leq \Lambda'_m$ (because the first term of \cref{eq:metagrad_lambda_anytime} is at least $2 \log 3 \geq 2 \log 2$),
    \[
        \sum_{j=1}^m a_j
        \leq 5 H \Lambda'_m + \frac{V^u_m}{5H}
        \leq 10 H \Lambda'_m .
    \]

    Case 2 ($\eta^* < \frac{1}{5H}$): taking the largest $i^*$ with $\eta_{i^*} \geq \eta^*$, we have $\eta^* \leq \eta_{i^*} < 2 \eta^*$ since the ratio of the grid is $\frac{1}{2}$.
    From $V^u_m \leq H^2 m$ and $\Lambda'_m \geq 1$ we get $\eta^* \geq \frac{1}{H \sqrt{m}}$, so $\frac{2^{-i^*}}{5H} \geq \frac{1}{H \sqrt{m}}$ gives $i^* \leq \frac{1}{2} \log_2 m$ and hence
    $2 \log (i^* + 2) + B_m \leq 2 \log ( \frac{1}{2} \log_2 m + 2 ) + B_m \leq \Lambda'_m$.
    Moreover $\frac{1}{100 H \eta_{i^*}^2} \leq \frac{1}{100 H (\eta^*)^2} = \frac{V^u_m}{100 H \Lambda'_m}$. Thus \cref{eq:anytime_master} gives
    \begin{align*}
        \sum_{j=1}^m a_j
        &\leq \frac{\Lambda'_m}{\eta_{i^*}} + \eta_{i^*} V^u_m + \frac{V^u_m}{100 H \Lambda'_m}
        \leq \frac{\Lambda'_m}{\eta^*} + 2 \eta^* V^u_m + \frac{V^u_m}{100 H \Lambda'_m}
        \\
        &= 3 \sqrt{\Lambda'_m V^u_m} + \frac{V^u_m}{100 H \Lambda'_m} .
    \end{align*}
    In either case \cref{eq:metagrad_regret_anytime} holds.
\end{proof}

\begin{theorem}
    \label{theo:metagrad_anytime_main}
    Under \Cref{assu:margin}, run growing-grid SGS-MetaGrad (\Cref{alg:metagrad_anytime}) on an arbitrary sequence of states $\{ s^t \}_{t=1}^T$. Then the following hold.
    \begin{description}
        \item[(i)] ($K$)
        \begin{equation}
            K
            \leq d + \frac{52 LD}{\gamma} \left( 2 \log ( \log d + 3 ) + d \right) + \frac{52 LD (d+2)}{\gamma} \log \max \left( \frac{52 LD (d+2)}{\gamma d},\, 1 \right) .
            \label{eq:anytime_mistake}
        \end{equation}
        \item[(ii)] ($R^{\mathrm{sub}}_T$)
        \begin{equation}
            R^{\mathrm{sub}}_T
            \leq 26\, LD \left( 2 \log ( \log d + 3 ) + d + (d+2) \log \max \left( \frac{26 LD (d+2)}{\gamma d},\, 1 \right) \right) .
            \label{eq:anytime_regret}
        \end{equation}
        \item[(iii)] ($\widetilde{R}_T$) With $c_0(K) := 2 \log ( \frac{1}{2} \log_2 K + 3 )$ and $K_{\max}$ the right-hand side of \cref{eq:anytime_mistake},
        \begin{equation}
            \begin{split}
                \widetilde{R}_T
                &\leq 26\, L D \left( c_0 ( K_{\max} ) + d \left( \log \left( 1 + \frac{K_{\max}}{49\, d} \right) + 1 \right) \right)
                \\
                &= O \left( L D\, d \log \max \left( \frac{LD}{\gamma},\, 2 \right) \right) .
            \end{split}
            \label{eq:tilde_metagrad}
        \end{equation}
    \end{description}
    In particular, none of the right-hand sides depends on the total number of rounds $T$ at all, and the doubly logarithmic factor involves only the dimension $d$, not $T$.
    Moreover, by \cref{eq:max_leq_tilde}, the cumulative decision regret $R^{\mathrm{est}}_T$ also has the same upper bound as \cref{eq:tilde_metagrad}.
\end{theorem}

\begin{proof}[Proof of \Cref{theo:metagrad_anytime_main}]
    If $K = 0$ everything is trivial, so assume $K \geq 1$.
    As in the proof of \Cref{theo:metagrad_main}, \Cref{alg:metagrad_anytime} is nothing but growing-grid MetaGrad (with $n = d$, $\mathcal{W} = \Theta$, $W = D$, $G = L$, $H = LD$) applied to the sequence of convex losses of the mistake rounds (of length $m = K$), and the verification of the subgradients and of $G, H$ is identical (the agreement of the creation schedule being the content of the beginning of the proof of \Cref{prop:metagrad_regret_anytime}). Noting that $W^2 G^2 / H^2 = 1$ and applying \Cref{prop:metagrad_regret_anytime} with $u = \bar{\theta}$, we obtain
    \[
        R := \sum_{k=1}^K r_{t_k}
        \leq 3 \sqrt{\Lambda_K V} + 10\, LD\, \Lambda_K + \frac{V}{100\, LD\, \Lambda_K} ,
        \qquad
        V := \sum_{k=1}^K r_{t_k}^2 ,
    \]
    where we have set
    $c_0(K) := 2 \log ( \frac{1}{2} \log_2 K + 3 )$ and
    $\Lambda_K := c_0(K) + d ( \log ( 1 + \frac{K}{49\, d} ) + 1 )$
    (the quantity obtained from \cref{eq:metagrad_lambda_anytime} with $n = d$, $W = D$, $G = L$, $H = LD$ and $m = K$).

    \textbf{Step 1.} (Estimate of $\sum_k r_{t_k}$)
    Since \Cref{lem:r_bounds} gives $0 < r_{t_k} \leq LD$, we have $V \leq LD \cdot R$, and from $\Lambda_K \geq 2 \log 3 \geq 1$ the third term is bounded by $\frac{V}{100 LD \Lambda_K} \leq \frac{R}{100}$. Rearranging,
    \[
        R \leq \frac{100}{99} \left( 3 \sqrt{\Lambda_K\, LD\, R} + 10\, LD\, \Lambda_K \right)
        = \sqrt{a R} + \zeta
    \]
    (with $a := ( \frac{100}{33} )^2 \Lambda_K\, LD$ and $\zeta := \frac{1000}{99}\, \Lambda_K\, LD$), and \cref{eq:self_bounding} gives
    \[
        \sum_{k=1}^K r_{t_k} = R \leq \frac{4}{3} ( a + \zeta )
        = \frac{4}{3} \left( \frac{10000}{1089} + \frac{1000}{99} \right) LD\, \Lambda_K
        \leq 26\, LD\, \Lambda_K .
    \]

    \textbf{(i)}
    We first estimate $c_0(K)$. Using $\frac{1}{2} \log_2 K \leq \log K$ (for $K \geq 1$, since $\frac{1}{2 \log 2} \leq 1$), $\log K \leq \log d + \log ( 1 + \frac{K}{d} )$, and $a' + \zeta' \leq a' ( 1 + \zeta' )$ (for $a' \geq 1$, $\zeta' \geq 0$) with $a' = \log d + 3$ and $\zeta' = \log ( 1 + \frac{K}{d} )$, we obtain
    \begin{align}
        c_0 (K)
        & \leq 2 \log \left( \log d + 3 + \log \left( 1 + \tfrac{K}{d} \right) \right)
        \leq 2 \log ( \log d + 3 ) + 2 \log \left( 1 + \log \left( 1 + \tfrac{K}{d} \right) \right) \notag \\
        & \leq 2 \log ( \log d + 3 ) + 2 \log \left( 1 + \tfrac{K}{d} \right)
        \label{eq:c0_bound}
    \end{align}
    (the last inequality by $\log (1 + x) \leq x$). From $\gamma K \leq R$ of \Cref{lem:r_bounds} and Step 1, together with $\log ( 1 + \frac{K}{49 d} ) \leq \log ( 1 + \frac{K}{d} )$ and \cref{eq:c0_bound}, the quantity $y := K / d$ satisfies
    \[
        y \leq c_1 + c_2 \log ( 1 + y ) ,
        \qquad
        c_1 := \frac{26 LD ( 2 \log ( \log d + 3 ) + d )}{\gamma d} ,
        \quad
        c_2 := \frac{26 LD ( d + 2 )}{\gamma d} .
    \]
    \Cref{lem:transcendental} gives $y \leq 2 c_1 + 1 + 2 c_2 \log \max ( 2 c_2, 1 )$, and multiplying both sides by $d$ gives \cref{eq:anytime_mistake}.

    \textbf{(ii)}
    As in the proof of \Cref{theo:main}(ii) we have $R^{\mathrm{sub}}_T = \sum_{k=1}^K \ell^{t_k} \leq R - \gamma K$, and by (i) and \cref{eq:c0_bound},
    \begin{align*}
        \sum_{k=1}^K \ell^{t_k}
        &\leq 26 LD ( 2 \log ( \log d + 3 ) + d ) + \max_{x \geq 0} \psi (x) ,
        \\
        \psi (x) &:= 26 LD ( d + 2 ) \log \left( 1 + \frac{x}{d} \right) - \gamma x .
    \end{align*}
    Since $\psi'(x) = \frac{26 LD (d+2)}{d + x} - \gamma$, if $\gamma d \geq 26 LD (d+2)$ then $\psi$ is nonincreasing and $\max_{x \geq 0} \psi = \psi(0) = 0$; otherwise, at the stationary point $x_{\max} = \frac{26 LD (d+2)}{\gamma} - d > 0$,
    \begin{align*}
        \psi (x_{\max})
        &= 26 LD (d+2) \log \frac{26 LD (d+2)}{\gamma d} - 26 LD (d+2) + \gamma d
        \\
        &\leq 26 LD (d+2) \log \frac{26 LD (d+2)}{\gamma d} .
    \end{align*}
    In either case $\max_{x \geq 0} \psi \leq 26 LD (d+2) \log \max ( \frac{26 LD (d+2)}{\gamma d}, 1 )$, so we obtain \cref{eq:anytime_regret}.

    \textbf{(iii)}
    Since \Cref{prop:metagrad_regret_anytime} holds for every $u \in \mathcal{W}$, we apply it with $u = \theta^*$ (where $\theta^* \in \Theta = \mathcal{W}$ by \Cref{assu:margin}(3)).
    We have $\langle w_k - \theta^*, g_k \rangle = \widetilde{r}_{t_k}$, and the contribution of the rounds without a mistake is $0$, so the inequality at the beginning of the proof with $r_{t_k}$ replaced by $\widetilde{r}_{t_k}$,
    \[
        \widetilde{R}_T = \sum_{k=1}^K \widetilde{r}_{t_k}
        \leq 3 \sqrt{\Lambda_K \widetilde{V}} + 10\, LD\, \Lambda_K + \frac{\widetilde{V}}{100\, LD\, \Lambda_K} ,
        \qquad
        \widetilde{V} := \sum_{k=1}^K \widetilde{r}_{t_k}^2 ,
    \]
    holds.
    Since $0 \leq \widetilde{r}_{t_k} \leq LD$ from \cref{eq:tilde_r_bounds} gives $\widetilde{V} \leq LD\, \widetilde{R}_T$, and $\Lambda_K \geq 1$ bounds the third term by $\frac{\widetilde{R}_T}{100}$, rearranging gives
    \[
        \widetilde{R}_T
        \leq \frac{100}{99} \left( 3 \sqrt{\Lambda_K\, LD\, \widetilde{R}_T} + 10\, LD\, \Lambda_K \right)
        = \sqrt{a \widetilde{R}_T} + \zeta
    \]
    (with $a = ( \frac{100}{33} )^2 \Lambda_K\, LD$ and $\zeta = \frac{1000}{99}\, \Lambda_K\, LD$). By \cref{eq:self_bounding} we get $\widetilde{R}_T \leq \frac{4}{3} ( a + \zeta ) \leq 26\, LD\, \Lambda_K$.
    Since $\Lambda_K$ is monotonically increasing in $K$, substituting $K \leq K_{\max}$ from (i) gives the explicit form of \cref{eq:tilde_metagrad}.
    The order expression follows from \cref{eq:c0_bound} and $K_{\max} = O ( \tfrac{d L D}{\gamma} \log \max ( \tfrac{2LD}{\gamma}, 2 ) )$.
\end{proof}

\section{Lower bounds on the margin by structure}
\label{app:margin_lower}

In this appendix we collect the statements of the lower bounds on $\gamma_\mathrm{sub}$ summarized in \Cref{sec:lower_bounds} (including the case where $\Theta$ is the unit ball).

We set the sets of points
\begin{equation}
    Z(s) := \{ x^*(\theta^*, s) - x \mid x \in X(s) \setminus \{ x^*(\theta^*, s) \} \} ,
    \qquad
    Z^* := \bigcup_{s \in \mathcal{S}} Z(s)
    \label{eq:Zstar}
\end{equation}
(the differences being taken over all of $X(s)$, not over $Y(s)$).
By the minimax theorem (\Cref{prop:Minimax_theorem_finite} in Appendix~\ref{app:proof_lower}),
\[
    \max_{\theta \in \Theta} \inf_{s \in \mathcal{S}} \min_{z \in Z(s)} \langle \theta, z \rangle
    = \min_{z \in \Conv \overline{Z^*}} \max_{\theta \in \Theta} \langle \theta, z \rangle
\]
holds. Since $Y(s) \subseteq X(s)$ by \Cref{assu:margin}(2), the left-hand side is at most the quantity $\gamma_\mathrm{sub}$ of \cref{eq:gamma_SL}, and hence a lower bound on the right-hand side gives a lower bound on $\gamma_\mathrm{sub}$ directly.
Moreover, under \Cref{assu:ILP} every $z \in Z^*$ has components satisfying $|z_i| \leq M_i$, so
\begin{equation}
    L_\mathrm{sub} \leq \| M \|_2
    \label{eq:L_leq_m}
\end{equation}
(used in the corollaries of \Cref{sec:explicit_upper}).
Below we quantify the separation of the polyhedron $\Conv Z^*$ from the origin.

\begin{theorem}[Explicit lower bound for a general ILP with the unit ball]
    \label{theo:gamma_sub_lowerbound_on_ball}
    Assume (1), (2) and (3) of \Cref{assu:margin} and \Cref{assu:ILP}, and let $\Theta = \{ \theta \in \mathbb{R}^d : \| \theta \| \leq 1 \}$ (the unit ball) and $d \geq 2$.
    Let $M$ be the vector in \cref{eq:def_M}.
    Then
    \[
        \gamma_\mathrm{sub}
        \geq \frac{1}{2^{d-1}\, \| M \|_2^{d-1}} .
    \]
\end{theorem}

\begin{proposition}[The low-dimensional case: when the affine hull does not contain the origin]
    \label{prop:gamma_sub_lowdim_ball}
    Assume (1), (2) and (3) of \Cref{assu:margin} and \Cref{assu:ILP}, let $\Theta$ be the unit ball and let $Z^* \neq \emptyset$. Set $k := \dim \aff Z^*$ and assume $0 \notin \aff Z^*$. Then
    $\gamma_\mathrm{sub} \geq ( 2 \| M \|_2 )^{-k}$.
\end{proposition}

\begin{theorem}[Explicit lower bound for a general ILP with the probability simplex]
    \label{theo:gamma_sub_lowerbound_on_simplex}
    Assume (1), (2) and (3) of \Cref{assu:margin} and \Cref{assu:ILP}, and let $\Theta = \Delta^{d-1} = \{ \theta \in \mathbb{R}_{\geq 0}^d : \sum_{i=1}^d \theta_i = 1 \}$ (the probability simplex).
    Let $M$ be the vector in \cref{eq:def_M}. Then
    \[
        \gamma_\mathrm{sub}
        \geq \frac{1}{2^{d-1}\, \max(d-1, \sqrt{2})\, \| M \|_2^{d-1}} .
    \]
\end{theorem}

See Appendix~\ref{app:proof_ball} and Appendix~\ref{app:proof_simplex} respectively for the proofs.

\subsection{General theory of lower bounds via test sets}
\label{sec:test_set}

The lower bounds for discrete convex structures are obtained uniformly through test sets (defined below).

\begin{definition}[Test set]
    \label{defi:test_set}
    For a bounded discrete set $X \subset \mathbb{Z}^d$ (which is finite by boundedness), a finite set $\mathcal{T} \subset \mathbb{Z}^d \setminus \{ 0 \}$ is a test set of $X$ if the following holds.
    For every $x^1 \in X$ and every $\theta \in \Theta$, if $\langle \theta, x^1 \rangle < \max_{x \in X} \langle \theta, x \rangle$, then there exists $g \in \mathcal{T}$ with
    $x^1 + g \in X$ and $\langle \theta, g \rangle > 0$.
\end{definition}

\begin{proposition}[{Decomposition towards an optimal point via a test set; \citealp{kitaoka2026explicit}}]
    \label{prop:discrete_gradient_descent}
    Take a bounded discrete set $X \subset \mathbb{Z}^d$ and a test set $\mathcal{T}$ of it.
    Then, for every $x^1 \in X$ and every $\theta \in \Theta$, there exist $x^* \in \argmax_{x \in X} \langle \theta, x \rangle$ and $g^1, \ldots, g^r \in \mathcal{T}$ ($r \in \mathbb{Z}_{\geq 0}$) such that
    \[
        x^* - x^1 = \sum_{i=1}^r g^i ,
        \qquad
        \langle \theta, g^i \rangle > 0 \quad (i = 1, \ldots, r) .
    \]
\end{proposition}

See Appendix~\ref{app:proof_lower} for the proof.

Below we assume (1), (2) and (3) of \Cref{assu:margin} and let $\theta^* \in \Theta$ be the true weight.
We define the set of weights whose signs are consistent with those of $\theta^*$ on the test set $\mathcal{T}$ by
\begin{equation}
    \Theta_{\mathcal{T}}(\theta^*)
    := \{ \theta \in \Theta \mid \text{for every } g \in \mathcal{T}, \ \langle \theta^*, g \rangle > 0 \Rightarrow \langle \theta, g \rangle > 0 \} .
    \label{eq:sign_consistent_cone}
\end{equation}
By definition $\theta^* \in \Theta_{\mathcal{T}}(\theta^*)$.

\begin{proposition}[Lower bound on the margin via a test set]
    \label{prop:opt_gap_down_test_set_gap}
    Assume that the set $\mathcal{T}$ is a test set of the discrete set $X(s)$ for every $s \in \mathcal{S}$.
    Then
    \begin{equation}
        \gamma_{\mathrm{sub}}
        \geq \sup_{\theta \in \Theta_{\mathcal{T}}(\theta^*)} \min_{g \in \mathcal{T},\, \langle \theta^*, g \rangle > 0} \langle \theta, g \rangle .
        \label{eq:test_set_lower}
    \end{equation}
\end{proposition}

See Appendix~\ref{app:proof_lower} for the proof.

\subsection{Definitions from discrete convex analysis and Graver bases}
\label{app:definitions}

In this subsection we collect the definitions, taken from the cited references, that were used in \Cref{sec:lower_bounds}.
For $x \in \mathbb{Z}^d$ we set $\mathrm{supp}^+(x) := \{ i \mid x_i > 0 \}$ and $\mathrm{supp}^-(x) := \{ i \mid x_i < 0 \}$.

\begin{definition}[M-convex set \citep{murota2003discrete}]
    \label{defi:mconvex}
    A set $X \subseteq \mathbb{Z}^d$ is an \textbf{M-convex set} if, for every $x, y \in X$ and every $i \in \mathrm{supp}^+(x - y)$, there exists $j \in \mathrm{supp}^-(x - y)$ such that
    $x - e_i + e_j \in X$ and $y + e_i - e_j \in X$
    (the exchange axiom).
    All elements of an M-convex set have the same coordinate sum $\sum_i x_i$, and M-convex sets coincide with the sets of integer points of integral base polyhedra.
\end{definition}

\begin{definition}[{M${}^\natural$-convex set \citep[cf.][]{murota2003discrete}}]
    \label{defi:mnatconvex}
    We adopt the convention $e_0 := 0 \in \mathbb{Z}^d$. A set $X \subseteq \mathbb{Z}^d$ is an \textbf{M${}^\natural$-convex set} if, for every $x, y \in X$ and every $i \in \mathrm{supp}^+(x - y)$, there exists $j \in \mathrm{supp}^-(x - y) \cup \{ 0 \}$ such that
    $x - e_i + e_j \in X$ and $y + e_i - e_j \in X$.
    M${}^\natural$-convex sets are obtained as coordinate projections of M-convex sets, and coincide with the sets of integer points of generalized integral base polyhedra.
\end{definition}

\begin{definition}[{Graver basis \citep[cf.][]{onn2010nonlinear}}]
    \label{defi:graver}
    For a matrix $\widetilde{A} \in \mathbb{Z}^{N \times n}$ we set $\ker_{\mathbb{Z}}(\widetilde{A}) := \{ g \in \mathbb{Z}^n \mid \widetilde{A} g = 0 \}$. Two vectors $g, h \in \mathbb{Z}^n$ are \textbf{sign consistent} ($g \sqsubseteq h$) if $g_i h_i \geq 0$ and $|g_i| \leq |h_i|$ hold componentwise. The \textbf{Graver basis} $\mathcal{G}(\widetilde{A})$ is the set of all $\sqsubseteq$-minimal elements of $\ker_{\mathbb{Z}}(\widetilde{A}) \setminus \{0\}$ (a finite set).
\end{definition}

\subsection{Polynomial lower bounds for M-convex and M\texorpdfstring{${}^\natural$}{natural}-convex structures}
\label{sec:mconvex_lower}

\begin{assumption}[M-convex feasible set]
    \label{assu:M-convex-projection}
    For every $s \in \mathcal{S}$, the set $X(s) \subseteq \mathbb{Z}^d$ is an M-convex set (defined in Appendix~\ref{app:definitions}).
\end{assumption}

In this case $X(s)$ is a set of finitely many integer points, so \Cref{assu:ILP} is satisfied as well.

\begin{proposition}[{Test set of an M-convex set; \citealp{murota1996convexity,murota1998discrete,murota2003discrete}}]
    \label{prop:M-convex_test_set}
    A test set of an M-convex set can be taken to be the set of single exchange vectors $\mathcal{T} = \{ e_i - e_j \mid i \neq j \}$.
\end{proposition}

\begin{theorem}[Polynomial lower bound for M-convex sets]
    \label{theo:gamma_sub_M_convex_poly}
    Assume (1), (2) and (3) of \Cref{assu:margin} and \Cref{assu:M-convex-projection}.
    Then, in the case $\Theta = \{\theta \in \mathbb{R}^d : \|\theta\|_2 \leq 1\}$ with $d \geq 2$,
    \[
        \gamma_\mathrm{sub} \geq \frac{2\sqrt{3}}{\sqrt{d(d^2-1)}}
        = \Omega\!\left(\frac{1}{d^{3/2}}\right),
    \]
    and in the case $\Theta = \Delta^{d-1}$ with $d \geq 2$,
    \[
        \gamma_\mathrm{sub} \geq \frac{2}{d(d-1)}
        = \Omega\!\left(\frac{1}{d^2}\right).
    \]
\end{theorem}

\begin{assumption}[M${}^\natural$-convex feasible set]
    \label{assu:Mnatural-convex-projection}
    For every $s \in \mathcal{S}$, the set $X(s) \subseteq \mathbb{Z}^d$ is an M${}^\natural$-convex set (defined in Appendix~\ref{app:definitions}).
\end{assumption}

\begin{proposition}[{Test set of an M${}^\natural$-convex set; \citealp{murota1999mconvex,murota2003discrete}}]
    \label{prop:Mnatural-convex_test_set}
    A test set of an M${}^\natural$-convex set can be taken to be $\mathcal{T} = \{ e_i - e_j,\ \pm e_i \mid i \neq j \}$.
\end{proposition}

\begin{theorem}[Polynomial lower bound for M${}^\natural$-convex sets]
    \label{theo:gamma_sub_Mnatural_convex_poly}
    Assume (1), (2) and (3) of \Cref{assu:margin} and \Cref{assu:Mnatural-convex-projection}.
    Then, in the case $\Theta = \{\theta \in \mathbb{R}^d : \|\theta\|_2 \leq 1\}$ with $d \geq 2$,
    \[
        \gamma_\mathrm{sub} \geq \sqrt{\frac{6}{d(d+1)(2d+1)}}
        \geq \frac{1}{d^{3/2}}
        = \Omega\!\left(\frac{1}{d^{3/2}}\right),
    \]
    and in the case $\Theta = \Delta^{d-1}$ with $d \geq 2$,
    \[
        \gamma_\mathrm{sub} \geq \frac{2}{d(d+1)}
        = \Omega\!\left(\frac{1}{d^2}\right).
    \]
\end{theorem}

See Appendix~\ref{app:proof_mconvex} for both proofs.

\subsection{General lower bounds for linear inequality constraints: independence of \texorpdfstring{$\| M \|_2$}{||M||2}}
\label{sec:graver_lower}

In this subsection we show, for general linear inequality constraints with integer coefficients, a lower bound determined solely by the $\ell_\infty$ norm of the Graver basis (defined in Appendix~\ref{app:definitions}) of the coefficient matrix. In particular, the lower bound depends neither on the right-hand side $b(s)$ nor on the range $\| M \|_2$ of the features. This is an essential improvement over \Cref{theo:gamma_sub_lowerbound_on_ball,theo:gamma_sub_lowerbound_on_simplex}.

\begin{assumption}[Linear inequality constraints]
    \label{assu:linear-inequality}
    For a coefficient matrix $A \in \mathbb{Z}^{N \times d}$ and right-hand sides $b(s) \in \mathbb{Z}^N$, let
    $X(s) = \{ x \in \mathbb{Z}^d \mid A x \leq b(s) \}$
    (which is bounded by \Cref{assu:margin}(2)).
\end{assumption}

By introducing slack variables we convert this into the system of equalities $\widetilde{A} \widetilde{x} = b(s)$ (with $\widetilde{A} = [A \mid \Id_N] \in \mathbb{Z}^{N \times (d+N)}$, $\widetilde{x} = (x, y)$ and $y \geq 0$). We define the $\ell_\infty$ norm of the Graver basis $\mathcal{G}(\widetilde{A})$ by
$g_\infty(\widetilde{A}) := \max \{ \| g \|_\infty \mid g \in \mathcal{G}(\widetilde{A}) \} \in \mathbb{Z}_{\geq 1}$.
This is determined by $\widetilde{A}$ (and hence by $A$) alone, and depends neither on $b(s)$ nor on $\| M \|_2$.

\begin{proposition}[{Test set from a Graver basis; \citealp{kitaoka2026explicit}}]
    \label{prop:linear_inequality_test_set}
    Under \Cref{assu:linear-inequality}, setting
    $\mathcal{T}_x := \pi_x(\mathcal{G}(\widetilde{A}))$
    by means of the projection $\pi_x \colon \mathbb{Z}^{d+N} \to \mathbb{Z}^d$, the set $\mathcal{T}_x$ is a test set of $X(s)$ for every $s \in \mathcal{S}$ and satisfies $\| g \|_\infty \leq g_\infty(\widetilde{A})$ for all $g \in \mathcal{T}_x$.
\end{proposition}

\begin{theorem}[Lower bound independent of $\| M \|_2$ for linear inequalities with the unit ball]
    \label{theo:gamma_sub_linear_inequality_ball}
    Assume (1), (2) and (3) of \Cref{assu:margin} and \Cref{assu:linear-inequality}, let $\Theta$ be the unit ball and let $d \geq 2$. Set $C_g := g_\infty(\widetilde{A})$. Then
    \[
        \gamma_\mathrm{sub}
        \geq \frac{1}{(2C_g\sqrt{d})^{d-1}} .
    \]
    In particular, the lower bound depends neither on $b(s)$ nor on $\| M \|_2$. If $A$ is a totally unimodular matrix then $C_g = 1$.
\end{theorem}

\begin{theorem}[Lower bound independent of $\| M \|_2$ for linear inequalities with the probability simplex]
    \label{theo:gamma_sub_linear_inequality_simplex}
    Assume (1), (2) and (3) of \Cref{assu:margin} and \Cref{assu:linear-inequality}, and let $\Theta = \Delta^{d-1}$ and $d \geq 2$. Set $C_g := g_\infty(\widetilde{A})$. Then
    \[
        \gamma_\mathrm{sub}
        \geq \frac{1}{\max(d-1, \sqrt{2})\,(2C_g\sqrt{d})^{d-1}} .
    \]
    In particular, the lower bound depends neither on $b(s)$ nor on $\| M \|_2$. If $A$ is a totally unimodular matrix then $C_g = 1$.
\end{theorem}

See Appendix~\ref{app:proof_graver} for both proofs.
Whereas the existing ILP lower bounds (\Cref{theo:gamma_sub_lowerbound_on_ball,theo:gamma_sub_lowerbound_on_simplex}) depend strongly on the norm $\| M \|_2$ of the coordinatewise ranges, the lower bounds of this subsection depend only on $g_\infty(\widetilde{A})$, a quantity determined by the coefficient matrix alone.
This does not remove the exponential dependence on $d$ but expresses it through $C_g$: for a general $A$, $C_g \leq d\, \Delta(A)$ holds with the largest absolute value $\Delta(A)$ of the square subdeterminants of $A$, which can be large, and an explicit family of ILP instances with $\gamma_\mathrm{sub}$ exponentially small in $d$ shows that an exponential dependence on $d$ cannot be avoided for general ILPs \citep{kitaoka2026explicit}.
Comparing the denominators of the two lower bounds (in the case of the unit ball), the ratio is $(C_g\sqrt{d}/\| M \|_2)^{d-1}$, so for problems with large transportation amounts, demands or capacities but a simple structure of the coefficient matrix (a small $g_\infty(\widetilde{A})$), the improvement factor is roughly $(\| M \|_2 / (C_g\sqrt{d}))^{d-1}$.

\begin{remark}[The M-convex and M${}^\natural$-convex lower bounds are not corollaries of this subsection]
    \label{rem:mconvex_not_corollary}
    An M${}^\natural$-convex set is the set of integer points of a generalized integral base polyhedron, and that polyhedron is described by submodular and supermodular inequalities over a family of subsets.
    The corresponding coefficient matrix (the matrix whose rows are the indicator vectors of the subsets) is in general \emph{not} totally unimodular (for instance, the rows $(1,1,0)$, $(0,1,1)$, $(1,0,1)$ form a square submatrix of determinant $\pm 2$).
    Therefore \Cref{theo:gamma_sub_M_convex_poly,theo:gamma_sub_Mnatural_convex_poly} cannot in general be derived as corollaries of the results of this subsection with $C_g = 1$.
    They are obtained directly from the fact that the test sets can be taken explicitly as $\{ e_i - e_j \}$ and $\{ e_i - e_j, \pm e_i \}$, and as lower bounds they are, at $\Omega(d^{-3/2})$ and $\Omega(d^{-2})$, far better than the $(2C_g\sqrt{d})^{-(d-1)}$-type bounds of this subsection.
\end{remark}

\section{Proofs for \texorpdfstring{\Cref{sec:lower_bounds}}{the lower bounds on the margin}, I: general theory}
\label{app:proof_lower}

\begin{proof}[Proof of \Cref{prop:discrete_gradient_descent}]
    Since $X \subset \mathbb{Z}^d$ is bounded, it is a finite set. We construct a sequence of points $\{ x^{(k)} \}_{k \geq 0}$ inductively as follows: set $x^{(0)} := x^1$, and as long as $x^{(k)} \in X$ satisfies $\langle \theta, x^{(k)} \rangle < \max_{x \in X} \langle \theta, x \rangle$, take, by \Cref{defi:test_set}, some $g^{(k+1)} \in \mathcal{T}$ with
    $x^{(k+1)} := x^{(k)} + g^{(k+1)} \in X$ and $\langle \theta, g^{(k+1)} \rangle > 0$.
    Then $\langle \theta, x^{(k)} \rangle$ is strictly increasing, so $x^{(0)}, x^{(1)}, \ldots$ are pairwise distinct, and by the finiteness of $X$ the construction terminates after finitely many steps $r \geq 0$.
    The terminal point does not satisfy the continuation condition, that is, $\langle \theta, x^{(r)} \rangle = \max_{x \in X} \langle \theta, x \rangle$, so $x^{(r)} \in \argmax_{x \in X} \langle \theta, x \rangle$.
    Setting $x^* := x^{(r)}$ and $g^i := g^{(i)}$ ($i = 1, \ldots, r$), the construction gives $x^{(k+1)} - x^{(k)} = g^{(k+1)}$, so summing over $k = 0, \ldots, r-1$ yields $x^* - x^1 = \sum_{i=1}^r g^i$ (if $x^1$ is already a maximizer then $r = 0$ and the sum is empty).
\end{proof}

\begin{proof}[Proof of \Cref{lem:margin_from_gamma}]
    Consider the function
    \[
        \varphi(\theta) := \inf_{s \in \mathcal{S}} \min_{x \in Y(s) \setminus \{ x^*(\theta^*, s) \}} \langle \theta, x^*(\theta^*, s) - x \rangle .
    \]
    Since $Y(s)$ is a finite set for each $s$ by \Cref{assu:margin}(2), the inner $\min$ is the minimum of finitely many linear functions, and hence concave and $L_\mathrm{sub}$-Lipschitz.
    Here, for a compact set $X(s)$ the extreme points of $\Conv X(s)$ belong to $X(s)$, so $Y(s) \subseteq X(s)$, and the norms of the gradients are bounded by $\| x^*(\theta^*, s) - x \| \leq L_\mathrm{sub}$ (\cref{eq:Lipschitz_SL}).
    Therefore $\varphi$, being the infimum of those functions over $s$, is concave and $L_\mathrm{sub}$-Lipschitz (note that $\varphi(\theta) \geq - L_\mathrm{sub} \| \theta \| > - \infty$ at every point), and it attains its maximum value $\gamma_\mathrm{sub}$ on the bounded closed set $\Theta$. Let $\bar{\theta}$ be a maximizer.

    Assume $\gamma_\mathrm{sub} > 0$ and let us verify (4) and (5) of \Cref{assu:margin} ((1), (2) and (3) being assumptions).
    (4): the identity $\varphi(\bar{\theta}) = \gamma_\mathrm{sub}$ is exactly \cref{eq:margin}.
    (5): since the definitions \cref{eq:def_DL} and \cref{eq:Lipschitz_SL} are identical we have $L = L_\mathrm{sub}$, and by assumption $0 < L = L_\mathrm{sub} < \infty$.
\end{proof}

The following two propositions are used for the minimax expression of \cref{eq:gamma_SL}.

\begin{proposition}
    \label{prop:commute_cl_conv_on_bdd}
    Let $A \subset \mathbb{R}^n$ be a bounded set. Then
    $\overline{\Conv A} = \Conv \overline{A}$.
\end{proposition}

\begin{proof}
    ($\supset$) From $A \subset \overline{A}$ we get $\Conv A \subset \Conv \overline{A}$. Taking the closures of both sides gives $\overline{\Conv A} \subset \overline{\Conv \overline{A}}$. Since $A$ is bounded, $\overline{A}$ is compact, and in a finite-dimensional space the convex hull of a compact set is compact (by Carath\'{e}odory's theorem, $\Conv \overline{A}$ is the image of the compact set $\Delta^{n} \times \overline{A}^{n+1}$ under the continuous map $(\lambda, x_0, \ldots, x_n) \mapsto \sum_{i=0}^n \lambda_i x_i$, where $\Delta^{n} := \{ \lambda \in \mathbb{R}_{\geq 0}^{n+1} : \sum_{i=0}^{n} \lambda_i = 1 \}$). Hence $\Conv \overline{A}$ is closed, so $\overline{\Conv \overline{A}} = \Conv \overline{A}$, and therefore $\overline{\Conv A} \subset \Conv \overline{A}$.

    ($\subset$) The set $\overline{\Conv A}$ is a closed convex set containing $A$. Since $\overline{A}$ is the smallest closed set containing $A$, we have $\overline{A} \subset \overline{\Conv A}$, and since $\overline{\Conv A}$ is convex, $\Conv \overline{A} \subset \overline{\Conv A}$.
\end{proof}

\begin{proposition}
    \label{prop:Minimax_theorem_finite}
    Let $\Theta$ be a bounded closed convex set and let $A \subset \mathbb{R}^d$ be a bounded closed set. Then
    \[
        \max_{\theta \in \Theta} \min_{a \in A} \langle \theta, a \rangle
        =
        \min_{a \in \Conv A} \max_{\theta \in \Theta} \langle \theta, a \rangle .
    \]
\end{proposition}

\begin{proof}
    Since $A$ is compact, so is $\Conv A$ (see the proof of \Cref{prop:commute_cl_conv_on_bdd}).
    By the maximum principle (the minimum of a linear function on $\Conv A$ is attained at an extreme point, and the extreme points of $\Conv A$ are contained in $A$),
    $\max_{\theta \in \Theta} \min_{a \in A} \langle \theta, a \rangle = \max_{\theta \in \Theta} \min_{a \in \Conv A} \langle \theta, a \rangle$.
    Since $\Theta$ and $\Conv A$ are both compact convex sets and $(\theta, a) \mapsto \langle \theta, a \rangle$ is bilinear, the minimax theorem gives
    $\max_{\theta \in \Theta} \min_{a \in \Conv A} \langle \theta, a \rangle = \min_{a \in \Conv A} \max_{\theta \in \Theta} \langle \theta, a \rangle$.
\end{proof}

Using the set $Z^*$ of \cref{eq:Zstar},
\begin{equation}
    \gamma_\mathrm{sub}
    \geq \max_{\theta \in \Theta} \inf_{z \in Z^*} \langle \theta, z \rangle
    = \max_{\theta \in \Theta} \min_{z \in \overline{Z^*}} \langle \theta, z \rangle
    = \min_{z \in \Conv \overline{Z^*}} \max_{\theta \in \Theta} \langle \theta, z \rangle
    \label{eq:gamma_minimax}
\end{equation}
holds. The first inequality is due to the differences of $Z^*$ being taken over all of $X(s)$ rather than over $Y(s)$ (\cref{eq:Zstar}): since $Y(s) \subseteq X(s)$, for each $\theta$ the quantity $\inf_{z \in Z^*} \langle \theta, z \rangle$ is at most the quantity inside the $\max$ of \cref{eq:gamma_SL} (this direction suffices for the argument yielding a lower bound). For the first equality we used that $\overline{Z^*}$ is compact, since $Z^*$ is bounded (every $z \in Z^*$ satisfies $\|z\| \leq L_\mathrm{sub} < \infty$), and that the infimum of the continuous function $\langle \theta, z \rangle$ on $Z^*$ coincides with its minimum on $\overline{Z^*}$; for the second equality we used \Cref{prop:Minimax_theorem_finite} (with $A = \overline{Z^*}$).

\begin{proof}[Proof of \Cref{prop:opt_gap_down_test_set_gap}]
    Take an arbitrary $\theta \in \Theta_{\mathcal{T}}(\theta^*)$.
    For every $s \in \mathcal{S}$ and $y \in X(s) \setminus \{ x^*(\theta^*, s) \}$, apply \Cref{prop:discrete_gradient_descent} with $X = X(s)$, $\theta = \theta^*$ and $x^1 = y$. Since $x^*(\theta^*, s)$ is the unique maximizer of $\theta^*$ on $X(s)$ by \Cref{assu:margin}(3), we have $x^* = x^*(\theta^*, s)$, and there exist $g^1, \ldots, g^r \in \mathcal{T}$ ($r \geq 1$) with
    \[
        x^*(\theta^*, s) - y = \sum_{i=1}^r g^i ,
        \qquad
        \langle \theta^*, g^i \rangle > 0 \quad (i = 1, \ldots, r)
    \]
    (where $r \geq 1$ because $x^*(\theta^*, s) \neq y$). Since each $g^i$ satisfies $\langle \theta, g^i \rangle > 0$,
    \[
        \langle \theta, x^*(\theta^*, s) - y \rangle
        = \sum_{i=1}^r \langle \theta, g^i \rangle
        \geq \min_{g \in \mathcal{T},\, \langle \theta^*, g \rangle > 0} \langle \theta, g \rangle .
    \]
    Taking the infimum over $y$ and $s$,
    $\inf_{s} \min_{y} \langle \theta, x^*(\theta^*, s) - y \rangle \geq \min_{g \in \mathcal{T},\, \langle \theta^*, g \rangle > 0} \langle \theta, g \rangle$.
    Since $Y(s) \subseteq X(s)$ by \Cref{assu:margin}(2), the left-hand side is at most the quantity inside the $\max$ of \cref{eq:gamma_SL}, and hence at most $\gamma_\mathrm{sub}$ since $\theta \in \Theta_{\mathcal{T}}(\theta^*) \subset \Theta$. Taking the supremum over $\theta \in \Theta_{\mathcal{T}}(\theta^*)$ on the right-hand side gives \cref{eq:test_set_lower}.
\end{proof}

\section{Auxiliary results: a maximum on the probability simplex and the distance to a lattice polytope}
\label{app:proof_separation}

\begin{proposition}
    \label{prop:max_1-x_on_simplex}
    Let $\xi > 0$ and set
    $F_\xi(z) := \sum_{i=1}^d (1 - z_i)^{\xi}$ for $z \in \Delta^{d-1}$.
    Then:
    \begin{enumerate}
        \item if $\xi \geq 1$ then $F_\xi$ is convex on $\Delta^{d-1}$ and its maximum is attained at a vertex; in particular $F_\xi(z) \leq d-1$;
        \item if $0 < \xi < 1$ then $F_\xi$ is strictly concave on $\Delta^{d-1}$ and its maximum is attained at the barycenter $z = (1/d, \ldots, 1/d)$; in particular $F_\xi(z) \leq d (1 - 1/d)^{\xi}$.
    \end{enumerate}
\end{proposition}

\begin{proof}
    The second derivative of each term $g(t) = (1-t)^{\xi}$ is $g''(t) = \xi (\xi - 1) (1-t)^{\xi-2}$. For $t \in [0, 1)$:
    when $\xi \geq 1$, $g'' \geq 0$ so $g$ is convex. Hence $F_\xi$ is convex on $\Delta^{d-1}$ and its maximum is attained at a vertex $\mathbf{e}_i$. From $F_\xi(\mathbf{e}_i) = (1-1)^{\xi} + (d-1)(1-0)^{\xi} = d-1$ the claim follows.
    When $0 < \xi < 1$, $g'' < 0$ so $g$ is strictly concave. Hence $F_\xi$ is strictly concave on $\Delta^{d-1}$, and by Jensen's inequality the maximizer is the barycenter $z = (1/d, \ldots, 1/d)$, with $F_\xi(1/d, \ldots, 1/d) = d (1 - 1/d)^{\xi}$.
\end{proof}

\begin{lemma}[Distance from the origin to a lattice polytope]
    \label{lem:dist_lattice_polytope}
    Let $A \subset \mathbb{Z}^d$ be a nonempty finite set with $0 \notin \Conv A$, and let $R = (R_1, \ldots, R_d) \in \mathbb{R}_{\geq 0}^d$ satisfy $|a_i| \leq R_i$ for every $a \in A$ and every $i = 1, \ldots, d$. Then $\| R \|_2 \geq 1$ and
    \[
        \min_{q \in \Conv A} \| q \| \geq ( 2 \| R \|_2 )^{-(d-1)} .
    \]
\end{lemma}

\begin{proof}
    Since $A \neq \emptyset$ and $0 \notin A$, some $a \in A$ is a nonzero integer vector, so $R_i \geq |a_i| \geq 1$ for some $i$ and hence $\| R \|_2 \geq 1$.
    Since $A$ is finite, $P := \Conv A$ is a polytope, and the minimum of $\| \cdot \|$ on $P$ is attained at a unique point $p^*$ (by the strict convexity of $\| \cdot \|^2$), with $p^* \neq 0$ because $0 \notin P$.

    \textbf{Step 1 (the face containing the nearest point).}
    Let $F$ be the unique face of $P$ whose relative interior contains $p^*$. Since $p^*$ minimizes $\frac{1}{2} \| q \|^2$ over $P$, we have $\langle p^*, q - p^* \rangle \geq 0$ for every $q \in P$. For $q \in F$, the point $p^* - \varepsilon ( q - p^* )$ also belongs to $F$ for sufficiently small $\varepsilon > 0$ because $p^* \in \mathrm{relint}\, F$, and applying the inequality to this point gives $\langle p^*, q - p^* \rangle \leq 0$. Hence $\langle p^*, q - p^* \rangle = 0$ for every $q \in F$, and by linearity for every $q \in \aff F$. Thus $p^*$ is the orthogonal projection of $0$ onto $\aff F$, and
    \[
        \dist ( 0, \aff F ) = \| p^* \| > 0 ,
        \qquad \text{in particular} \quad 0 \notin \aff F .
    \]
    Moreover $k := \dim F \leq d-1$: if $k = d$, then $\aff F = \mathbb{R}^d$ and the orthogonal projection of $0$ onto it would be $p^* = 0$, a contradiction.

    \textbf{Step 2 (a Gram determinant estimate).}
    The vertices of $F$ are vertices of $P$, hence elements of $A$. If $k = 0$, then $F = \{ p^* \}$ with $p^* \in A$ a nonzero integer vector, so $\dist ( 0, \aff F ) = \| p^* \| \geq 1$. If $k \geq 1$, take affinely independent vertices $y^1, \ldots, y^{k+1} \in A$ of $F$ spanning $\aff F$ and set $v_i := y^{i+1} - y^1 \in \mathbb{Z}^d$ ($i = 1, \ldots, k$). Each component satisfies $|v_{i,j}| \leq 2 R_j$, so $\| v_i \| \leq 2 \| R \|_2$. Since $0 \notin \aff F = y^1 + \spn ( v_1, \ldots, v_k )$, the vectors $v_1, \ldots, v_k, y^1$ are linearly independent. Writing $G$ for Gram determinants and decomposing $- y^1 = w_U + w_\perp$ with $w_U \in U := \spn ( v_1, \ldots, v_k )$ and $w_\perp \perp U$, we have $\dist ( 0, \aff F ) = \| w_\perp \|$, and elementary column operations give $\det G ( v_1, \ldots, v_k, y^1 ) = \det G ( v_1, \ldots, v_k ) \, \| w_\perp \|^2$, that is,
    \[
        \dist ( 0, \aff F )^2 = \frac{\det G ( v_1, \ldots, v_k, y^1 )}{\det G ( v_1, \ldots, v_k )} .
    \]
    The numerator is the Gram determinant of linearly independent integer vectors, hence a positive integer, and the denominator is at most $\prod_{i=1}^k \| v_i \|^2 \leq ( 2 \| R \|_2 )^{2k}$ by Hadamard's inequality. In either case $\dist ( 0, \aff F ) \geq ( 2 \| R \|_2 )^{-k}$.

    Combining the two steps with $2 \| R \|_2 \geq 1$ and $k \leq d-1$, we obtain
    $\min_{q \in P} \| q \| = \| p^* \| = \dist ( 0, \aff F ) \geq ( 2 \| R \|_2 )^{-k} \geq ( 2 \| R \|_2 )^{-(d-1)}$.
\end{proof}

\section{Proofs of \texorpdfstring{\Cref{theo:gamma_sub_lowerbound_on_ball} and \Cref{prop:gamma_sub_lowdim_ball}}{the lower bounds for a general ILP with the unit ball}}
\label{app:proof_ball}

\begin{proof}[Proof of \Cref{theo:gamma_sub_lowerbound_on_ball}]
    \textbf{Step 1 (reduction to a distance problem).}
    By \Cref{assu:ILP} we have $X(s) \subset \mathbb{Z}^d$ and $x^*(\theta^*, s) \in \mathbb{Z}^d$ for each $s$, so $Z(s) \subset \mathbb{Z}^d$ and hence $Z^* \subset \mathbb{Z}^d$. Furthermore, for every $z = x^*(\theta^*, s) - x \in Z(s)$, since $x^*(\theta^*, s), x \in X(s)$, each component satisfies
    \[
        |z_i| \leq \max_{x \in X(s)} x_i - \min_{x \in X(s)} x_i \leq M_i .
    \]
    Hence
    $Z^* \subset \Lambda := \prod_{i=1}^d \{-M_i, -M_i + 1, \ldots, M_i\}$,
    so $Z^*$ is a finite set. Therefore $\overline{Z^*} = Z^*$, and $\Conv \overline{Z^*} = \Conv Z^*$ is a bounded closed convex polytope.
    From \cref{eq:gamma_minimax} and the fact that $\max_{\|\theta\| \leq 1} \langle \theta, z \rangle = \|z\|$ when $\Theta$ is the unit ball,
    \begin{equation}
        \gamma_\mathrm{sub} \geq \min_{z \in \Conv Z^*} \|z\| .
        \label{eq:gamma_sub_as_dist_ball}
    \end{equation}

    \textbf{Step 2 ($0 \notin \Conv Z^*$).}
    By \Cref{assu:margin}(3), for every $s \in \mathcal{S}$ and every $x \in X(s) \setminus \{x^*(\theta^*, s)\}$ we have $\langle \theta^*, x^*(\theta^*, s) - x \rangle > 0$, that is, $\langle \theta^*, z \rangle > 0$ for all $z \in Z^*$. Since $Z^*$ is a finite set,
    $\delta := \min_{z \in Z^*} \langle \theta^*, z \rangle > 0$,
    and by the linearity of convex combinations, $\langle \theta^*, z \rangle \geq \delta > 0$ for every $z \in \Conv Z^*$. In particular $0 \notin \Conv Z^*$.

    \textbf{Step 3 (the estimate).}
    If $Z^* = \emptyset$, the claim is trivial, so assume $Z^* \neq \emptyset$. By Steps 1 and 2, $Z^* \subset \Lambda$ is a nonempty finite set of integer vectors with $0 \notin \Conv Z^*$, and every $z \in Z^*$ satisfies $|z_i| \leq M_i$. Applying \Cref{lem:dist_lattice_polytope} with $A = Z^*$ and $R = M$ gives
    \[
        \min_{z \in \Conv Z^*} \| z \| \geq ( 2 \| M \|_2 )^{-(d-1)} = \frac{1}{2^{d-1} \| M \|_2^{d-1}} ,
    \]
    and combining this with \cref{eq:gamma_sub_as_dist_ball} gives the claim.
    
\end{proof}

\begin{proof}[Proof of \Cref{prop:gamma_sub_lowdim_ball}]
    Steps 1--2 of the proof of \Cref{theo:gamma_sub_lowerbound_on_ball} use no assumption on the dimension, so they hold as they are, and
    \[
        \gamma_\mathrm{sub} \geq \min_{z \in \Conv Z^*} \| z \|
        \geq \dist ( 0, \aff Z^* ) ,
        \qquad Z^* \subset \Lambda \subset \mathbb{Z}^d .
    \]
    Take affinely independent $y^1, \ldots, y^{k+1} \in Z^*$ (spanning $\aff Z^*$) and set $v_i := y^{i+1} - y^1 \in \mathbb{Z}^d$ ($i = 1, \ldots, k$). Since $y^i \in \Lambda$, each component satisfies $|v_{i,j}| \leq 2 M_j$ and hence $\|v_i\| \leq 2 \| M \|_2$. Since $\aff Z^* = y^1 + \spn(v_1, \ldots, v_k)$, the condition $0 \notin \aff Z^*$ is equivalent to $y^1 \notin \spn(v_1, \ldots, v_k)$, and then $(v_1, \ldots, v_k, y^1)$ is linearly independent.

    The distance from the point $0$ to the affine subspace $y^1 + \spn(v_1, \ldots, v_k)$ can be expressed, using Gram determinants (written $G$), as
    \[
        \dist ( 0, \aff Z^* )^2
        = \frac{\det G(v_1, \ldots, v_k, y^1)}{\det G(v_1, \ldots, v_k)}
    \]
    (decomposing orthogonally as $-y^1 = w_U + w_\perp$ with $w_U \in U := \spn(v_1, \ldots, v_k)$ and $w_\perp \perp U$, the multilinearity of Gram determinants and elementary column operations give $\det G(v, -y^1) = \det G(v) \cdot \|w_\perp\|^2$, together with $\det G(v, -y^1) = \det G(v, y^1)$).
    The numerator is the Gram determinant of linearly independent integer vectors, hence a positive integer, in particular $\geq 1$. The denominator is bounded, by Hadamard's inequality for Gram determinants, as
    $\det G(v_1, \ldots, v_k) \leq \prod_{i=1}^k \|v_i\|^2 \leq (2\| M \|_2)^{2k}$.
    Altogether we obtain $\dist(0, \aff Z^*) \geq (2\| M \|_2)^{-k}$.
\end{proof}

\section{Proof of \texorpdfstring{\Cref{theo:gamma_sub_lowerbound_on_simplex}}{the lower bound for a general ILP with the probability simplex}}
\label{app:proof_simplex}

\begin{proof}[Proof of \Cref{theo:gamma_sub_lowerbound_on_simplex}]
    \textbf{Step 1 (reduction to the lattice structure).}
    As in Step 1 of the proof of \Cref{theo:gamma_sub_lowerbound_on_ball}, the set $Z^* \subset \Lambda := \prod_{i=1}^d \{-M_i, \ldots, M_i\} \subset \mathbb{Z}^d$ is finite and $\Conv Z^*$ is a bounded closed convex polytope. When $\Theta = \Delta^{d-1}$ we have $\max_{\theta \in \Theta} \langle \theta, z \rangle = \max_i z_i$, so \cref{eq:gamma_minimax} gives
    \begin{equation}
        \gamma_\mathrm{sub}
        \geq \min_{z \in \Conv Z^*} \max_{i = 1, \ldots, d} z_i .
        \label{eq:gamma_sub_as_max_simplex}
    \end{equation}

    \textbf{Step 2 ($\Conv Z^* \cap \mathbb{R}_{\leq 0}^d = \emptyset$).}
    By \Cref{assu:margin}(3) we have $\langle \theta^*, z \rangle > 0$ for every $z \in Z^*$. Since $\theta^* \in \Delta^{d-1}$ has nonnegative components summing to $1$, the inequality $\sum_i \theta^*_i z_i > 0$ implies that at least one $z_i > 0$, and in particular $\max_i z_i > 0$. By the linearity of convex combinations, $\langle \theta^*, z \rangle > 0$ for every $z \in \Conv Z^*$, and hence $\max_i z_i > 0$. Rewritten in the language of sets,
    \begin{equation}
        \Conv Z^* \cap \mathbb{R}_{\leq 0}^d = \emptyset .
        \label{eq:disjoint_orthant_sub}
    \end{equation}
    Below we assume $Z^* \neq \emptyset$ (if $Z^* = \emptyset$ then the $\min$ in \cref{eq:gamma_sub_as_max_simplex} is $+\infty$ as an infimum over the empty set and the claim is trivial). In this case some $z \in Z^*$ satisfies $\max_i z_i > 0$, and since $z$ is an integer vector, $z_i \geq 1$ for some $i$, hence $M_i \geq 1$ and in particular $\| M \|_2 \geq 1$.

    \textbf{Step 3 (the polyhedron and the nonnegativity of its facet normals).}
    Consider the Minkowski sum
    $\widetilde{P} := \Conv Z^* + \mathbb{R}_{\geq 0}^d$.
    The set $\widetilde{P}$ is a polyhedron (the Minkowski sum of a bounded polytope and a polyhedral cone; the Minkowski--Weyl decomposition, \citealp[cf.][]{schrijver1986theory}), and the following hold.
    \begin{enumerate}
        \item (Full-dimensionality, pointedness, and integrality of the vertices) Since $\widetilde{P}$ contains $z + \mathbb{R}_{\geq 0}^d$ for $z \in \Conv Z^*$, it is $d$-dimensional. Its characteristic cone is $\mathbb{R}_{\geq 0}^d$, which is pointed, so $\widetilde{P}$ has vertices. Furthermore the vertices of $\widetilde{P}$ are elements of $Z^*$: a point $x = z + w$ (with $z \in \Conv Z^*$, $w \geq 0$, $w \neq 0$) can be written as $x = \frac{1}{2} z + \frac{1}{2} (z + 2w)$, the midpoint of two distinct points of $\widetilde{P}$, so it is not a vertex, and hence the vertices belong to $\Conv Z^*$; since $\Conv Z^* \subseteq \widetilde{P}$, the vertices of $\widetilde{P}$ are extreme points of $\Conv Z^*$, and as extreme points of the convex hull of the finite set $Z^*$ they belong to $Z^* \subset \Lambda$.
        \item (Nonnegativity of the facet normals) Writing a valid inequality defining an arbitrary facet $F$ of $\widetilde{P}$ as $a_F \cdot x \geq c_F$ (with $\widetilde{P} \subseteq \{x : a_F \cdot x \geq c_F\}$ and $F = \widetilde{P} \cap \{x : a_F \cdot x = c_F\}$), for every $x \in \widetilde{P}$ and $t \geq 0$ we have $x + t e_i \in \widetilde{P}$, so $a_F \cdot (x + t e_i) \geq c_F$ holds for all $t \geq 0$, which gives $(a_F)_i \geq 0$, that is, $a_F \geq 0$.
        \item (Separation of the origin) By \cref{eq:disjoint_orthant_sub}, for every $z \in \Conv Z^*$ and $w \geq 0$ we have $\max_i (z + w)_i \geq \max_i z_i > 0$, so $\widetilde{P} \cap \mathbb{R}_{\leq 0}^d = \emptyset$ and in particular $0 \notin \widetilde{P}$. Since $\widetilde{P}$ is a $d$-dimensional polyhedron, it coincides with the intersection of the valid inequalities defined by its facets (a standard fact of polyhedral theory, \citealp[cf.][]{schrijver1986theory}). Therefore there exists a facet $F$ with $a_F \cdot 0 = 0 < c_F$, that is, whose valid inequality separates the origin.
    \end{enumerate}

    \textbf{Step 4 (construction of an integral nonnegative normal vector from lattice points and directions in $\mathbb{R}_{\geq 0}^d$).}
    Take the facet $F$ of Step 3. Since $F$ is a face of the pointed polyhedron $\widetilde{P}$, it is a pointed polyhedron, and its vertices are vertices of $\widetilde{P}$, hence elements of $Z^* \subset \Lambda$. Moreover, by $a_F \geq 0$, the characteristic cone of $F$ is
    \[
        \rec(F)
        = \mathbb{R}_{\geq 0}^d \cap \{w : a_F \cdot w = 0\}
        = \cone \{e_i : i \in I_0\} ,
        \qquad I_0 := \{i : (a_F)_i = 0\}
    \]
    (since $w \geq 0$ together with $a_F \cdot w = 0$ forces $w_i = 0$ for every $i$ with $(a_F)_i > 0$). Taking, among the vertices of $F$, affinely independent points $y^1, \ldots, y^p \in Z^* \subset \Lambda$ ($1 \leq p \leq d$) spanning the affine hull of the vertex set, we have
    \[
        \aff F = \aff(y^1, \ldots, y^p) + \spn \{e_i : i \in I_0\} ,
        \qquad \dim \aff F = d - 1 ,
    \]
    so we can choose $I' \subseteq I_0$ with $|I'| = d - p$ such that $y^1, \ldots, y^p$ and $y^1 + e_i$ ($i \in I'$) are $d$ affinely independent points spanning $\aff F$.

    Define the $(d-1) \times d$ integer matrix $U$ whose rows are the difference vectors, namely the rows $v_k := y^{k+1} - y^1$ ($k = 1, \ldots, p-1$; each component satisfying $|v_{k,j}| \leq 2 M_j$ since $y^k \in \Lambda$) and the rows $e_i$ ($i \in I'$), and construct the integer vector $a$ from the cofactors
    $a_i := (-1)^i \det(U^{(i)}) \in \mathbb{Z}$ ($i = 1, \ldots, d$),
    where $U^{(i)}$ is the matrix with the $i$-th column deleted. The rows of $U$ are linearly independent (being difference vectors of $d$ affinely independent points), so $a \neq 0$, and since $a$ is orthogonal to all the rows of $U$, it is a normal direction of $\aff F$, that is, parallel to $a_F$. Choosing the sign in the same direction as $a_F$, we have $a = \lambda a_F$ for some $\lambda > 0$, so
    \[
        a \in \mathbb{Z}_{\geq 0}^d \setminus \{0\} ,
        \qquad \Conv Z^* \subseteq \widetilde{P} \subseteq \{x : a \cdot x \geq c\} ,
        \qquad c := a \cdot y^1 = \lambda c_F > 0
    \]
    holds, and $c \in \mathbb{Z}_{> 0}$ gives $c \geq 1$.

    \textbf{Step 5 (the estimate).}
    First the case $p = 1$: the rows of $U$ are the $d-1$ unit vectors $\{e_i\}_{i \in I'}$, so $a = \pm e_{i_0}$ (with $i_0 \notin I'$), and nonnegativity gives $a = e_{i_0}$ and $c = y^1_{i_0} \geq 1$. Hence for every $z \in \Conv Z^*$ we have $\max_i z_i \geq z_{i_0} = a \cdot z \geq c \geq 1$, and since $\| M \|_2 \geq 1$ the right-hand side of the claim is at most $1$, so the claim follows. Below we assume $p \geq 2$.

    For every $z \in \Conv Z^*$, setting $\lambda_i := a_i / \sum_{j=1}^d a_j$ under $\sum_j a_j > 0$ (which holds since $a \neq 0$ and $a_j \geq 0$), we have $\lambda \in \Delta^{d-1}$, so
    \begin{equation}
        \max_{i = 1, \ldots, d} z_i \geq \sum_{i=1}^d \lambda_i z_i = \frac{a \cdot z}{\sum_j a_j} \geq \frac{c}{\sum_j a_j} \geq \frac{1}{\sum_{j=1}^d a_j} .
        \label{eq:convex_lower_sub_simplex}
    \end{equation}

    We now find an upper bound on $\sum_j a_j$. By Hadamard's inequality, the contribution of the difference-vector rows of $U^{(i)}$, namely the vectors $v_k^{(i)} \in \mathbb{Z}^{d-1}$ obtained from $v_k$ by deleting the $i$-th component, is $\prod_{k=1}^{p-1} \|v_k^{(i)}\| \leq \prod_{k=1}^{p-1} 2 ( \sum_{j \neq i} M_j^2 )^{1/2}$ and the contribution of the unit-vector rows is at most $1$, so
    \begin{align*}
        a_i \leq |a_i|
        &\leq \left( 2 \left( \textstyle\sum_{j \neq i} M_j^2 \right)^{1/2} \right)^{p-1}
        \leq \left( 2 \left( \textstyle\sum_{j \neq i} M_j^2 \right)^{1/2} \right)^{d-1}
        \\
        &= 2^{d-1} \left( \sum_{j \neq i} M_j^2 \right)^{(d-1)/2}
    \end{align*}
    holds (for the second inequality: if $\sum_{j \neq i} M_j^2 \geq 1$ it follows since the base is at least $2$ and $p - 1 \leq d - 1$, whereas if $\sum_{j \neq i} M_j^2 = 0$ then all the difference-vector rows $v_k^{(i)}$ are zero vectors, so $a_i = 0$ since $p \geq 2$; in either case it holds).
    Since $a_i \geq 0$,
    \[
        \sum_{i=1}^d a_i
        \leq 2^{d-1} \sum_{i=1}^d \left( \sum_{j \neq i} M_j^2 \right)^{(d-1)/2}
        = 2^{d-1}\, \| M \|_2^{d-1} \sum_{i=1}^d \left( 1 - \frac{M_i^2}{\| M \|_2^2} \right)^{(d-1)/2}
    \]
    (noting that $\| M \|_2 \geq 1 > 0$).
    Setting $\mu_i := M_i^2 / \| M \|_2^2$, we have $\mu \in \Delta^{d-1}$. Applying \Cref{prop:max_1-x_on_simplex} to $\mu$ with $\xi = (d-1)/2$: for $d \geq 3$ we have $\xi \geq 1$, so $\sum_i (1 - \mu_i)^{(d-1)/2} \leq d-1$; for $d = 2$ we have $\xi = 1/2 < 1$, so $\sum_i (1 - \mu_i)^{1/2} \leq d (1 - 1/d)^{1/2} = \sqrt{2}$. Hence
    \begin{equation}
        \sum_{i=1}^d a_i \leq 2^{d-1}\, \max(d-1, \sqrt{2})\, \| M \|_2^{d-1} .
        \label{eq:sum_a_bound_sub_simplex}
    \end{equation}
    Combining \cref{eq:gamma_sub_as_max_simplex,eq:convex_lower_sub_simplex,eq:sum_a_bound_sub_simplex} gives the claim.
\end{proof}

\section{Proofs of the lower bounds for \texorpdfstring{M-convex and M${}^\natural$-convex structures}{M-convex structures}}
\label{app:proof_mconvex}

\begin{proof}[Proof of \Cref{theo:gamma_sub_M_convex_poly}]
    By \Cref{prop:M-convex_test_set}, the test set can be taken to be $\mathcal{T} = \{ e_i - e_j \mid i \neq j \}$.
    Take a permutation $\sigma$ so that $\theta^*_{\sigma(1)} \geq \theta^*_{\sigma(2)} \geq \cdots \geq \theta^*_{\sigma(d)}$ (fixing an arbitrary order among components of equal value).

    The case $\Theta = \{ \theta : \| \theta \|_2 \leq 1 \}$: setting the arithmetic arrangement $\theta^\dagger_{\sigma(i)} = \delta ( (d+1)/2 - i )$ ($i = 1, \ldots, d$) with $\delta = 2\sqrt{3}/\sqrt{d(d^2-1)}$, we have $\| \theta^\dagger \|_2 = 1$ and hence $\theta^\dagger \in \Theta$.
    If $\theta^*_i > \theta^*_j$ then $i$ ranks above $j$ in the order of $\sigma$, so $\theta^\dagger_i > \theta^\dagger_j$, that is, $\langle \theta^\dagger, e_i - e_j \rangle > 0$; hence $\theta^\dagger \in \Theta_{\mathcal{T}}(\theta^*)$.
    Furthermore, for every $(i, j)$ with $\langle \theta^*, e_i - e_j \rangle > 0$, the difference $\theta^\dagger_i - \theta^\dagger_j$ is at least the difference $\delta$ between adjacent ranks of $\sigma$, so \Cref{prop:opt_gap_down_test_set_gap} gives
    \[
        \gamma_\mathrm{sub}
        \geq \min_{g \in \mathcal{T},\, \langle \theta^*, g \rangle > 0} \langle \theta^\dagger, g \rangle
        \geq \delta
        = \frac{2\sqrt{3}}{\sqrt{d(d^2-1)}} .
    \]

    The case $\Theta = \Delta^{d-1}$: setting $\theta^\dagger_{\sigma(i)} = (d-i) \cdot 2/(d(d-1))$ ($i = 1, \ldots, d$), we have $\theta^\dagger \in \Delta^{d-1}$, and as above $\theta^\dagger \in \Theta_{\mathcal{T}}(\theta^*)$. The difference between adjacent ranks is $2/(d(d-1))$, so \Cref{prop:opt_gap_down_test_set_gap} gives
    $\gamma_\mathrm{sub} \geq 2/(d(d-1))$.
\end{proof}

\begin{proof}[Proof of \Cref{theo:gamma_sub_Mnatural_convex_poly}]
    By \Cref{prop:Mnatural-convex_test_set}, the set $\mathcal{T} = \{ e_i - e_j,\ \pm e_i \mid i \neq j \}$ is a test set. The elements $g \in \mathcal{T}$ with $\langle \theta^*, g \rangle > 0$ are of three kinds: $g = e_i - e_j$ (with $\theta^*_i > \theta^*_j$), $g = +e_i$ (with $\theta^*_i > 0$), and $g = -e_i$ (with $\theta^*_i < 0$). Below we construct, for a fixed $\theta^*$, a sign-consistent $\theta^\dagger \in \Theta_{\mathcal{T}}(\theta^*)$ and apply \Cref{prop:opt_gap_down_test_set_gap}.

    The case $\Theta = \{ \theta \in \mathbb{R}^d : \| \theta \|_2 \leq 1 \}$: reorder the coordinates so that $\theta^*_1 \geq \cdots \geq \theta^*_d$, and let $d_+, d_0, d_-$ ($d_+ + d_0 + d_- = d$) be the numbers of positive, zero and negative components (if $\theta^* = 0$ then $S^+ = \{ g \in \mathcal{T} \mid \langle \theta^*, g \rangle > 0 \} = \emptyset$ and the lower bound of \Cref{prop:opt_gap_down_test_set_gap} holds trivially, so below we may assume $\theta^* \neq 0$, that is, $d_+ + d_- \geq 1$). Define the integer vector
    \[
        \omega = (\, d_+, d_+ - 1, \ldots, 1,\ \underbrace{0, \ldots, 0}_{d_0},\ -1, \ldots, -d_- \,)
    \]
    and set $\theta^\dagger := c\, \omega$ with $c := ( \sum_i \omega_i^2 )^{-1/2}$ (so $\| \theta^\dagger \|_2 = 1$ and hence $\theta^\dagger \in \Theta$). Each of the three kinds of $g$ above satisfies $\langle \theta^\dagger, g \rangle \geq c > 0$: for $g = e_i - e_j$ we have $\omega_i - \omega_j \geq 1$ and hence $\langle \theta^\dagger, g \rangle = c (\omega_i - \omega_j) \geq c$; for $g = +e_i$ we have $\omega_i \geq 1$ and hence $\langle \theta^\dagger, g \rangle = c \omega_i \geq c$; for $g = -e_i$ we have $-\omega_i \geq 1$ and hence $\langle \theta^\dagger, g \rangle = c (-\omega_i) \geq c$. Hence $\theta^\dagger \in \Theta_{\mathcal{T}}(\theta^*)$, and \Cref{prop:opt_gap_down_test_set_gap} gives
    $\gamma_\mathrm{sub} \geq c$.
    Finally, from $\sum_i \omega_i^2 = \sum_{k=1}^{d_+} k^2 + \sum_{k=1}^{d_-} k^2 \leq \sum_{k=1}^d k^2 = \frac{d(d+1)(2d+1)}{6}$ (since $d_+ + d_- \leq d$), together with $d(d+1)(2d+1) \leq 6 d^3$, we get $c \geq \sqrt{6/(d(d+1)(2d+1))} \geq d^{-3/2}$.

    The case $\Theta = \Delta^{d-1}$: since $\theta^* \geq 0$, an element $g = -e_i$ has $\langle \theta^*, g \rangle = -\theta^*_i \leq 0$ and hence does not satisfy $\langle \theta^*, g \rangle > 0$. Reorder the coordinates so that $\theta^*_1 \geq \cdots \geq \theta^*_d$ and set
    $\theta^\dagger_i := \frac{2 (d - i + 1)}{d (d+1)}$ ($i = 1, \ldots, d$);
    then, being a decreasing sequence with all components positive and $\sum_i \theta^\dagger_i = 1$, we have $\theta^\dagger \in \Delta^{d-1}$. The elements $g$ with $\langle \theta^*, g \rangle > 0$ are limited to the two kinds $g = e_i - e_j$ (with $\theta^*_i > \theta^*_j$, hence $i < j$) and $g = +e_i$ (with $\theta^*_i > 0$), and both satisfy $\langle \theta^\dagger, g \rangle \geq \frac{2}{d(d+1)} > 0$ (for the former $\langle \theta^\dagger, g \rangle = (j - i) \frac{2}{d(d+1)}$, and for the latter $\langle \theta^\dagger, g \rangle = \theta^\dagger_i \geq \theta^\dagger_d = \frac{2}{d(d+1)}$). Hence $\theta^\dagger \in \Theta_{\mathcal{T}}(\theta^*)$, and \Cref{prop:opt_gap_down_test_set_gap} gives $\gamma_\mathrm{sub} \geq \frac{2}{d(d+1)}$.
\end{proof}

\section{Proofs of the lower bounds via Graver bases}
\label{app:proof_graver}

\begin{proof}[Proof of \Cref{prop:linear_inequality_test_set}]
    Take any $s \in \mathcal{S}$, $x^1 \in X(s)$ and $\theta \in \Theta$, and suppose $\langle \theta, x^1 \rangle < \max_{x \in X(s)} \langle \theta, x \rangle$.
    Since $X(s)$ is bounded by \Cref{assu:margin}(2) and lies on the integer lattice, it is a finite set, so a maximizer $x^2 \in \argmax_{x \in X(s)} \langle \theta, x \rangle$ exists.
    Adding slacks and setting $\widetilde{x}^j := (x^j,\, b(s) - A x^j) \in \mathbb{Z}^{d+N}$ ($j = 1, 2$), we have $\widetilde{A} \widetilde{x}^j = b(s)$, and the last $N$ components (the slack components) are nonnegative.
    The difference $\widetilde{z} := \widetilde{x}^2 - \widetilde{x}^1 \in \ker_{\mathbb{Z}}(\widetilde{A}) \setminus \{ 0 \}$ can be written, by the sign-consistent decomposition property of Graver bases (for any integer matrix $B$, every $0 \neq z \in \ker_{\mathbb{Z}}(B)$ decomposes into a sum of elements of $\mathcal{G}(B)$ that are sign consistent with $z$; \citealp[cf.][]{sturmfels1996grobner,onn2010nonlinear}), as
    $\widetilde{z} = \sum_{k=1}^r \widetilde{g}^k$ with $\widetilde{g}^k \in \mathcal{G}(\widetilde{A})$ and $\widetilde{g}^k \sqsubseteq \widetilde{z}$.
    By sign consistency, for every $J \subseteq \{ 1, \ldots, r \}$ each component of the partial sum $\widetilde{x}^1 + \sum_{k \in J} \widetilde{g}^k$ takes a value between the corresponding components of $\widetilde{x}^1$ and $\widetilde{x}^2$. In particular the slack components stay nonnegative, and the equality with respect to $\widetilde{A}$ is preserved, so the first $d$ components of the partial sum belong to $X(s)$.
    Setting $\widetilde{\theta} := (\theta, 0) \in \mathbb{R}^{d+N}$, we have $\langle \widetilde{\theta}, \widetilde{z} \rangle = \langle \theta, x^2 - x^1 \rangle > 0$, so $\langle \widetilde{\theta}, \widetilde{g}^{k_0} \rangle > 0$ for some $k_0$.
    Setting $g := \pi_x(\widetilde{g}^{k_0}) \in \mathcal{T}_x$, the partial sum with $J = \{ k_0 \}$ gives $x^1 + g \in X(s)$ and $\langle \theta, g \rangle > 0$.
    Finally, since a projection does not increase the $\ell_\infty$ norm, $\| \pi_x(\widetilde{g}) \|_\infty \leq \| \widetilde{g} \|_\infty \leq g_\infty(\widetilde{A})$.
\end{proof}

\begin{proof}[Proof of \Cref{theo:gamma_sub_linear_inequality_ball}]
    By \Cref{prop:linear_inequality_test_set}, the set $\mathcal{T}_x$ is a test set with $\| g \|_\infty \leq C_g$ for all $g \in \mathcal{T}_x$. Set $S^+ := \{ g \in \mathcal{T}_x \mid \langle \theta^*, g \rangle > 0 \} \subset \mathbb{Z}^d$.

    \textbf{Step 1 (reduction to a distance via minimax).}
    If $S^+ = \emptyset$, then by the test-set property every point of every $X(s)$ is optimal for $\theta^*$, so by \Cref{assu:margin}(3) every $X(s)$ is a singleton and the claim is trivial; hence assume $S^+ \neq \emptyset$.
    Since $\langle \theta^*, g \rangle > 0$ for each $g \in S^+$, we have $\langle \theta^*, q \rangle > 0$ for every $q \in \Conv(S^+)$, and in particular $0 \notin \Conv(S^+)$. By \Cref{prop:Minimax_theorem_finite} ($B^d = \{ \|\theta\|_2 \leq 1 \}$ being bounded, closed and convex, and $S^+$ finite) together with $\max_{\|\theta\|_2 \leq 1} \langle \theta, q \rangle = \| q \|_2$,
    \[
        \max_{\theta \in B^d} \min_{g \in S^+} \langle \theta, g \rangle
        = \min_{q \in \Conv(S^+)} \| q \|_2
        = \dist ( 0, \Conv(S^+) ) > 0 .
    \]
    On the other hand, \Cref{prop:opt_gap_down_test_set_gap} gives $\gamma_\mathrm{sub} \geq \sup_{\theta \in \Theta_{\mathcal{T}_x}(\theta^*)} \min_{g \in S^+} \langle \theta, g \rangle$. A $\theta \in B^d$ attaining the maximum above satisfies $\min_{g \in S^+} \langle \theta, g \rangle = \dist(0, \Conv(S^+)) > 0$, that is, $\langle \theta, g \rangle > 0$ for all $g \in S^+$, so $\theta \in \Theta_{\mathcal{T}_x}(\theta^*)$ (since $S^+ = \{ g \mid \langle \theta^*, g \rangle > 0 \}$). Hence $\gamma_\mathrm{sub} \geq \dist(0, \Conv(S^+))$.

    \textbf{Step 2 (the estimate).}
    The set $S^+$ is a nonempty finite set of integer vectors with $0 \notin \Conv(S^+)$ (Step 1), and every $g \in S^+$ satisfies $|g_j| \leq C_g$ for all $j$. Applying \Cref{lem:dist_lattice_polytope} with $A = S^+$ and $R = (C_g, \ldots, C_g)$, for which $\| R \|_2 = C_g \sqrt{d}$, gives $\dist ( 0, \Conv(S^+) ) \geq ( 2 C_g \sqrt{d} )^{-(d-1)}$. Combining this with Step 1 gives $\gamma_\mathrm{sub} \geq 1/(2C_g\sqrt{d})^{d-1}$.
    
\end{proof}

\begin{proof}[Proof of \Cref{theo:gamma_sub_linear_inequality_simplex}]
    By \Cref{prop:linear_inequality_test_set}, the set $\mathcal{T}_x$ is a test set with $\| g \|_\infty \leq C_g$. Set $S^+ := \{ g \in \mathcal{T}_x \mid \langle \theta^*, g \rangle > 0 \} \subset \mathbb{Z}^d$.

    \textbf{Step 1 (reduction via minimax).}
    \Cref{prop:opt_gap_down_test_set_gap} gives
    \[
        \gamma_\mathrm{sub} \geq \sup_{\theta \in \Theta_{\mathcal{T}_x}(\theta^*)} \min_{g \in S^+} \langle \theta, g \rangle .
    \]
    By \Cref{prop:Minimax_theorem_finite} ($\Delta^{d-1}$ being bounded, closed and convex, and $S^+$ finite) together with $\max_{\theta \in \Delta^{d-1}} \langle \theta, q \rangle = \max_i q_i$, we have $\max_{\theta \in \Delta^{d-1}} \min_{g \in S^+} \langle \theta, g \rangle = \min_{q \in \Conv(S^+)} \max_i q_i$. From $\theta^* \geq 0$ and $\langle \theta^*, g \rangle > 0$ for each $g \in S^+$ we get $\max_i q_i > 0$ for every $q \in \Conv(S^+)$, that is, $\Conv(S^+) \cap \mathbb{R}_{\leq 0}^d = \emptyset$, so $\min_{q \in \Conv(S^+)} \max_i q_i > 0$. A $\theta \in \Delta^{d-1}$ attaining this maximum satisfies $\langle \theta, g \rangle > 0$ for all $g \in S^+$, so $\theta \in \Theta_{\mathcal{T}_x}(\theta^*)$. Hence $\gamma_\mathrm{sub} \geq \min_{q \in \Conv(S^+)} \max_i q_i > 0$.

    \textbf{Steps 2--4 (an integral nonnegative normal vector via the polyhedron).}
    As in Steps 3--4 of the proof of \Cref{theo:gamma_sub_lowerbound_on_simplex} (replacing $Z^*$ by $S^+$ and the lattice box $\Lambda$ by $\{ g \in \mathbb{Z}^d : \| g \|_\infty \leq C_g \}$), consider the polyhedron $\widetilde{P} := \Conv(S^+) + \mathbb{R}_{\geq 0}^d$. By Step 1 we have $\Conv(S^+) \cap \mathbb{R}_{\leq 0}^d = \emptyset$, so $0 \notin \widetilde{P}$; the facet normals of $\widetilde{P}$ can be taken nonnegative, and there is a facet separating the origin. From the vertices $y^1, \ldots, y^p \in S^+$ ($1 \leq p \leq d$, lattice points) and the points $y^1 + e_i$ ($i \in I'$, $|I'| = d - p$) in unit-vector directions spanning the affine hull of that facet, the cofactor construction yields an integral normal vector $a \in \mathbb{Z}_{\geq 0}^d \setminus \{0\}$ and $c := a \cdot y^1 \in \mathbb{Z}_{\geq 1}$ with $\Conv(S^+) \subseteq \{ x : a \cdot x \geq c \}$. In the case $p = 1$ the vector $a$ is a unit vector $e_{i_0}$, so $\max_i q_i \geq q_{i_0} \geq c \geq 1$ for every $q \in \Conv(S^+)$, whereas the right-hand side of the claim is at most $1$, so the claim follows immediately. Below we assume $p \geq 2$.

    \textbf{Step 5 (the estimate).}
    For every $q \in \Conv(S^+)$, setting $\lambda_i := a_i / \sum_j a_j \in \Delta^{d-1}$ gives $\max_i q_i \geq \sum_i \lambda_i q_i = (a \cdot q)/\sum_j a_j \geq c/\sum_j a_j \geq 1/\sum_j a_j$. Since $y^k \in S^+$ gives $|(v_k)_j| \leq 2C_g$, taking $M_j = C_g$ in the Hadamard estimate of Step 5 of the proof of \Cref{theo:gamma_sub_lowerbound_on_simplex}, \cref{eq:sum_a_bound_sub_simplex} gives
    \[
        \sum_{j=1}^d a_j \leq 2^{d-1}\max(d-1, \sqrt{2})\,(C_g\sqrt{d})^{d-1}
        = \max(d-1, \sqrt{2})\,(2C_g\sqrt{d})^{d-1} .
    \]
    Combining the above gives $\gamma_\mathrm{sub} \geq 1/\sum_j a_j \geq 1/(\max(d-1, \sqrt{2})\,(2C_g\sqrt{d})^{d-1})$.
\end{proof}

\section{Derivation of the explicit upper bounds by problem class}
\label{app:proof_corollaries}

In this appendix we derive each entry of \Cref{tab:explicit_upper}. We first give the general form of the substitution.

\begin{corollary}[The three guarantees in terms of $\gamma$]
    \label{cor:explicit_general}
    Under \Cref{assu:margin}, the following hold (in the order expressions we may assume $\gamma \leq LD$, since $K = 0$ otherwise by \Cref{lem:r_bounds}).
    \begin{description}
        \item[(i)] SGS-OGD (\Cref{alg:ogd} with $\alpha = D/(L\sqrt{2})$) satisfies
        \begin{equation*}
            K \leq \frac{2 L^2 D^2}{\gamma^2} ,
            \qquad
            \widetilde{R}_T \leq \frac{2 L^2 D^2}{\gamma} .
        \end{equation*}
        \item[(ii)] ONS (\Cref{alg:ons}) satisfies
        \begin{align*}
            K
            & = O \left( \frac{d L D}{\gamma} \log \max \left( \frac{2 L D}{\gamma}, 2 \right) \right) ,
            \\
            \widetilde{R}_T
            &\leq L D \left( 1 + 2 d \log \left( 2 + \frac{2 L D}{\gamma} \right) \right)
            = O \left( L D\, d \log \max \left( \frac{L D}{\gamma}, 2 \right) \right) .
        \end{align*}
        \item[(iii)] Growing-grid SGS-MetaGrad (\Cref{alg:metagrad_anytime}) satisfies
        \begin{align*}
            K
            &= O \left( \frac{d L D}{\gamma} \log \max \left( \frac{2 L D}{\gamma}, 2 \right) \right) ,
            \\
            \widetilde{R}_T
            &= O \left( L D\, d \log \max \left( \frac{L D}{\gamma}, 2 \right) \right) .
        \end{align*}
    \end{description}
\end{corollary}

\begin{proof}
    The bounds are those in items (i) and (iii) of \Cref{theo:ogd,theo:main,theo:metagrad_anytime_main}.
    The order estimates hold because, in both (ii) and (iii), the leading terms are $\frac{L D}{\gamma} \cdot d \log \max ( \frac{2 L D}{\gamma}, 2 )$ (for $K$) and $L D\, d \log \max ( \frac{L D}{\gamma}, 2 )$ (for $\widetilde{R}_T$).
\end{proof}

By \Cref{lem:margin_from_gamma}, \Cref{assu:margin} holds with $\gamma = \gamma_\mathrm{sub}$ and $L = L_\mathrm{sub}$ under each structure.
The bounds of \Cref{cor:explicit_general} are monotonically nonincreasing in $\gamma$, since the factors $1/\gamma$, $\log \max ( \cdot/\gamma, 1 )$ and $\log ( 2 + \cdot/\gamma )$ are nonincreasing in $\gamma$; hence substituting the lower bounds of \Cref{tab:gamma_lower} for $\gamma$ gives upper bounds.
Below, $M$ is the vector in \cref{eq:def_M}, and in the ILP and linear-inequality entries we also use $L \leq \| M \|_2$ (\cref{eq:L_leq_m}).
Moreover, in reducing the logarithmic factors to the forms of the entries of \Cref{tab:explicit_upper} ($\log (2 \| M \|_2)$ for ILPs, $\log (2 C_g d L)$ for linear inequalities, and $\log (2 d L)$ for the M-convex and M${}^\natural$-convex cases), we assume $L \geq 1$ and $C_g \geq 1$ (whence $\| M \|_2 \geq 1$ as well, since $L \leq \| M \|_2$).
The diameter $D = \diam(\Theta)$ is $2$ for the unit ball and $\sqrt{2}$ for the probability simplex.

\paragraph{General ILPs (\Cref{theo:gamma_sub_lowerbound_on_ball,theo:gamma_sub_lowerbound_on_simplex})}
For the unit ball,
\begin{equation*}
    \frac{1}{\gamma} \leq 2^{d-1} \| M \|_2^{d-1} = O ( 2^{d} \| M \|_2^{d-1} ) .
\end{equation*}
By \Cref{cor:explicit_general}(i), SGS-OGD satisfies
\begin{align*}
    K &\leq \frac{2 L^2 D^2}{\gamma^2} = \frac{8 L^2}{\gamma^2} \leq 2 \cdot 4^{d} \| M \|_2^{2d-2} L^2 = O ( 4^{d} \| M \|_2^{2d} ) ,
    \\
    \widetilde{R}_T &\leq \frac{2 L^2 D^2}{\gamma} = \frac{8 L^2}{\gamma} = O ( 2^{d} \| M \|_2^{d+1} ) ,
\end{align*}
and by (ii) and (iii), ONS and growing-grid SGS-MetaGrad satisfy
\begin{align*}
    K &= O \left( \frac{d L D}{\gamma} \log \max \left( \frac{2 L D}{\gamma}, 2 \right) \right) = O ( d^{2} 2^{d} \| M \|_2^{d} \log (2\| M \|_2) ) ,
    \\
    \widetilde{R}_T &= O \left( L D\, d \log \max \left( \frac{L D}{\gamma}, 2 \right) \right)
    \\
    &= O ( \| M \|_2 d \cdot d \log (2\| M \|_2) ) = O ( d^2 \| M \|_2 \log (2\| M \|_2) )
\end{align*}
(where we used $\log (1/\gamma) = O(d \log \| M \|_2 + d)$).
For the probability simplex we have $1/\gamma = O(2^{d} d\, \| M \|_2^{d-1})$, so the same substitution gives, for SGS-OGD,
\begin{align*}
    K &= O ( d^2 4^{d} \| M \|_2^{2d} ) ,
    \\
    \widetilde{R}_T &= O ( d\, 2^{d} \| M \|_2^{d+1} ) ,
\end{align*}
and, for ONS and growing-grid SGS-MetaGrad,
\begin{align*}
    K &= O ( d^{3} 2^{d} \| M \|_2^{d} \log (2\| M \|_2) ) ,
    \\
    \widetilde{R}_T &= O ( d^2 \| M \|_2 \log (2\| M \|_2) ) .
\end{align*}

\paragraph{Linear inequalities (\Cref{theo:gamma_sub_linear_inequality_ball,theo:gamma_sub_linear_inequality_simplex})}
For the unit ball,
\begin{equation*}
    \frac{1}{\gamma} \leq (2C_g\sqrt{d})^{d-1} .
\end{equation*}
SGS-OGD satisfies
\begin{align*}
    K &= O ( L^2 / \gamma^2 ) = O ( L^2 (2C_g\sqrt{d})^{2d-2} ) ,
    \\
    \widetilde{R}_T &= O ( L^2 / \gamma ) = O ( L^2 (2C_g\sqrt{d})^{d-1} ) ,
\end{align*}
and ONS and growing-grid SGS-MetaGrad satisfy
\begin{align*}
    K &= O \left( \frac{d L}{\gamma} \log \max \left( \frac{2 L}{\gamma}, 2 \right) \right) = O ( L d^{2} (2C_g\sqrt{d})^{d-1} \log (2C_g dL) ) ,
    \\
    \widetilde{R}_T &= O \left( L d \log \max \left( \frac{2 L}{\gamma}, 2 \right) \right) = O ( L d^2 \log (2C_g dL) )
\end{align*}
(where we used $\log (1/\gamma) = O(d \log (C_g d))$).
For the probability simplex we have $1/\gamma = O(d\, (2C_g\sqrt{d})^{d-1})$, so the power of $d$ goes up by one and we obtain the probability-simplex entries of \Cref{tab:explicit_upper}.

\paragraph{M-convex and M${}^\natural$-convex (\Cref{theo:gamma_sub_M_convex_poly,theo:gamma_sub_Mnatural_convex_poly})}
For the unit ball we have $1/\gamma \leq \sqrt{d(d^2-1)}/(2\sqrt{3}) = O(d^{3/2})$, so SGS-OGD satisfies
\begin{align*}
    K &\leq \frac{8 L^2}{\gamma^2} \leq \frac{2}{3} L^2 d (d^2-1) = O ( L^2 d^3 ) ,
    \\
    \widetilde{R}_T &\leq \frac{8 L^2}{\gamma} \leq \frac{4 L^2 \sqrt{d(d^2-1)}}{\sqrt{3}} = O ( L^2 d^{3/2} ) ,
\end{align*}
and ONS and growing-grid SGS-MetaGrad satisfy
\begin{align*}
    K &= O \left( \frac{d L}{\gamma} \log \max \left( \frac{2 L}{\gamma}, 2 \right) \right) = O ( L d^{5/2} \log (2dL) ) ,
    \\
    \widetilde{R}_T &= O \left( L d \log \max \left( \frac{2 L}{\gamma}, 2 \right) \right) = O ( L d \log (2dL) ) .
\end{align*}
For the probability simplex we have $1/\gamma \leq d(d-1)/2 = O(d^2)$ and $D = \sqrt{2}$, so similarly SGS-OGD satisfies
\begin{align*}
    K &\leq \frac{4 L^2}{\gamma^2} \leq L^2 d^2 (d-1)^2 = O ( L^2 d^4 ) ,
    \\
    \widetilde{R}_T &\leq \frac{4 L^2}{\gamma} \leq 2 L^2 d(d-1) = O ( L^2 d^2 ) ,
\end{align*}
and ONS and growing-grid SGS-MetaGrad satisfy
\begin{align*}
    K &= O ( L d^{3} \log (2dL) ) ,
    \\
    \widetilde{R}_T &= O ( L d \log (2dL) ) .
\end{align*}
In the M${}^\natural$-convex case only the constants of the bound on $1/\gamma$ change ($d^{3/2}$ for the unit ball and $d(d+1)/2$ for the probability simplex), and the orders coincide.
In particular, when $L = O(\sqrt{d})$, the number of mistakes is $O(d^4)$ to $O(d^5)$ for SGS-OGD and $O(d^3 \log d)$ to $O(d^{7/2} \log d)$ for ONS.

\section{Derivation of the entry of \citet{sakaue2025aistats} in \Cref{tab:comparison}}
\label{app:gap_entry}

\citet[Theorem 5.2]{sakaue2025aistats} assumes the $\Delta$-gap condition, namely that
$\langle \theta^*, x - \hat{x} \rangle \geq \Delta \| x - \hat{x} \|$ holds for the agent's optimal action $x$ and every
$\hat{x} \in X(s^t)$ in every round $t$ (for a norm $\| \cdot \|$ on the actions), and bounds the sum $\widetilde{R}_T$ by
\[
    \widetilde{R}_T \leq \frac{2^{5/4} L_\infty B^3}{\lambda^{3/2} \Delta^2} ,
\]
where $L_\infty$ (denoted $K$ in that paper, renamed here because $K$ is the number of mistakes in this paper) is an upper bound on $\| \hat{x}^t - x^t \|$, the regularizer $\psi \colon \Theta \to \mathbb{R}$ of their FTRL is
$\lambda$-strongly convex with respect to the dual norm, and $B$ is any constant with
\[
    B^2 \geq \max \Big\{ 2^{5/2} \lambda \max_{\theta, \theta' \in \Theta} \| \theta - \theta' \|_\star^2 ,\;
    \max_{\theta, \theta' \in \Theta} ( \psi(\theta) - \psi(\theta') ) \Big\} .
\]
The right-hand side of the above bound on $\widetilde{R}_T$ is increasing in $B$, so the smallest admissible $B$ is taken.

For $\Theta = \Delta^{d-1}$ the authors take $\| \cdot \| = \| \cdot \|_\infty$ on the actions, $\| \cdot \|_\star = \| \cdot \|_1$
on the weights, and the entropic regularizer $\psi(\theta) = \langle \theta, \log \theta \rangle$, which is $1$-strongly convex with
respect to $\| \cdot \|_1$ by Pinsker's inequality, so that $\lambda = 1$.
The two quantities in the definition of $B$ are then constants and $\log d$, respectively: the $\ell_1$ diameter of $\Delta^{d-1}$ is
$2$, so the first is $2^{9/2}$; and $\psi$ ranges over $[-\log d, 0]$ on $\Delta^{d-1}$, attaining $0$ at a vertex
$\mathbf{e}_i$ and $-\log d$ at the barycenter $(1/d, \ldots, 1/d)$, so the second is $\log d$.
Hence $B^2 = \max \{ 2^{9/2}, \log d \}$, which is $\Theta(\log d)$ as $d \to \infty$.
Finally, $L_\infty$ is the $\ell_\infty$ diameter of the feasible sets, which is at most $\| M \|_\infty$ by \cref{eq:def_M}.
Substituting these gives
\[
    \widetilde{R}_T = O \left( \frac{\| M \|_\infty (\log d)^{3/2}}{\Delta^2} \right) ,
\]
which is the entry of \Cref{tab:comparison}; the same bound applies to $R^{\mathrm{est}}_T$ by \cref{eq:max_leq_tilde}.
Two remarks are in order.
First, the numerical constant is loose: \citet{sakaue2025aistats} set $B = 2^{11/4} \sqrt{\log d}$, which is admissible for $d \geq 2$ but
larger than the smallest admissible value, although the order in $d$ is unaffected.
Second, the form depends on the choice of the regularizer and of the pair of norms; the entropic choice above is the one with which
\citet{sakaue2025aistats} recover the guarantee of \citet{Barmann-2018-online} on the probability simplex.

\end{document}